\documentclass[12pt]{colt2017}
\usepackage{algorithm,algorithmic}
\usepackage{amsmath, amssymb, multirow, paralist}
\newif\ifpdf
\ifx\pdfoutput\undefined
   \pdffalse
\else
   \pdfoutput=1
   \pdftrue
\fi
\ifpdf
   \usepackage{graphicx}
   \usepackage{epstopdf}
\else
   \usepackage{graphicx}
\fi

\usepackage{enumitem}
\usepackage[utf8]{inputenc}

\newtheorem{thm}{Theorem}

 \newtheorem{assumption}{Assumption}

\def \R {\mathbb{R}}

\def \w {\mathbf{w}}
\def \v {\mathbf{v}}

\def \x {\mathbf{x}}

\def \x {\mathbf{x}}

\def \diag {\mathrm{diag}}
\def \dist {\mathrm{dist}}
\def \dom {\mbox{dom}}

\def \z {\mathbf{z}}

\def \y {\mathbf{y}}

\def \N {\mathbb{N}}

\title[Iterative Shrinkage-Thresholding Algorithm]{Efficient Rank Minimization via Solving Non-convex Penalties by Iterative Shrinkage-Thresholding Algorithm}

\coltauthor{\Name{Zaiyi Chen} \Email{czy6516@mail.ustc.edu.cn}\\
 \addr University of Science and Technology of China\\
 }
\begin{document}
\maketitle

\begin{abstract}
Rank minimization (RM) is a wildly investigated task of finding solutions by exploiting low-rank structure of parameter matrices.
Recently, solving RM problem by leveraging non-convex relaxations
has
received significant attention.
It has been demonstrated by some theoretical and experimental work that non-convex relaxation, e.g. Truncated Nuclear Norm Regularization (TNNR) \citep{hu2013fast} and Reweighted Nuclear Norm Regularization (RNNR) \citep{zhong2015nonconvex}, can provide a better approximation of original problems than
convex relaxations.
However, designing an efficient algorithm with theoretical guarantee remains a challenging problem.
In this paper, we propose a simple but efficient proximal-type method, namely Iterative Shrinkage-Thresholding Algorithm(ISTA), with concrete analysis
to solve rank minimization problems with both non-convex weighted and reweighted nuclear norm as low-rank regularizers.
Theoretically, the proposed method
could converge to the critical point under very mild assumptions with the rate in the order of $O(1/T)$.
Moreover, the experimental results on both synthetic data and real world data sets show that proposed algorithm outperforms state-of-arts in both efficiency and accuracy.
\end{abstract}

\section{Introduction}
In recent years, rank minimization technic has been successfully employed in various data mining and machine learning tasks.
For example, In matrix completion \citep{candes2009exact} we assume that partially observed matrix is low-rank; in image denoising \citep{candes2011robust}, backgrounds of videos and faces under varying illumination are regarded as falling into a low rank subspace.
In this paper, we consider the unconstrained objective function, which may be more effective for noisy data.
A unconstrained RM problem
can be formulated as
\begin{equation}\label{origin prob}
  \min_{X\in\R^{m\times n}} \; f(X)+\lambda \cdot \mathrm{rank}(X)
\end{equation}
where the $f$ measures the empirical risk and rank function can be viewed as a regularizer.
It has been proved that solving (\ref{origin prob}) is NP-hard
due to the noncontinuous and nonconvex nature of the rank function.
In order to
tackle this NP-hard problem,
common approaches usually relax the rank function to various regularizers, which can be categorized into convex and non-convex relaxations.

A well known convex relaxation of rank function is nuclear norm $\|\cdot\|_\ast$. It has been shown that nuclear norm is the convex envelope of rank function over the unit ball of the spectral norm in different works~\citep{fazel2001rank,recht2010guaranteed}. In other words, the nuclear norm is the tightest convex approximation of the rank function.
\cite{candes2009exact} have shown that low rank solutions
can be recovered perfectly via minimizing nuclear norm under incoherence assumption in matrix completion problems, which requires that the absolute value of both left and right singular vectors are no greater than some $\mu\in (0,1)$.
Due to the convexity of the nuclear norm, there are many sophisticated algorithms off the shelf.
These algorithms can achieve global optimal solutions efficiently with theoretical guarantees, examples include but not limited to SVT \citep{cai2010singular} and APGL \citep{toh2010APGL} for constrained and unconstrained objectives objectives respectively.
However, the nuclear norm suffers from the major limitation that all singular values are simultaneously minimized,
which implies that large singular values are penalized more heavily than small ones.
More importantly, the underlying matrix may not satisfy the incoherence property in real applications, and the data may be grossly corrupted. In these circumstances,
methods based on nuclear norms usually fail to find a good solution.
Even worse, the resulting global optimal solution may deviate significantly from the ground truth.

A straightforward idea is to use non-convex relaxations to overcome the unbalanced penalization of different singular values.
Essentially, they will penalize larger singular values less and shrink smaller ones,
since the large singular values are dominant in preserving major information of a matrix.
A representative non-convex relaxation 
is the truncated nuclear norm~\citep{hu2013fast}, which is defined as the
sum of the smallest $r$ singular values.
By minimizing only the smallest $r$ singular values, one can avoid penalizing large singular values.
In real world applications, non-convex relaxations usually outperform convex relaxations and could be more robust to noise~\citep{hu2013fast,gu2014weighted}.
The approach of truncated nuclear norm
could achieve more accurate solutions than nuclear norm methods practically. To solve the truncated nuclear norm, a two-layer loop algorithm was proposed that
implies substantial computational difficulty due to the hardness of non-convex objectives. Besides, the number of singular values to be penalized is hard to determine.
Inspired by \citep{candes2008enhancing}, which uses the weighted $\ell_1$ norm to enhance sparsity, \cite{zhong2015nonconvex} proposed a reweighted nuclear norm framework to handle these problems, but only subsequence convergence analysis was proposed, which is an improvable conclusion in the view of optimization.
Stronger results, such as convergence of sequence and faster convergence rate, are required to guarantee the effectiveness of proposed algorithm.

On the other hand, real world data is often obtained from multiple domains rather than a single domain.
For example, in recommendation task a user that rates ``romance" higher than ``horror" in the \emph{movie} domain may have the same preference in the \emph{book} domain.
Intuitively, the appropriately exploited correlations of different domains can be helpful to model the objects better
and improve the quality of prediction.
Existing work in multi-task learning has already shown that the underlying consistency among different domains can reasonably improve the performance of learning models~\citep{singh2008relational,chaudhurietal2009,whiteetal2012}.
However, non-convex regularizers are rarely used in multi-view rank minimization tasks due to the difficulty of optimizing multiple variables simultaneously when regularizers are non-convex.

Fortunately, many recent theoretical works~\citep{attouch2010proximal,attouch2013convergence,bolte2014proximal} show that Kurdyka-\L{}ojasiewicz (KL) property~\citep{kurdyka1998gradients} is an effective tool for non-convex analysis, which also prove that proximal algorithm is applicable to the non-convex and non-smooth functions. KL property shows its strength from the fact that it covers a large class of functions including considerably large family of convex and non-convex functions especially in machine learning and data mining applications~\citep{li2016calculus,yang2016adaptive}. Due to the difficulty of optimizing matrices, there exist very few work using this very powerful tool in rank minimization problem.

Motivated by investigating the special structures of objective functions, we will show that the proposed framework, \textbf{Iterative Shrinkage-Thresholding Algorithm (ISTA)}, is capable to solve rank minimization problem whenever objectives satisfy KL property.
Compared with the state-of-the-art algorithms, ISTA is simpler and faster to converge to high-quality solution, i.e. the critical point, with solid theoretical guarantees.
In the following sections, we demonstrated that this algorithm is also applicable to reweighted nuclear norm regularizer and multi-variable non-convex objective functions in multi-view learning tasks.
Comprehensive experiments show that non-convex regularizers outperform convex relaxation and matrix factorization based methods when solving rank minimization problems in practice, and ISTA is stable and converge faster than existing algorithms in the view of iteration complexity.
\section{Primaries - Theoretical Guarantee of Singular value Regularizer with Non-descending Weights}\label{sec:primary}

In this section, some basic definitions and propositions will be reviewed at first, and then the assumptions of objective functions is given. After that, an elementary result is shown in Theorem~\ref{thm:main1}.

In general, unconstrained rank minimization problem can be formulated as
\begin{align}\label{eqn:problem}
\min_{X\in R^{m\times n}}F(X) = f(X)+g(X)
\end{align}
As mentioned in the previous section, the nuclear norm may not be a good approximation of the rank function due to the fact that it adds up all the singular values equally, which implies that large singular values are penalized more heavily
than small ones. Due to this fact, we consider regularization term $g$ in the following form to penalize larger singular values less
\begin{equation}\label{eqn:penalty}
g(X) = \w^\top\mathbb{\sigma}= \sum_{i = 1}^{n}w_i |\sigma_i|
\end{equation}
where $0<w_1\leq w_2 \leq \ldots \leq w_n$. It is easy to see that using non-descending weights could not only alleviate unreasonable penalization of large singular values, but also enhance the low-rankness of matrix, which lead to a nearly unbiased low-rank approximation by choosing appropriate weights. Unlike existing work that analyze special regularizer like (\ref{eqn:penalty})~\citep{gu2014weighted}, our analyses are based on the property of objective function, and can be extend to a considerable large family of regularizers in rank minimization problem easily.

Without losing generality, we assume $n<m$. The singular value decomposition(SVD) of $X_t$ is denoted by
\begin{align*}
X = U\diag(\sigma(X))V^\top
\end{align*}
where orthogonal matrices $U\in \R^{m\times m}$, $V\in \R^{n\times n}$ are consist of left and right singular vectors respectively, $\sigma(X)\in\R^{n}$ is the vector of singular values, $\sigma_1(X)\geq\sigma_2(X)\geq \ldots \geq \sigma_n(X)$, and $\diag(\sigma)\in \R^{m\times n}$ is a diagonal matrix such that $\diag(\sigma)_{i,i} = \sigma_{i}$ and $\diag(\sigma)_{i,j} = 0$ for $i\neq j$, whose row and column number adjust to the dimensions of left and right hand side matrices. The set of $(U,V)$ that satisfy the SVD is denoted by $O(U,V)$.

Due to the non-convexity of objective functions, we use the limiting differential in our analysis for correctness. The definitions of limiting differential is given as follows.
\begin{definition}\label{def:differential}
\emph{(limiting subdifferential)} Let $E$ be an Euclidian space and $f:E \rightarrow (-\infty, +\infty]$ is a proper lower semi-continuous function.
\begin{itemize}
\item[(i)] {
    The regular subdifferential of $f$ at $\x\in \mathrm{dom} \; f$, denoted by $\hat{\partial }f(\x)$, is the set of vector $\y$ such that
    \begin{align*}
    \liminf_{\z \rightarrow 0}\frac{f(\x+\z) - f(\x) - \langle \y, \z \rangle}{\|y\|} \geq 0.
    \end{align*}
    }
\item[(ii)] {
    The limiting subdifferential of $f$ at $\x\in \mathrm{dom} \; f$, denoted by $\partial f(\x)$ is defined as
    \begin{align*}
    \partial f(\x) = \big\{\y : \exists \x^k \rightarrow \x,\; f(\x^k) \rightarrow f(\x) \; \mathrm{and} \; \y^k \in \hat{\partial } f(\x^k) \rightarrow \y\; \mathrm{as}\; k \rightarrow \infty\big\}
    \end{align*}
    }
\item[(iii)] {
    The directional derivative at $\x\in \mathrm{dom}\; f$ in direction $d\in \R^d$, denoted by $f'(\x,d)$, is defined as
    \begin{align*}
    f'(\x;d) \triangleq \liminf_{\lambda\downarrow 0,\;d' \rightarrow d} \frac{f(\x+\lambda d') - f(\x)}{\lambda}
    \end{align*}
    }
\end{itemize}
\end{definition}
More details could be found in Chapter~8.B~\citep{rockafellar2009variational}.

The distance from any subset $S\subset
\mathbb{R}^{m\times n}$ to any point $X\in \mathbb{R}^{m\times n} $ is defined as
\begin{align*}
\mathrm{dist}(X,S)=\inf\{\|Y-X\|_F, Y\in S\}
\end{align*}
and denote $\Phi$ as the class of all concave and continuous functions $\varphi:[0,\delta) \rightarrow \R_+$, $\delta >0$, such that
\begin{itemize}
\item[(i)] {$\varphi(0) = 0$.}
\item[(ii)] {$\varphi$ is continuous differentiable on $(0,\delta)$, and continuous at 0.}
\item[(iii)] {$\varphi'(x)>0$ for all $x\in (0, \delta)$.}
\end{itemize}

As an important property in the following analysis, the definition
of the Kurdyka-\L{}ojasiewicz (KL)
property~\citep{kurdyka1998gradients} is summarized below.
\begin{definition}\label{def_KL_func}
\emph{(KL property)} Let $f:\mathbb{R}^n\rightarrow(-\infty,+\infty]$ be
proper and lower semi-continuous.
\item[(i)]{A function $f$ has the \emph{KL property} at
    $\bar{\mu}\in\mathrm{dom}\;\partial f:=\{u\in\mathbb{R}^n:
    \partial f(u)\neq\varnothing\} $ if there exist $\delta \in(0,+\infty]$,
    a neighborhood $N$ of $\bar{u}$ and a function $\varphi\in\Phi$, such
    that for all
    \begin{align*}
    u\in N\cap[f(\bar{u})<f(u)<f({u})+\eta],
    \end{align*}
    the following inequality holds
    \begin{align*}
    \varphi'(f(u)-f(\bar{u}))\mathrm{dist}(0,\partial f(u))\geq1.
    \end{align*}
    }
\item[(ii)]{If $f$ satisfies the \emph{KL property} at each point
    of $\mathrm{dom}\;\partial f$, then $f$ is called a \emph{
    KL function}.
    }
\end{definition}
If KL property is hold in the neighbourhood of critical points, one could notice that it builds the connection between the norm of gradient and objective gap to the critical points. This observation makes a major contribution to the convergence analysis in non-convex optimizations.
Another aspect of KL functions we would like to mention is that they are widespread in machine learning applications including both convex and non-convex cases. Following lemma implies that a certain family of functions satisfies KL property. In the next lemma, we can find a sufficient condition of KL property.
\begin{lemma}\label{lem:KL_semialgebraic}
\citep{bolte2007lojasiewicz} Let extended value function $f$ be a proper and lower semi-continuous function. If $f$ is semi-algebraic, then it satisfies the KL property at any point of $\dom\;f$.
\end{lemma}

The semi-algebraic function for Euclidean space is defined as follows.
\begin{definition}\label{def:semi_alge}
\emph{(Semi-algebraic sets and functions)}
\begin{itemize}
\item[(i)]{A subset $S$ of Euclidean space $E$ is a real semi-alggebraic
    set if there exists a finite number of polynomial functions $ \zeta_{ij},\zeta'_{ij}
    : E \rightarrow\mathbb{R}$ such that
    \begin{align*}S=\cup_{j=1}^p{\cap_{i=1}^q{\{u\in E, \zeta_{ij}(u)=0\; \mbox{and} \;\zeta '_{ij}(u)<0\}}}
    \end{align*}
  }
\item[(ii)]{
    A function $r: E\rightarrow(-\infty,+\infty]$ is called
    semi-algebraic if its graph
    \begin{align*}
    \{(u,\xi)\in E\times\R:r(u)=\xi\}
    \end{align*}
    is a semi-algebraic subset of $\mathbb{R}^{n+1}$.
    }
\end{itemize}
\end{definition}

The propositions of semi-algebraic sets (functions) has been summarized in the following proposition.
\begin{proposition}\label{proposition:semi_alge}
\emph{(examples of semi-algebraic functions )} Following class of functions
have KL property.
\begin{enumerate}[label=(\roman*),ref=\roman*]
\item\label{proposition:semi_alge1}{
    Real polynomial functions.
    }
\item\label{proposition:semi_alge2}{
    Finite sums and product of semi-algebraic functions.
    }
\item\label{proposition:semi_alge3}{
    composition of semi-algebraic functions.
    }
\item\label{proposition:semi_alge4}{
    At last, $L_p$ norm is semi-algebraic whenever
    p is rational.
    }
\end{enumerate}
\end{proposition}

To benefit from the KL property and make convergence analysis viable, we will make following assumptions of objective function $F$ in (\ref{eqn:problem}).
\begin{assumption}\label{assump:1}

\begin{enumerate}[label=(\roman*),ref=\roman*]
\item\label{assump:a1}{ $f(X)$ is a proper, lower bounded, $L$-smooth function, that is
\begin{align*}
\|\nabla f(X)-\nabla f(Y)\|\leq L\|X-Y\|,\; \mathrm{for\; all\;} X,Y\in\dom f,
\end{align*}
 and $f(X)$ is semi-algebraic with respect to $X\in\R^{m\times n}$.}
\item\label{assump:a2}{and $g(X)$ is a proper, lower semi-continuous and non-smooth relaxation of rank function, which can be seen as a regularizer with respect to the eigenvalues $\sigma(X)$.}
\item \label{assump:a3} {f(X) is a coercive function, e.g. $f(X) = \frac{1}{2} \| X- Y\|^2_F$}.
\end{enumerate}
\end{assumption}

Here are a few more words on Assumption~\ref{assump:1}.(\ref{assump:a3}).
To prove the convergence of the sequence $\{X_t\}_{t\in\N}$ that is generated by the proximal algorithm, the boundedness of $\{X_t\}_{t\in\N}$ is always required, which also make sure that the result is meaningful and will not go to infinity practically.
Assumption~\ref{assump:1}.(\ref{assump:a3}) automatically guarantees the boundedness of the generated sequence and is satisfied by a large family of empirical risk measurements, e.g. Bregman divergence with bounded $Y$.

To achieve the goal of this paper, we first need to show that the objective function has KL property if it satisfies certain conditions.
\begin{lemma}\label{lem:multi_KL_property}
Objective function (\ref{eqn:problem}) satisfies the KL property,  if $f$ satisfies Assumption~\ref{assump:1}.(\ref{assump:a1}) 
and penalty $g$ is defined as in (\ref{eqn:penalty}).
\end{lemma}
\begin{proof}
When $g$ is
defined as in (\ref{eqn:penalty}), we first investigate the
auxiliary function $g':\mathbb{R}^{m\times n}
\times \mathbb{R}^{m\times m}\times \mathbb{R}^{n\times n}\times\R^{n}
\rightarrow \mathbb{R}$, which satisfies that $g'(X,U,V,\sigma)=g(X)$, $X = U\diag(\sigma)V^\top$ and
$(U,V)\in O(U,V)$.
Its graph in $\mathbb{R}^{n\times m_d}
\times \mathbb{R}^{n\times k}\times \mathbb{R}^{m_d\times k}\times\mathbb{R}$
can be written as
\begin{equation}\label{eq_lemma6_2}
\begin{split}
\{&(X,U,V,\sigma, \xi):\;
\sigma_i\in\mathbb{R}_+,
X=U\mathrm{diag}(\{\sigma_i\})V^\top,\;U^\top U-I=0,\;V^\top
V-I=0,\;\\
&\mbox{and} \;\sum_{i=1}^kw_i\sigma_i-\xi=0, \mathrm{for }w_i\geq 0, i = 1,\ldots, n\}
\end{split}
\end{equation}
We can see that the graph of $g'$ in the subspace
$\{X:X\in\mathbb{R}^{m\times n}\}$ is exactly the graph of $g$. Base on
Definition~\ref{def:semi_alge}, (\ref{eq_lemma6_2}) is a
semi-algebraic set. Then following Tarski-Seidenberg
Theorem~\citep{coste2000introduction}, the graph of $g$ is also a
semi-algebraic set, since its image can be obtained by the
projection of a semi-algebraic set onto the space $\R^{m\times n}\times \R$ by sending $(X,U,V,\sigma, \xi)$ to $(X,\xi)$.

On the other hand, by assumption we have that $f$ is a semi-algebraic function. Thus $F(X) = f(X) + g(X)$, sum of $f$ and $g$, is also a semi-algebraic function by Proposition~\ref{proposition:semi_alge}.(\ref{proposition:semi_alge2}), which complete the proof.
\end{proof}

Given this lemma, we can see that polynomial regularizers with respect to singular values are no longer terrifying to solve. The Iterative Soft-Thresholding Algorithm(ISTA) is a good solution for optimizing non-smooth composite objective defined in (\ref{eqn:problem}).
ISTA has many different name in the area of optimization, including, proximal algorithm(PG), forward-backward splitting and mirror descent \citep{nesterov2013gradient,duchi2010composite,beck2003mirror}. The general step of ISTA is to solve a strongly convex problem iteratively, which is
\begin{align*}
X_{t+1}\in \arg\min_{X\in \mathbb{R}^{n\times m} } & f(X_{t})+\langle\nabla f(X_{t}),X-X_{t}\rangle\\
&+\frac{1}{2\mu}\|X-X_{t}\|_F^2+g(X),\; t\geq0
\end{align*}

Define the shrinkage-thresholding operator, also known as proximal mapping, for $g$ at $M_t$ as
\begin{equation}
\label{eqn:proximal_map}
\begin{split}
P_{g}^\mu (M_{t}) &= \arg\min_{X\in \mathbb{R}^{n_\times m} } m(X_t,X)\\
&=\arg\min_{X\in \mathbb{R}^{n_\times m} }
\frac{1}{2}\|X-M_{t}\|_F^2+\mu g(X), \mu >0
\end{split}\end{equation}

Then the update scheme will be
\begin{align*}
X_{t+1}\in P_{g}^\mu(X_t-\mu\nabla f(X_t))
\end{align*}

\begin{algorithm}[t]
\caption{Iterative Soft-Thresholding Algorithm (ISTA)} \label{alg_pg0}
\begin{algorithmic}
\STATE \textbf{Input}: Observed matrix $Y$, Lipschitz constant $L$.
\STATE \textbf{Initialize}: $X=0$, step size $\mu<1/L$
\FOR {$t=1,2,\ldots$}
    \STATE $M_{t+1} = X_{t}-\mu\nabla_{X_{t}} f(X_{t})$
    \STATE $X_{t+1}=P_{g}^\mu(M_{t+1})$

\ENDFOR
\end{algorithmic}
\end{algorithm}

Based on the following lemma, a solution of
(\ref{eqn:proximal_map}) can be
found in an easier way.

\begin{lemma}\label{lemma_yu}
\citep{zhang2011penalty} Let $\|\cdot\|$ be a unitarily invariant
norm on $\mathbb{R}^{n\times m}$ (i.e., $\|LXR\| = \|X\|$ for any
unitary matrix $L,\; R$) and let $Q:\mathbb{R}^{n \times
m}\rightarrow \mathbb{R}$ be a unitarily invariant function (i.e.,
$Q(LXR)=Q(X)$ for any unitary matrix $L,\; R$ and any $X\in
\mathbb{R}^{n\times m}$). Let $A=U\Sigma V^\top \in
\mathbb{R}^{n\times m}$ be given, 
and $h$ be a
non-decreasing function on $[0,\infty)$. Then $X^* = U
\mathrm{Diag}(\x^*)V^\top$ is a global optimal solution of the
problem
\begin{equation}\label{eq_lem11}
\min_X Q(X)+h(\|X-A\|)
\end{equation}
where $\x^*$ is the global optimal solution of the problem
\begin{equation}\label{eq_lem12}
\min_\x
Q( \diag(\x))+h(\| \diag(\x)-\Sigma\|)
\end{equation}
\end{lemma}

It is worthwhile to further our discussion regarding this lemma. If
we set $g(\x)= Q(\mathrm{diag}(\x))$, then $g$ can be viewed
as an extension of the \emph{symmetric gauge function} for $Q$, in
which case $g$ is a function on $\mathbb{R}^n$ whose value is
invariant under permutations but could be variant under sign changes
of components. Due to these facts, we can view a unitarily invariant
function $Q$ as an extension of a unitarily invariant norm.
More examples and analyses of symmetric gauge functions in normed vector space can be found in~\citep{lewis2003mathematics}. As a
result, if the empirical risk $ f$ is measured by a norm in vector
space, or more generally by a unitarily invariant function, and
non-smooth regularization terms $g$ penalize the unitarily
invariant norms of variables non-decreasingly, Lemma~\ref{lemma_yu}
indicates that the shrinkage operator could be computed in an easier way.
This observation can be applied to (\ref{eqn:penalty}), which gives use following corollary.

\begin{corollary}\label{corollary_proximal_map} Assume that $g$ is a function of singular values of $X$.
The shrinkage operator $P_g^\mu$ of $g$ in the form of (\ref{eqn:penalty})
can be computed as:
\begin{equation}\label{eqn:proximal1}
\begin{split}
&P_{g}^\mu(M_{t+1})=U_{t+1} \mathrm{Diag}(\x^*) V_{t+1}^\top,\\
&x_i^*=(\sigma_i(M_{t+1})-\mu w_{i})_+ \triangleq \left\{
    \begin{aligned}
    &\sigma_i(M_{t+1})-\mu w_{i,t},\; \mbox{if}\;\sigma_i(M_{t})>w_{i}\\
    &0,\qquad\qquad\quad\;\; \mbox{otherwise}
    \end{aligned}
    \right.
\end{split}
\end{equation}
where $M_{t+1} = U_{t+1}\diag(\sigma_{t+1})V_{t+1}^\top$, $(U,V)\in O(U_{t+1},V_{t+1})$.
\end{corollary}
\begin{proof}
It is obvious that the Frobenius norm is unitarily invariant,
$h(\theta)=\theta^2/2$ is nondecreasing on $[0,\infty)$, and
penalties defined as in (\ref{eqn:penalty}) are also unitarily
invariant and separable for each singular value. Given that all
assumptions of Lemma \ref{lemma_yu} are satisfied, the proximal map of $g$ can be calculated by
\begin{equation}\label{eq_theo1}
\begin{split}
&P_{g}^\mu(M_{t+1}) = U_{t+1}\mathrm{\diag}(\x^*)V_{t+1}^\top,\\
&\x^* = \arg\min_{\x\in
\mathbb{R}^{n} \geq 0}\frac{1}{2}\|\mathbb{\sigma}(M_{t+1})-\x\|_F^2+\mu \w_t^\top |\x|.
\end{split}
\end{equation}
where $|\x| = [|x_1|,\ldots, |x_n|]$. Using the shrinkage operator \citep{parikh2014proximal}, we can conclude that (\ref{eqn:proximal1}) is the analytical solution of (\ref{eq_theo1}), which complete the proof.
\end{proof}

The update scheme (\ref{eqn:proximal1}) gives a shrinkage-threshold step to the singular value of $X_t$ in each step. Equipped with these results, we can conclude following result.

\begin{thm}\label{thm:main1}
Suppose that all conditions in Assumption~\ref{assump:1} are hold. Let $\{X_t\}_{t \in \N}$ be the sequence generated by Algorithm~\ref{alg_pg0}, which is bounded. We have that
\begin{equation}\label{eqn:thm1_sum_bound}
\sum_{t = 0}^\infty \|X_{t+1}-X_{t}\|_F <\infty
\end{equation}
and $\{X_t\}_{t \in \N}$ converges to a critical point $X_*$ of $F$.
\end{thm}
The proof of above theorem is included in the Appendix.~\ref{sec:app1} for completeness.

Further more, if desingularizing function $\varphi$ for defining KL property could be chosen to be of the form
\begin{align*}
\varphi (s) = cs^{1-\alpha},
\end{align*}
where $c>0$ and $\alpha \in (0,1]$, then as shown in \citep{attouch2009convergence}, the convergence rate of $X_t$, which is measured by $\|X_t-X_*\|_F$, depends on $\alpha$, which can be summarized as
\begin{enumerate}[label=(\roman*),ref=\roman*]\label{fact:1}
\item\label{fact:1.1}{ If $\alpha = 0$, then $\{X_t\}_{t\in \N}$ converges in finite steps.}
\item\label{fact:1.2}{If $\alpha \in (0,\frac{1}{2}]$, then $\{X_t\}_{t\in \N}$ converges locally linearly, which means there exist $\omega >0$ and $\tau \in [0,1)$ such that $\|X_t-X_*\|_F\leq \omega \tau ^t$, when $X_t$ is in a small enough neighborhood of $X_*$.}
\item\label{fact:1.3}{If $\alpha \in (\frac{1}{2}, 1)$, hen $\{X_t\}_{t\in \N}$ converges locally sublinearly, which means there exist a $\omega >0$ such hat $\|X_t-X_*\|_F\leq \omega t^{-\frac{1-\alpha}{2\alpha-1}}$,when $X_t$ is in a small enough neighborhood of $X_*$.
    }
\end{enumerate}

\textbf{Remark}: (i) Compared with factorization based methods, the benefit of proximal algorithm is that it does not need any assumption about the singular gap of $\{X_t\}_{t\in\N}$, due to the fact that Lanczos method could solve the eigenvalue decomposition efficiently even if the eigengap of $X^\top X$ is zero \citep{kuczynski1992estimating}.
Meanwhile, (\ref{eqn:thm1_sum_bound}) tells that benefitting from the KL property, the first order guarantee proximal gradient $\frac{1}{\eta}\|X_T-X_{T+1}\|_F$ converges on the order of $O\big(\frac{1}{T}\big)$ with respect to the total iteration number $T$, which is even better than the classical results $O\big(\frac{1}{\sqrt{T}}\big)$ for general non-convex functions, as discussed in the Remark~7 of \citep{attouch2010proximal}.

(ii) As a general framework, we can see that Algorithm~\ref{alg_pg0} is applicable to a large family of non-convex rank minimization problems, such as truncated nuclear norm.
In Section 5 of \citep{hu2013fast}, the authors proposed a two loop algorithm named TNNR-APGL to solve (\ref{eqn:problem}) though the objective does not give out explicitly in their paper.
In spite of $O\big(\frac{1}{T^2}\big)$ convergence of inner loop for proximal gradient $\frac{1}{\eta}\|X_{l,T}-X_{l,T+1}\|_F$, the total convergence rate of TNNR-APGL is still unknown since it contains an outer loop which changes $A_l$ and $B_l$ in every iteration.
Thus compared with TNNR-APGL, Algorithm~\ref{alg_pg0} is a simpler single loop algorithm with explicit convergence rate. In experiments, we will show that ISTA is faster than TNNR-APGL almost always.

(iii) Above fact (\ref{fact:1.1}) to (\ref{fact:1.3}) give us a different but strong result that the convergence rate of $\|X_t-X_*\|_F$ will be known whenever $\alpha < 1$ in the KL inequality is given.
Even if calculating the exact exponent $\alpha$ for general KL function is a very difficult problem as shown in \citep{li2016calculus,necoara2015linear}, We can still give an upper bound of $\alpha$ that $\alpha \leq 1-2\times 3^{m\times n}$ by the main theorem of \citep{d2005explicit}.
More importantly, the convergence rate of Algorithm~\ref{alg_pg0} will be the same no matter $\alpha$ is known or not. To the best of our knowledge, this is the first work that reveals $\|X_t-X_*\|_F$ type convergence rate for RM problem.

\section{Enhancing Theoretical Guarantee for Reweighted Singular Value Regularizer}

\begin{algorithm}[t]
\caption{Iterative Shinkage-Thresholding and Reweighted Algorithm} \label{alg_spg}
\begin{algorithmic}
\STATE \textbf{Input}: Observed matrix $Y$, Lipschitz constant $L$, approximation parameter $\varepsilon $.
\STATE \textbf{Initialize}: $X=0$, step size $\mu<1/L$
\FOR {$t=1,2,\ldots$}
    \STATE $M_{t+1} = X_{t}-\mu\nabla_{X_{t}} f(X_{t})$
    \STATE $X_{t+1}=P_{u_t}^\mu(M_{t+1})$, where $u_t$ is defined in (\ref{eqn:upper_bound}) and shrinkage operator $P_{u_t}^\mu (M_{t+1})$ is computed as (\ref{eqn:proximal1}).
\ENDFOR
\end{algorithmic}
\end{algorithm}

Although we can make a more reasonable penalty for rank minimization problem and solve it efficiently by Algorithm~\ref{alg_pg0} based on the analysis in previous section, the results may not be satisfactory without well-tuned penalty parameters $\w$ when facing a real-world problem.
In this section, we will extend our analysis to solve a more sophisticated penalty that could reduce the requirement of penalty parameters. 
More precisely, a more complicated case of (\ref{eqn:penalty}) is to use reweighting strategy in defining $w_{i,t}$ in each iteration.  As proposed in \citep{candes2008enhancing}, reweighting strategy could outperform LASSO regularizer in finding sparse solution. 
The intuition behind this fact is that the reweighted $l_1$ norm makes a better approximation of cardinality  of support function in $\R^d$ in the view of its graph. 
In the light of this fact, \cite{zhong2015nonconvex} proposed a reweighted nuclear norm for the sake of low-rank structure in matrix completion task. Specifically, they used the iterative shrinkage-thresholding method to solve
\begin{equation}\label{eqn:proximal_problem2}
X_{t+1} = P_{u_t}^\mu (X_t-\nabla f(X_t)) = \arg\min_X \frac{1}{2\mu} \|X-(X_t- \mu \nabla f(X_t))\|_F^2+\sum_{i=1}^n w_{t,i}|\sigma_i(X)|
\end{equation}
iteratively, where
\begin{equation}\label{eqn:penalty2}
w_{t,i} = \frac{p}{(\sigma_i(X_t)+\varepsilon)^{1-p}}
\end{equation}
$\varepsilon>0$ is a negligible constant and $ 0<p<1$. This successive procedure is described in Algorithm~\ref{alg_spg}. The iterative reweighted algorithm falls in the general class of Majorization Minimization \citep{hunter2004tutorial}. To see this, we can consider following penalty
\begin{equation}\label{eqn:penalty3}
g(X) = \sum_{i=1}^n (|\sigma_i(X)|+\varepsilon)^p
\end{equation}
which is a continuous, differentiable concave function with respect to $|\sigma_i(X)|$ for $i=1,\ldots,n$. The absolute function is included for the concreteness of implying Lemma~\ref{lemma_yu}, though it does not change anything since singular values are all non-negative.
One can easily find a linearized upper bound for (\ref{eqn:penalty3}) at $|\sigma(X_t)|$ whenever $p\in(0,1) $, which is
\begin{equation}\label{eqn:upper_bound}
u_t(X) = u(X,X_t) = g(X_{t})+\sum_{i=1}^m \frac{p}{(\sigma_i(X_{t}) + \varepsilon )^{1-p}} (|\sigma(X)|-|\sigma(X_t)|)
\end{equation}
Observing that the coefficients of linear term are denoted by $w_{t,i}$, as shown in (\ref{eqn:penalty2}), we can view Algorithm~\ref{alg_spg} as a procedure that iteratively minimizes a upper bound function not only for $f$, but also for $g$ at $X_t$. As a consequence, the real objective function in (\ref{eqn:problem}) turns into the form
\begin{equation}\label{eqn:problem2}
\min_{X\in\R^{m\times n}} F(X)=f(X) + \sum_{i=1}^n (|\sigma_i(X)|+\varepsilon)^p
\end{equation}
We can see that $g$ defined in (\ref{eqn:penalty3}) makes a better approximation of rank function compared to nuclear norm as $p$ tends to zero though it is not reachable.

In \citep{zhong2015nonconvex}, the authors made some efforts to show the convergence of Algorithm~\ref{alg_spg} but only subsequence convergence was obtained. In a more previous work, \citep{attouch2010proximal} shows that reweighted $l_1$ norm can be viewed as a alternating minimization problem which alternatively solve $\w_t$ and $\x_t\in \R^d$. Unlike previous work, in the rest of this section we will prove the convergence of Algorithm~\ref{alg_spg} for $X\in\R^{m\times n}$ and not consider it as an alternating minimization process.

To make sure that $F(X)$ is still a KL function, we assume $p$ is a rational number according to Proposition~\ref{proposition:semi_alge}.(\ref{proposition:semi_alge1}), (\ref{proposition:semi_alge2}), (\ref{proposition:semi_alge3}) and (\ref{proposition:semi_alge4}).

\begin{lemma}
\label{lem:limit_point2}
\emph{(Properties of limit$(X_0)$)} Assume that $p$ is a rational number and Assumption~\ref{assump:1} is hold. Let
$\{X_t\}_{t\in\mathbb{N}}$ be the sequence generated by Algorithm~\ref{alg_pg0}
with start point $X_0$. The following assertions hold.
\begin{itemize}
\item[(i)]{$\varnothing\neq\emph{limit}(X_0)\subset \mathrm{crit}\;(F)$,
    where $\mathrm{crit}\;(F)$ is the set of critical points of $F$.}
\item[(ii)]{We have
    \begin{equation}\label{eqn:limit2_point1}
    \lim_{t\rightarrow\infty}{\mathrm{dist}(X_t,\mathrm{limit}(Z_0))}=0.
    \end{equation}}
\item[(iii)]{$\emph{limit}(X_0)$ is a non-empty, compact and
connected set.}
\item[(iv)]{The objective $F$ is finite and constant on $\emph{limit}(X_0)$.}
\end{itemize}
\end{lemma}
The proof of this lemma in included in the Apendix.~\ref{sec:app2}.

Since all tools used in the proof of Theorem~\ref{thm:main1} has been verified by Lemma~\ref{lem:limit_point2} and the proofs therein. we can conclude the following theorem similar to Theorem~\ref{thm:main1}.

\begin{thm}\label{thm:main2}
Suppose that $p$ is rational and all conditions in Assumption~\ref{assump:1} are hold. Let $\{X_t\}_{t \in \N}$ be the sequence generated by Algorithm~\ref{alg_spg}, which is bounded. We have that
\begin{equation}\label{eqn:thm2_sum_bound}
\sum_{t = 0}^\infty \|X_{t+1}-X_{t}\|_F <\infty
\end{equation}
and $\{X_t\}_{t \in \N}$ converges to a critical point $X_*$ of $F$ defined by (\ref{eqn:problem2}).
\end{thm}

As mentioned in \citep{candes2008enhancing}, it is very important to give a good starting point $X_0$ for Algorithm~\ref{alg_spg}.
Essentially, the performance of reweighted nuclear norm is effected by the starting point significantly as the weights $\w_{t+1}$ relies on $X_t$,
and the bad starting point could give a bad guess of ideal penalties for singular values and mislead following steps.
Due to the fact that the reconstruction error of problem (\ref{eqn:problem}) with non-convex penalty (\ref{eqn:penalty}) is smaller than convex penalties in experiments, we initialize Algorithm~\ref{alg_spg} with the solution given by Algorithm~\ref{alg_pg0}.
On the other hand, smaller $p$ may not always give better results by running Algorithm~\ref{alg_spg} as the penalty becomes more "non-convex" and Algorithm~\ref{alg_spg} is more possible to be stuck in "poor" critical points, as we can see in the experiments.

\section{Solving Rank Minimization Problem with Multiple Matrices}
\begin{algorithm}[t]
\caption{Alternating Iterative Shrinkage-Thresholding Algorithm (Alter-ISTA) for Multiple
Variables } \label{alg_pg1}
\begin{algorithmic}
\STATE \textbf{Input}: Observed matrices $\{Y^d\}$ for each view and the largest
Lipschitz constant $L_{\max}$.
\STATE \textbf{Initialize}: $X^d=0$,
$\mu<1/L_{\max}$
\FOR {$t=1,2,\ldots$}
    \FOR {$d=0,\ldots,D$}
    \STATE $M^d_{t} = X^d_{t}-\mu\nabla_{X^d_{t}}\ell(X^d_{t})$
    \STATE $X^d_{t+1}=P_{h_d}^\mu(M^d_{t})$
    \ENDFOR
\ENDFOR
\end{algorithmic}
\end{algorithm}
In many data mining and machine learning tasks, we may have to optimizing more than one data matrices simultaneously, which are more complicated problems compared to what we have discussed in the previous sections.
For instance, \cite{singh2008relational} and \cite{zhangetal2010} show that the performance of learning models can be markedly improved by exploiting the data from multiple domains.
To be specific, we will focus on a matrix completion problem in multiple domains and show that the global convergence is still achievable even if the objective is non-convex and with the penalties defined as in previous sections.

In multi-domain scenarios, given observations indexed by
$\{\Omega_d,d=1,\ldots,D\}$ from $D$ domains:
$\{Y^d\in\mathbb{R}^{n\times m_{d}},\; d=1,\ldots,D\}$ where
matrices $\{Y^d\}$ are aligned in rows, correlations among the
multiple domains can be exploited to improve the quality of matrix
completion. Specifically, we assume there exist consistency shared
among multiple domains as well as independent patterns for each
separate domain. In the case of multi-domain recommendation where
matrices $\{Y^d\}$ correspond to rating matrices on different types
of items such as \emph{user}$\times$\emph{movie} and
\emph{user}$\times$\emph{book}, it is natural to assume that
\emph{users} have some mutual interests across domains, as well as
some distinct interests in each domain.

Consider the latent factors of \emph{users} and \emph{items} by
factorizing a rating matrix $X=UV^\top$, where $U$ and $V$
correspond to low-rank \emph{user$\times$latent factor} and
\emph{item$\times$latent factor} matrices. In multiple domains, the
consistent patterns can be represented by a shared
\emph{user$\times$latent factor} matrix $U$. As a consequence, the
observations in the $d$-th domain can be factorized as
$Y^d_{\Omega_d}=(U^dV^{d^\top}+\tilde{U}{\tilde{V}}^{d^\top}+\varepsilon^d)_{\Omega_d}$,
where $\tilde{X^d}=\tilde{U}{\tilde{V}}^{d^\top}$ represents shared user interests on the
$d$-th domain; and $X^d=U^dV^{d^\top}$
corresponds to domain specific user preference. The rating behaviors
of shared user interests on various domains can be summarized in the
matrix $X^0=[\tilde{X}^1,\ldots,\tilde{X}^D]=\tilde{U}\cdot[\tilde{V}^{1^\top}, \ldots, \tilde{V}^{D^\top}]$,
which is a horizontal concatenation of $\{\tilde{X}^d\}$. To learn the
shared and domain specific user interests, we apply a general
singular value regularizer $h_0 $ on $X^0$, and $h_d $ on
${X^d}$ for $d=1,\ldots,D$.
Then the optimization problem can be formulated as follows
\begin{equation}\label{eqn:problem3}
\min_{\mbox{\tiny $\begin{array}{c}
\{{X}^d\}\R^{m_d\times n_d},\\
d=0,\ldots,D
\end{array}$
}} F(X^0,\ldots, X^D) = f({X}^0,\ldots,{X}^D)+\sum_{d=0}^D
g^d({X}^d)
\end{equation}
where $D>1$.

Similar to Assumption~\ref{assump:1}, we make following assumptions for $F$ as defined in (\ref{eqn:problem3}).
\begin{assumption}\label{assump:2}
\begin{enumerate}[label=(\roman*),ref=\roman*]
\item\label{assump:b1} {Multivariate function $f(X^0,\ldots,X^D)$ is lower bounded,
    continuously differentiable, and has \emph{$L_d$-Lipschitz} continuous
    partial gradient with respect to each $X^d$, that is
    \begin{align*}
    \|\nabla_{X^d_1}f(X^0,\ldots, X^d_1, \ldots, X^D) - \nabla_{X^d_2}f(X^0,\ldots, X^d_2, \ldots, X^D)\|_F \leq L_d\|X^d_1-X^d_2\|_F,
    \end{align*}
    for all $X^d_1,\;X^d_2\in \R^{m_d\times n_d}$, $d=0,\ldots, D$.}

\item\label{assump:b2} {$\nabla f$ is Lipschitz
    continuous on bounded subsets of $\mathbb{R}^{n_0\times
    m_0}\times\ldots\times\mathbb{R}^{n_D\times m_D} \rightarrow
    \mathbb{R}$. That is, for each bounded subsets $B_0\times\ldots\times
    B_D$, there exists a constant $M>0$, such that for all $(X^0_1,\ldots,X^D_1),\;(X^0_2,\ldots,X^D_2)\in B_0
    \times\ldots\times B_D$, the following inequality holds:
    \begin{align*}
    \begin{split}
    \big\|\big(&\nabla_{X^0}\ell(X^0_1,\ldots,X^D_1)-\nabla_{X^0}\ell(X^0_2,\ldots,X^D_2),\ldots,\\
       &\nabla_{X^D}\ell(X^0_1,\ldots,X^D_2)-\nabla_{X^D}\ell(X^0_2,\ldots,X^D_2)\big)\big\|_F\\
    \leq &M\big\|\big(X^0_1-X^0_2,\ldots,X^D_1-X^D_2\big)\big\|_F.
    \end{split}
    \end{align*}}
\item\label{assump:b3} {Each penalty component $g^d:\mathbb{R}\rightarrow \mathbb{R}$
is a proper, lower bounded function.
}

\item\label{assump:b4}{
    Function $F$ has the KL property.
}
\end{enumerate}
\end{assumption}
It is easy to show that (\ref{assump:b1}) and (\ref{assump:b2}) are satisfied whenever $f$ is $C^2$ continuous.

When $g^d$ is defined as (\ref{eqn:penalty}),
following the similar analysis as in Section~\ref{sec:primary}, the convergence property can be summarized by following theorem and the proof is included in the Appendex.~\ref{sec:app3}.

\begin{thm}\label{thm:main3}
 If $\{g^d\}_{d=0,\ldots,D}$ are defined as (\ref{eqn:penalty}), all conditions in Assumption~\ref{assump:2} are hold and
a step size is chosen such that $\mu<1/L_{\max}$ where $L_{\max}$ is
the maximum of $\{L_d\}_{d=0,\ldots,D}$, then the sequence
$\{(X^0_{t},\ldots,X^D_{t})\}_{t\in \mathbb{N}}$ generated by any
alternative proximal gradient method, such as Algorithm
\ref{alg_pg1}, will have finite length and converge to a critical
point of (\ref{eqn:problem3}). That is
\item[(i)]{
    The sequence $\{Z_t\}_{t\in\mathbb{N}}$ has
    finite length,
    \begin{equation}\label{eq_theo2_1_1}
    \sum_{t=1}^\infty\|Z_{t+1}-Z_t\|_F<\infty
    \end{equation}
    }
\item[(ii)]{
    The sequence $\{Z_t\}_{t\in\mathbb{N}}$ converges
    to a critical point $Z^*$ of (\ref{eqn:problem3}).
    }

\end{thm}

It is not hard to see that multivariate  reweighted penalties, e.g. $g^d$ is defined as (\ref{eqn:penalty2}) for $d=0,\ldots, D$, can also be solved  by revising Algorithm~\ref{alg_spg} into an alternating framework, which is given in Algorithm~\ref{alg_spg1}. This result is summarized by following corollary.

\begin{algorithm}[t]
\caption{Alternating Iterative Shrinkage-Thresholding and Reweighted Algorithm (Alter-ISTRA) for Multiple
Variables } \label{alg_spg1}
\begin{algorithmic}
\STATE \textbf{Input}: Observed matrices $\{Y^d\}$ for each view, the largest
Lipschitz constant $L_{\max}$ and approximation parameter $\varepsilon$.
\STATE \textbf{Initialize}: $X^d=0$,
$\mu<1/L_{\max}$
\FOR {$t=1,2,\ldots$}
    \FOR {$d=0,\ldots,D$}
    \STATE $M^d_{t} = X^d_{t}-\mu\nabla_{X^d_{t}}\ell(X^d_{t})$
    \STATE $X^d_{t+1}=P_{u^d_t}^\mu(M^d_{t})$ where $u^d_t$ is defined in (\ref{eqn:proximal1}).
    \ENDFOR
\ENDFOR
\end{algorithmic}
\end{algorithm}
\begin{corollary}
Suppose that $p$ is rational and all conditions in Assumption~\ref{assump:2} are hold. Let $\{Z_t\}_{t \in \N}$ be the sequence generated by Algorithm~\ref{alg_spg1}, which is bounded. We have that
\begin{equation}\label{eqn:thm3_sum_bound}
\sum_{t = 0}^\infty \|Z_{t+1}-Z_{t}\|_F <\infty
\end{equation}
and $\{Z_t\}_{t \in \N}$ converges to a critical point $Z_*$ of $F$ with $g^d$ defined by (\ref{eqn:penalty2}) for all $d=0,\ldots,D$.
\end{corollary}

\section{Computational Difficulties and Solutions}\label{sec:accelerate}
The most time consuming part of above algorithms is an SVD computation in each
iteration, which makes its scalability an issue in real-world
applications. To accelerate the convergence, we use line-search to
choose $\eta_{t}$ instead of a constant step size. Specifically,
one can decrease $\eta_{t}$ by $\eta_{t}=\mu\eta_{t-1},\;
\mu<1$ and make sure the inequality
\begin{equation}\label{eq_linesearch}
\ell(X_{t+1})\leq \ell(X_{t})-\sigma\|X_{t+1}-X_{t}\|_F^2,\sigma\in(0,1)
\end{equation}
is strictly satisfied until $\eta_{t+1}<1/L_{\max}$, which is known as backtracking \citep{beck2009fast}. In the
meantime, a larger step size would lead to fewer positive components
when solving shrinkage-thresholding problems, which implies lower
rank of $X_{t+1}$ and fewer singular values to compute. The
convergence is still promised by this strategy.

Furthermore, as we observed from the convergent sequence in experiments, the rank
would start and decrease from a large number which entails inefficient
computation at the beginning. We use a decreasing sequence
$\{\tau_0,\ldots,\tau_l\}$ with $\tau_l = 1$ to reduce the number
of singular values above the threshold. In each iteration, the
proximal map is computed as $P_h^{\tau_{(t)}\mu_{(t)}}(M_{(t)})$. It
is clear that the convergence property is not affected as
$\{\tau_i\}$ is a finite sequence.
Besides, stochastic SVD
\citep{shamir2015stochastic} is also a practical approach to compute
singular values for large datasets.

\section{Experiments}
In this section, we conduct experiments on the matrix completion task
with both synthetic and real data.


\subsection{Synthetic Data}
We first compare the Algorithm~\ref{alg_pg0} and Algorithm~\ref{alg_spg}, ISTA and ISTRA respectively, with four commonly used matrix completion methods,
among which SVT \citep{cai2010singular}, APGL \citep{toh2010APGL} are based on the nuclear norm, SVP \citep{jain2010guaranteed}
adopts nuclear norm with affine constrains,
and TNNR \citep{hu2013fast}\footnote{The code is from \url{https://github.com/xueshengke/TNNR}.},
denoted by TNNR\_origin in this section,
is the state-of-the-art nonconvex algorithm using  the truncated nuclear norm.
The best results of algorithms in \citep{hu2013fast} are reported to make a fair comparison and to insure the convergence, we enlarge the maximum number of iteration of inner loop from 200 to 1000.
All algorithms are well tuned, e.g. penalty parameters are chosen between $[1:10:1000]$, to achieve the best performances. The stopping criterion is $\|X_{t+1}-X_t\|_F/\|X_0\|_{\Omega}\leq 10^{-4}$, where start point $X_0$ is chosen to be observed matrix for all methods to make a fair comparison.

We generate synthetic $m\times n$ matrix by $M+aZ$, where $M$ is the ground truth matrix of rank $b$,
$Z$ is Gaussian white noise, and $a$ controls the noise level.
$M$ is generated by $M=AB$, where $A\in \mathbb{R}^{m\times b}$ and $B\in \mathbb{R}^{b\times n}$ both have i.i.d. Gaussian entries.
The set of observed entries $\Omega$ is uniformly sampled. We adopt the widely used measure called relative error ($RE = \|X^\ast-M \|_F/\|M\|_F$)
to evaluate the accuracy of the recovered matrix $X^\ast$.
All reported results are the averages of 10 rounds to avoid the negative effects of randomness. When observed ratio is less than $20\%$, we also tuned $\w_{1:rank}$ between $[1:1:10]$ for ISTA and ISTRA to achieve better performances.

\begin{figure*}[htbp]
 \centering
\includegraphics[scale=0.28]{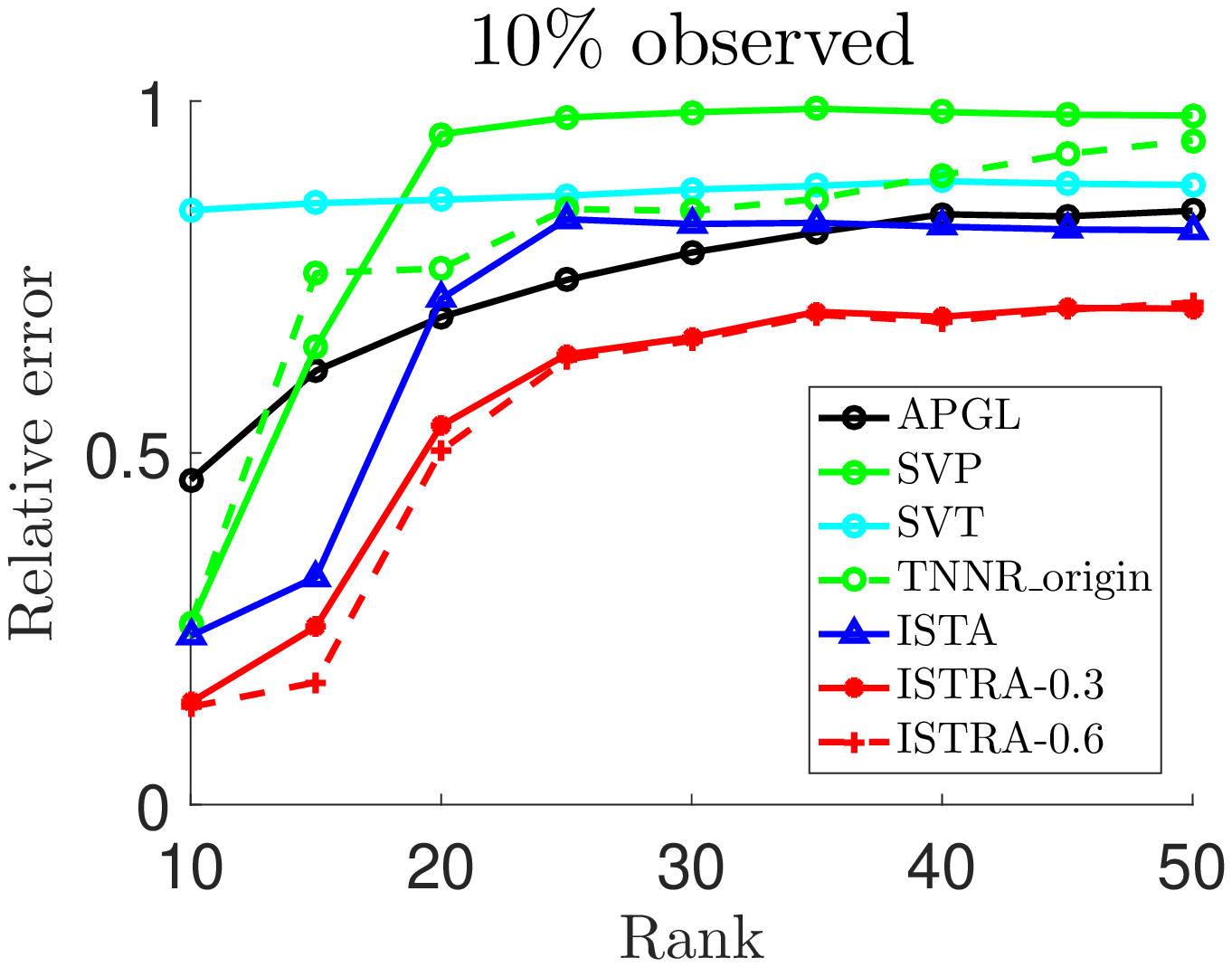}\hspace*{\fill}
\includegraphics[scale=0.28]{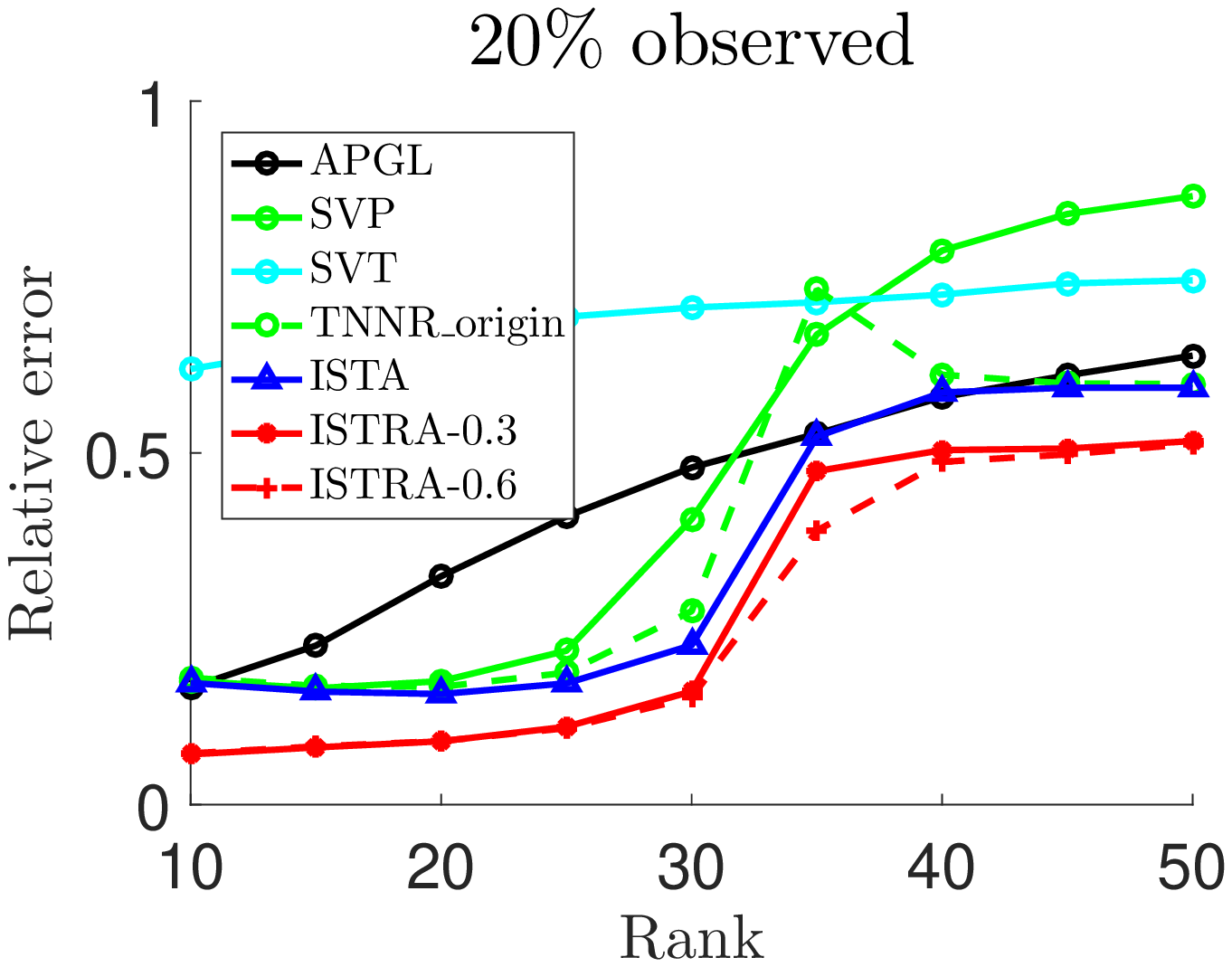}\hspace*{\fill}
\includegraphics[scale=0.28]{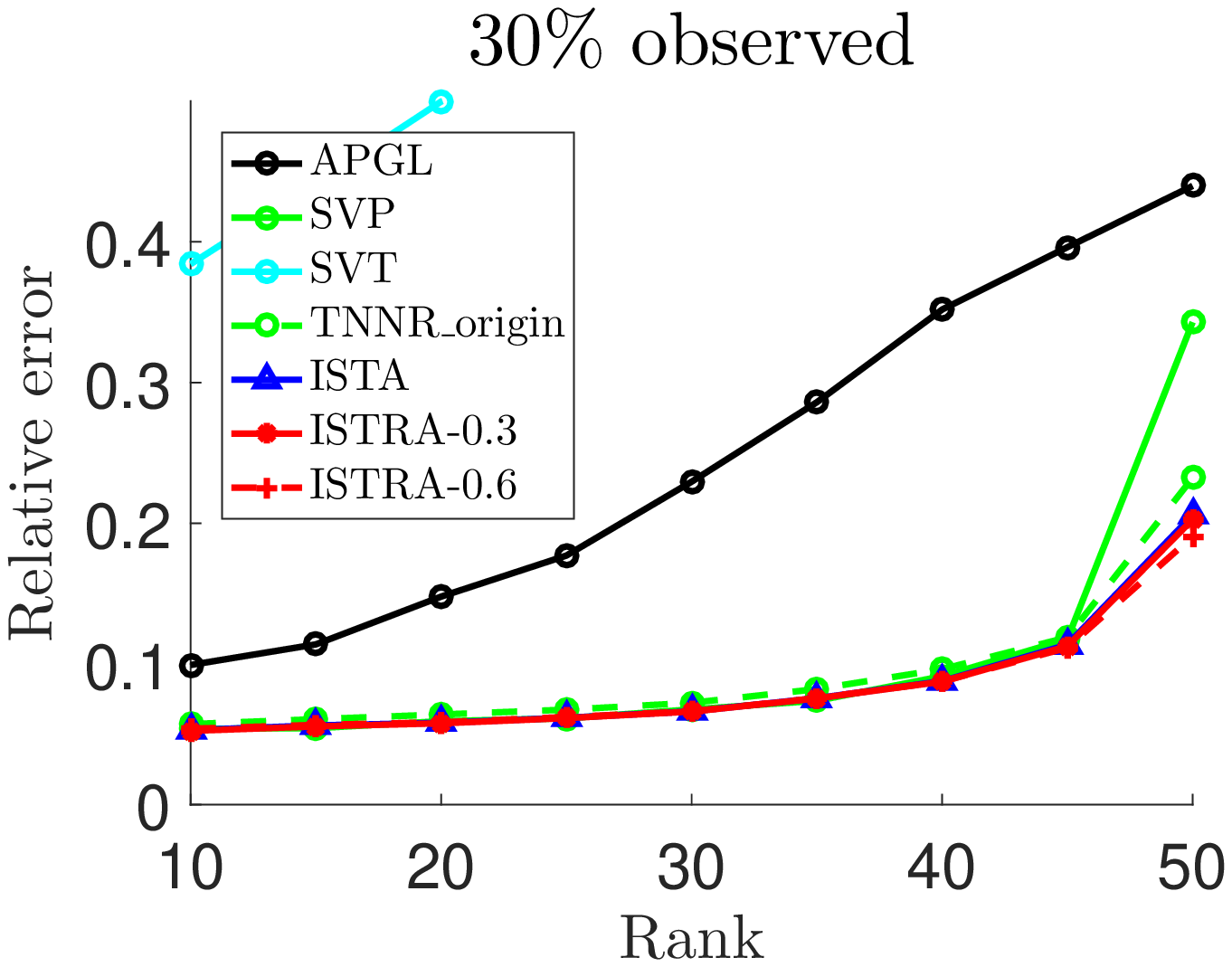}\hspace*{\fill}
\includegraphics[scale=0.28]{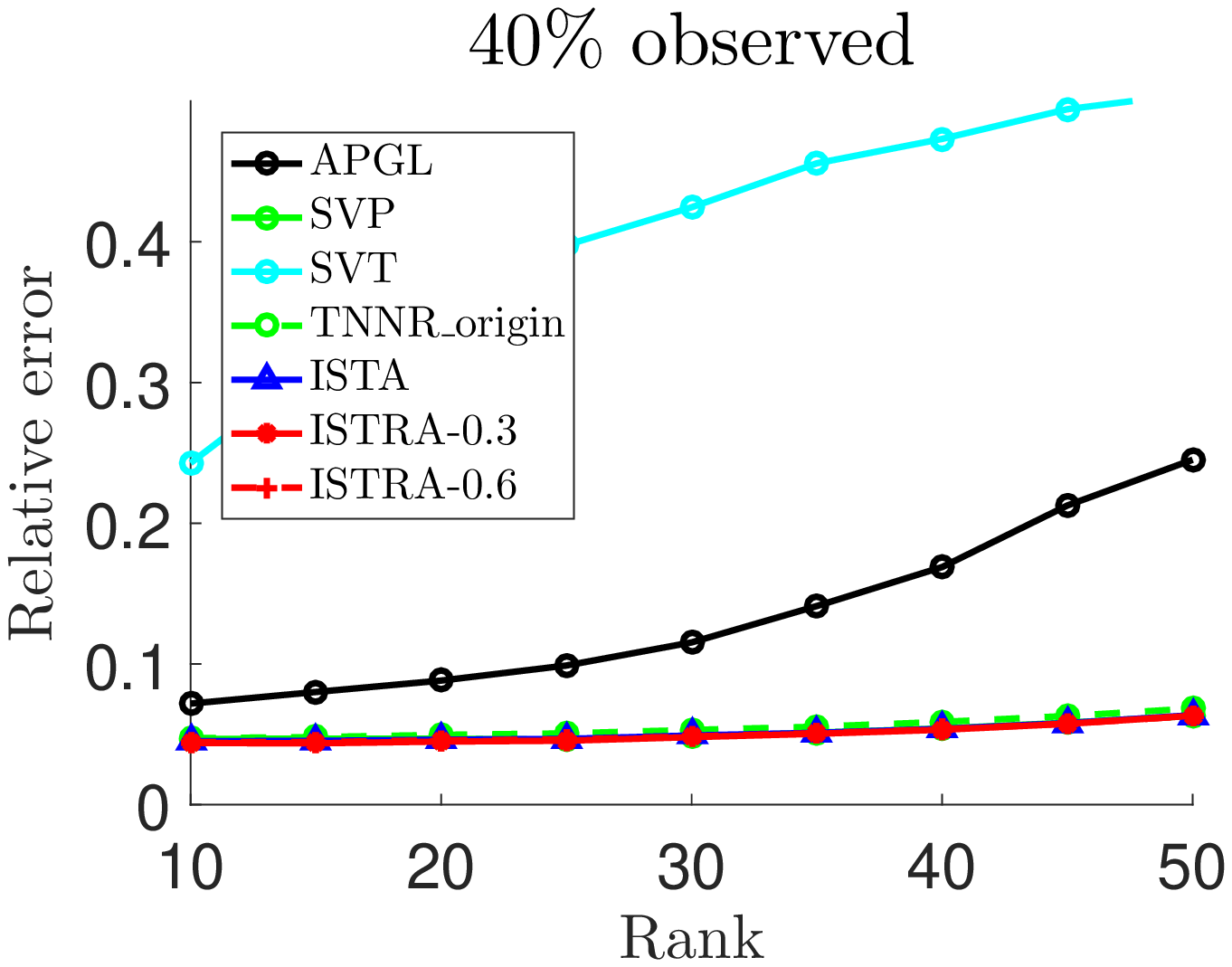}
\caption{{Relative error versus rank
with different observations}}
\label{Fig.all_rank}
\end{figure*}
\begin{figure*}[htbp]
 \centering
\includegraphics[scale=0.28]{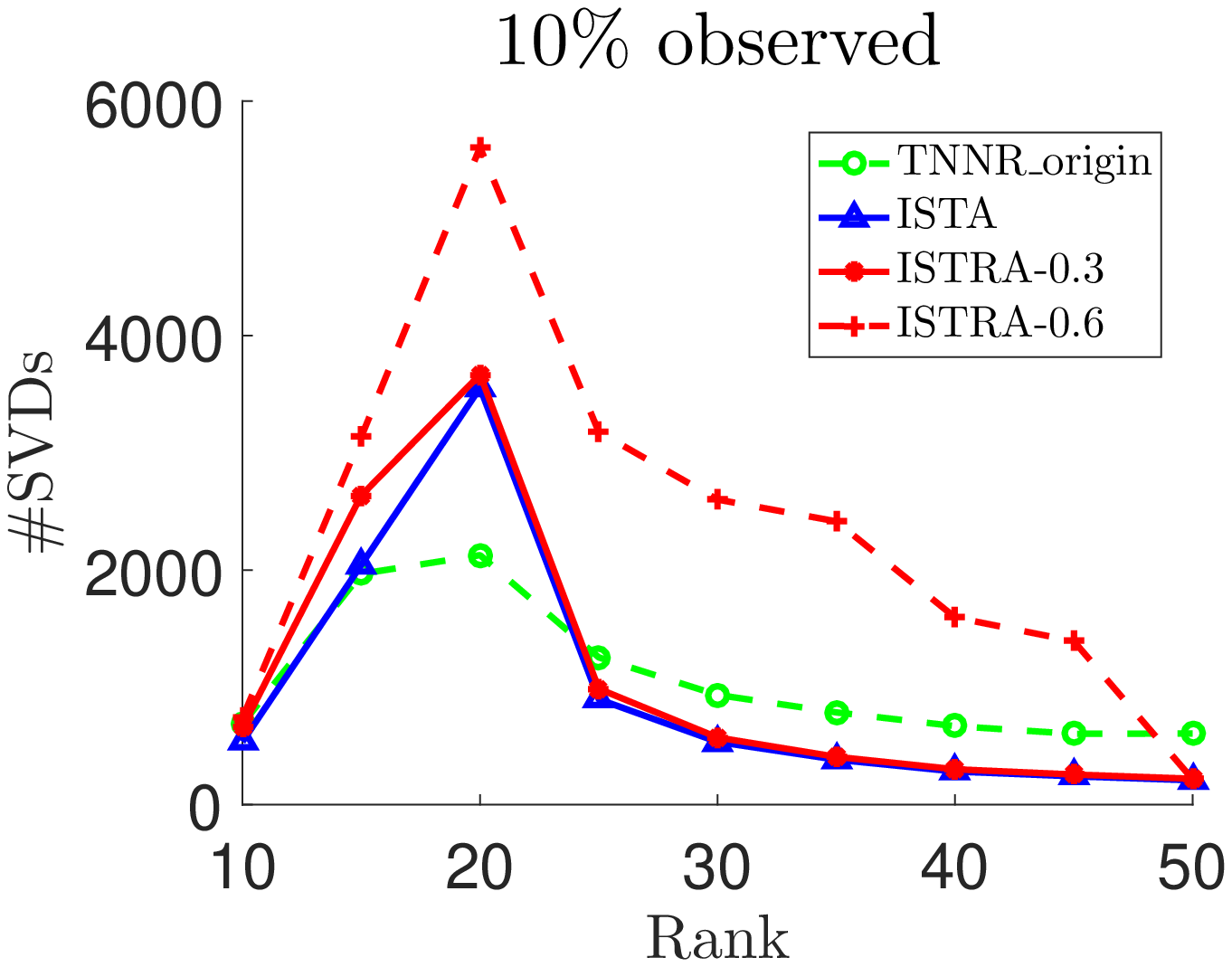}\hspace*{\fill}
\includegraphics[scale=0.28]{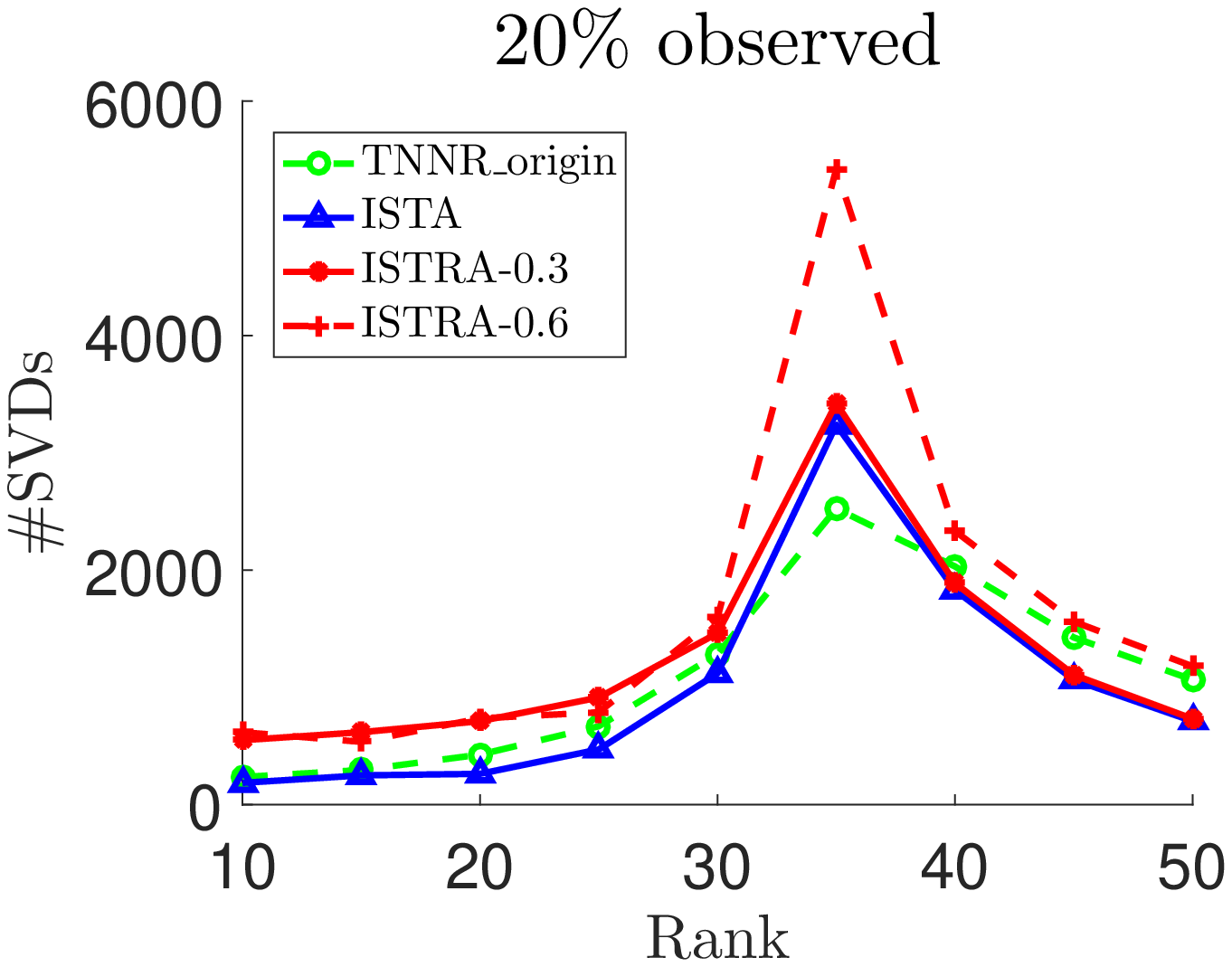}\hspace*{\fill}
\includegraphics[scale=0.28]{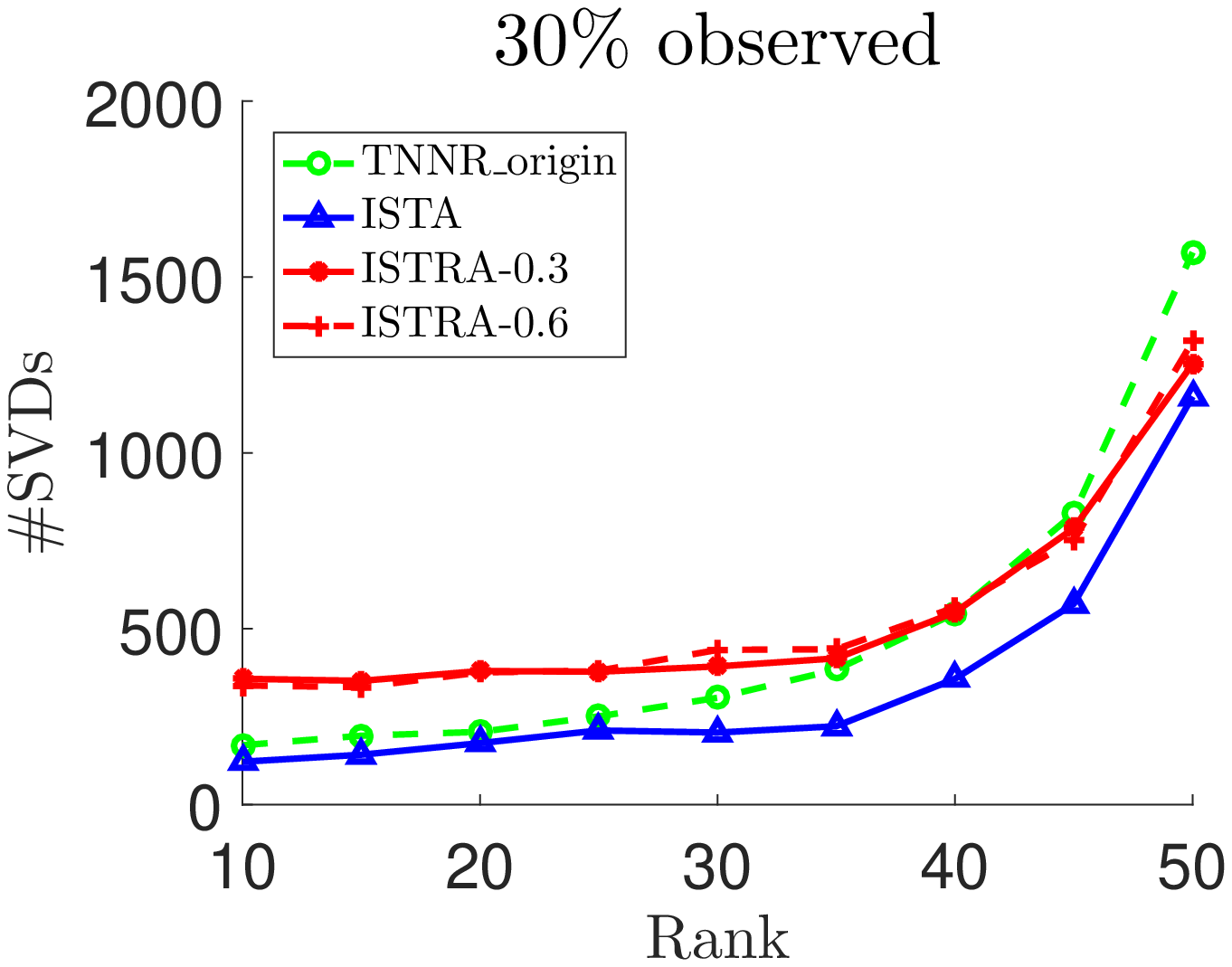}\hspace*{\fill}
\includegraphics[scale=0.28]{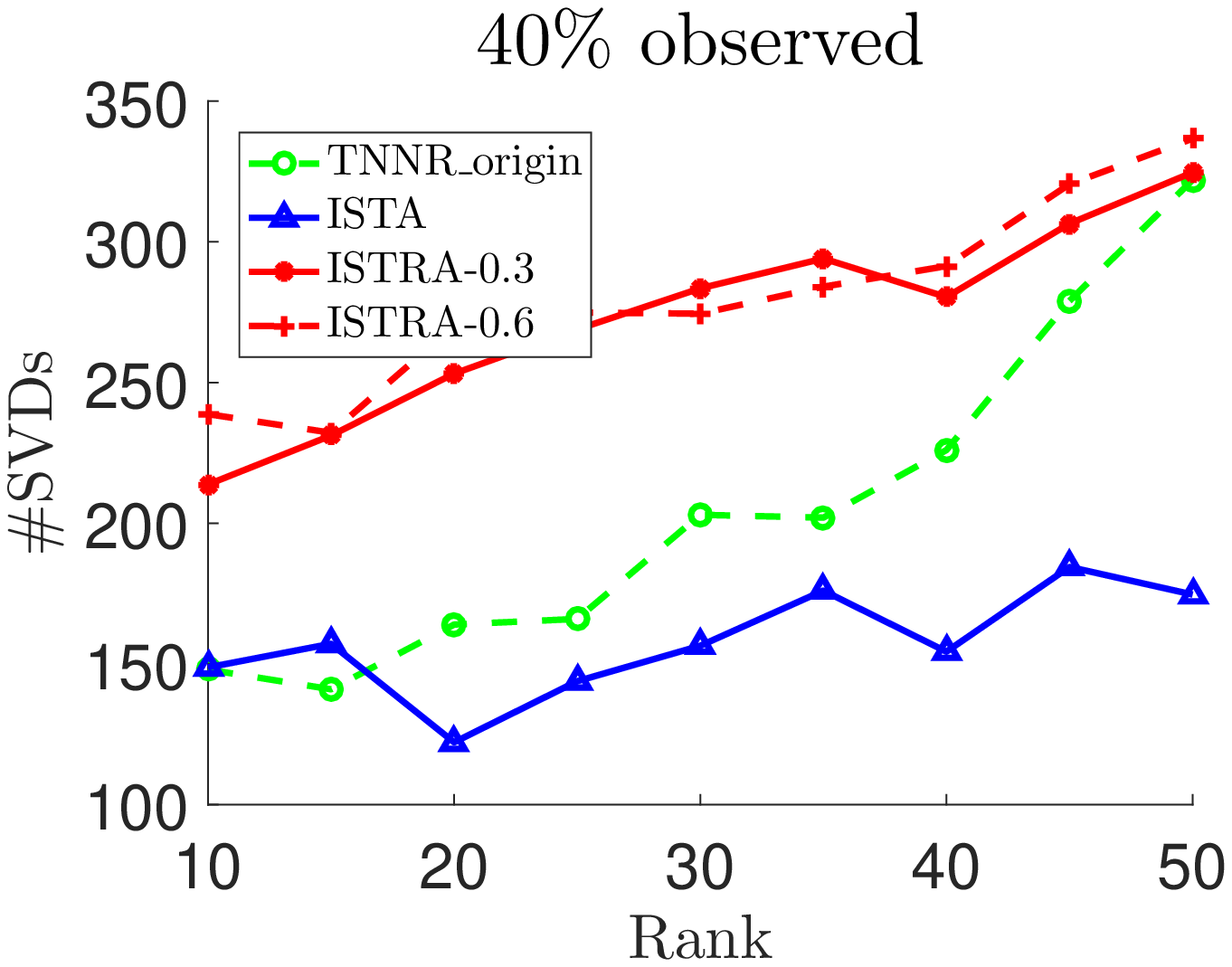}
\caption{{Number of SVD computations versus rank
with different observations}}
\label{Fig.all_rank_svd}
\end{figure*}

First, we fix the matrix size and noise level to be $400\times 300$, $a=0.5$ respectively, and change the rank
with different observed ratios.
The results are shown in Figure \ref{Fig.all_rank}.
Next, we fix the matrix size and rank to be $400\times 300$, $b=30$ respectively, and change the noise level
with different observed ratios.
 We found that that all algorithms were failed in recovering matrices when observed ratio is $10\%$.
As a result, the starting ratio is raised up to $20\%$. The results has been shown in Figure \ref{Fig.all_noise}.
To verify the computational effectiveness, the number of SVDs for TNNR\_origin, ISTA and ISTRA are shown in Figure~\ref{Fig.all_rank_svd} and Figure~\ref{Fig.all_noise_svd} with same settings as in Figure~\ref{Fig.all_rank} and Figure~\ref{Fig.all_noise} respectively. To make a fair comparison, the number of SVD computations for preprocessing the starting point for ISTRA has been added.
\begin{figure}[htbp]
 \centering
\includegraphics[scale=0.28]{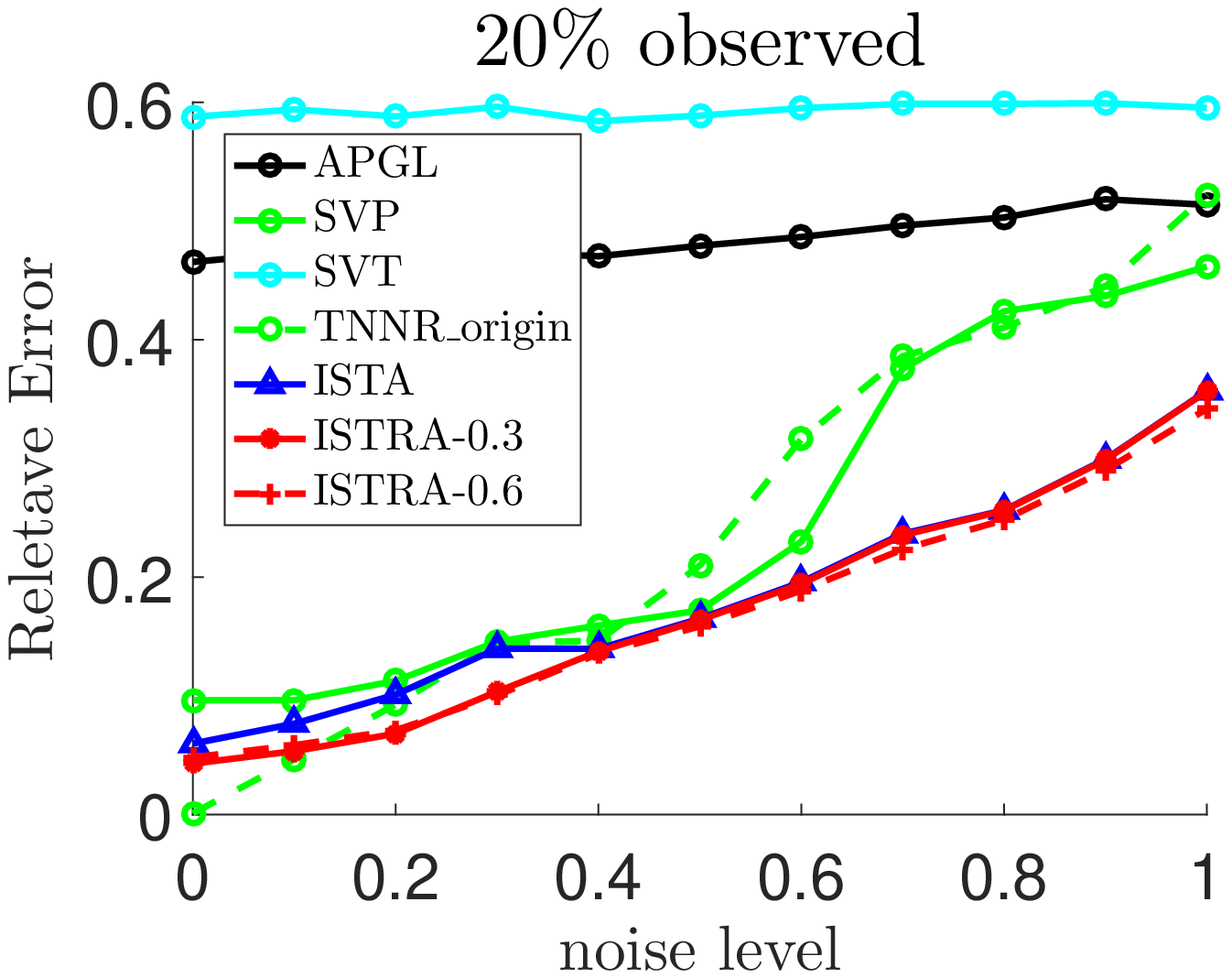}\hspace*{\fill}
\includegraphics[scale=0.28]{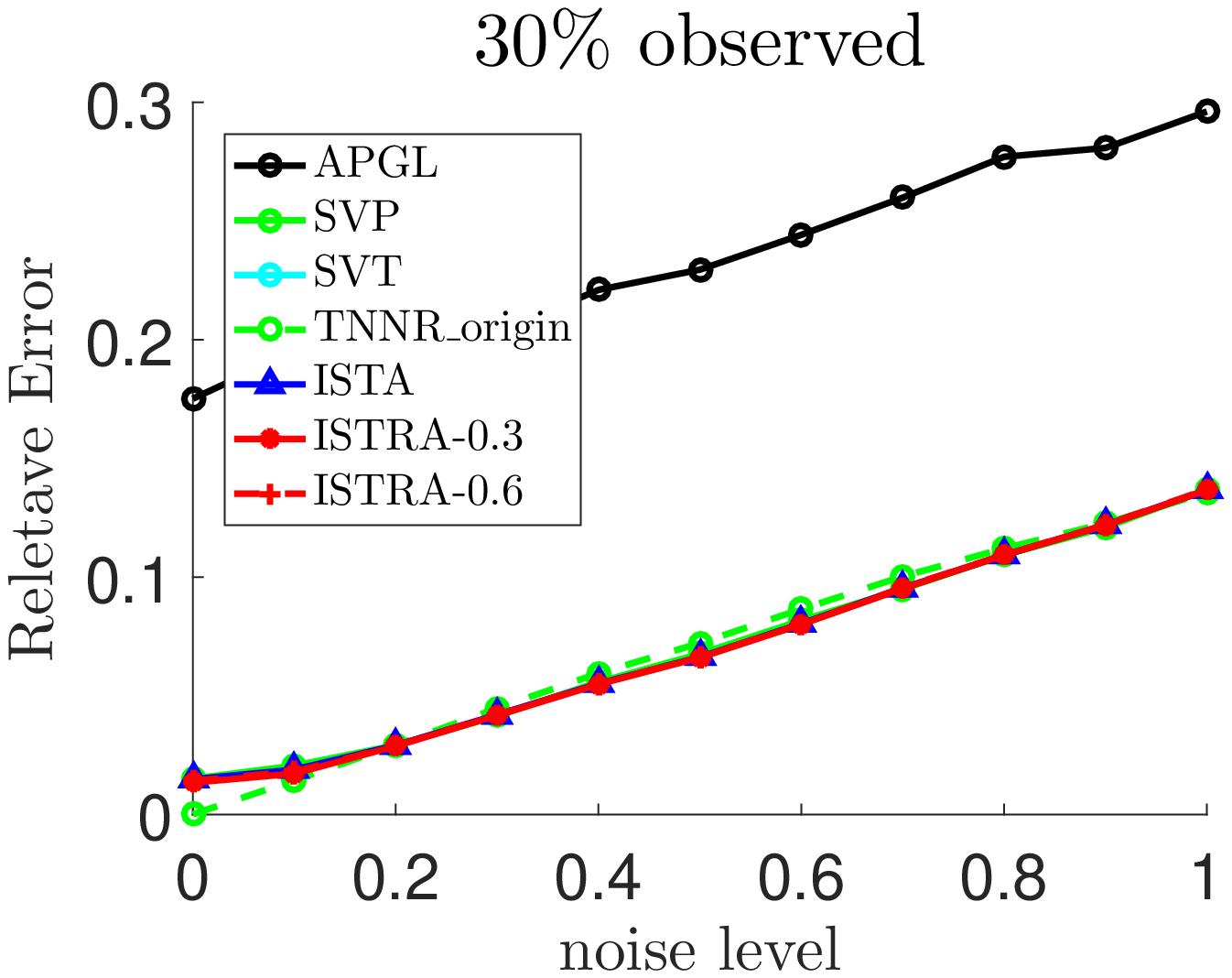}\hspace*{\fill}
\includegraphics[scale=0.28]{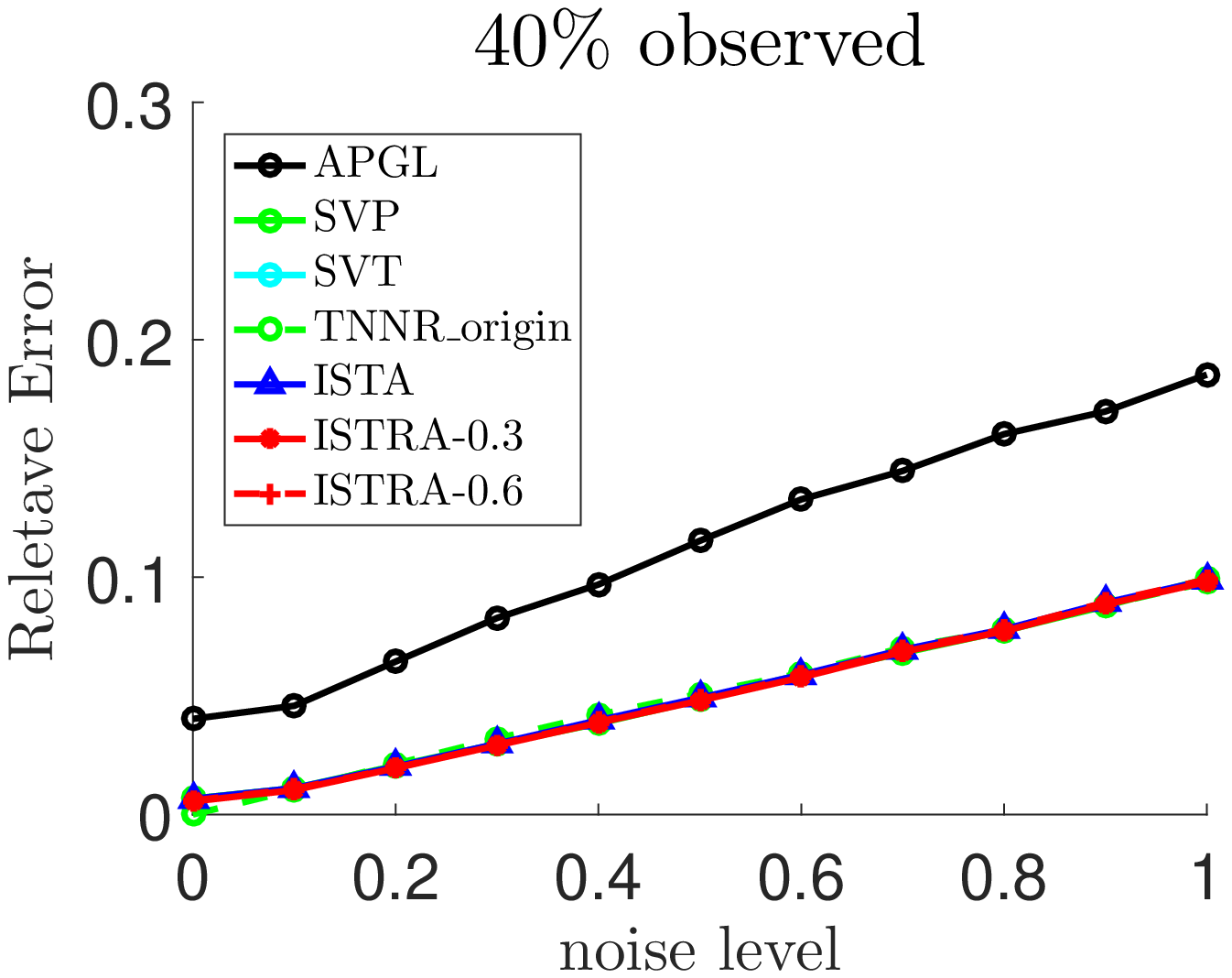}\hspace*{\fill}
\includegraphics[scale=0.28]{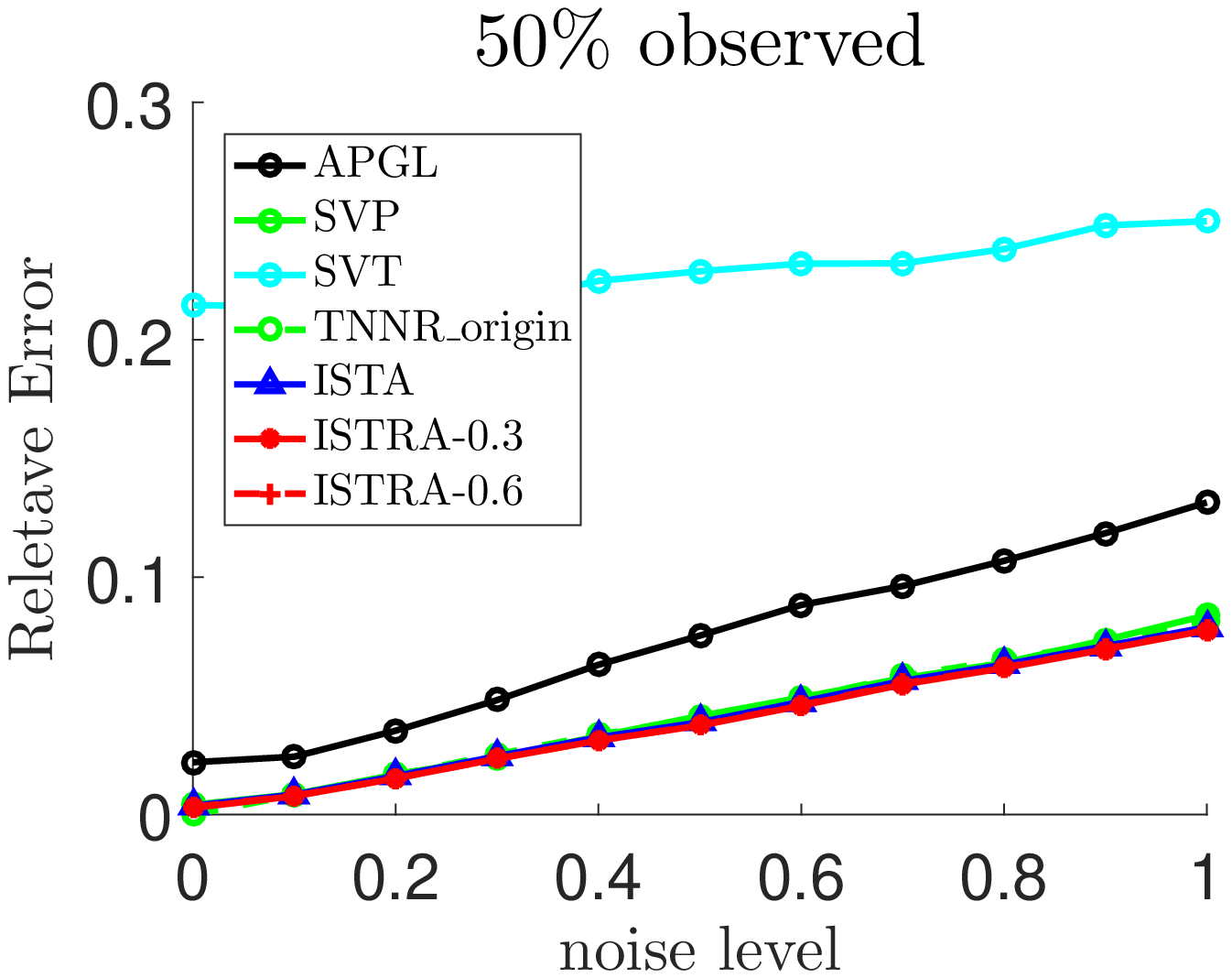}
\caption{{Relative error versus noise
with different observations}}
\label{Fig.all_noise}
\end{figure}
\begin{figure*}[htbp]
 \centering
\includegraphics[scale=0.28]{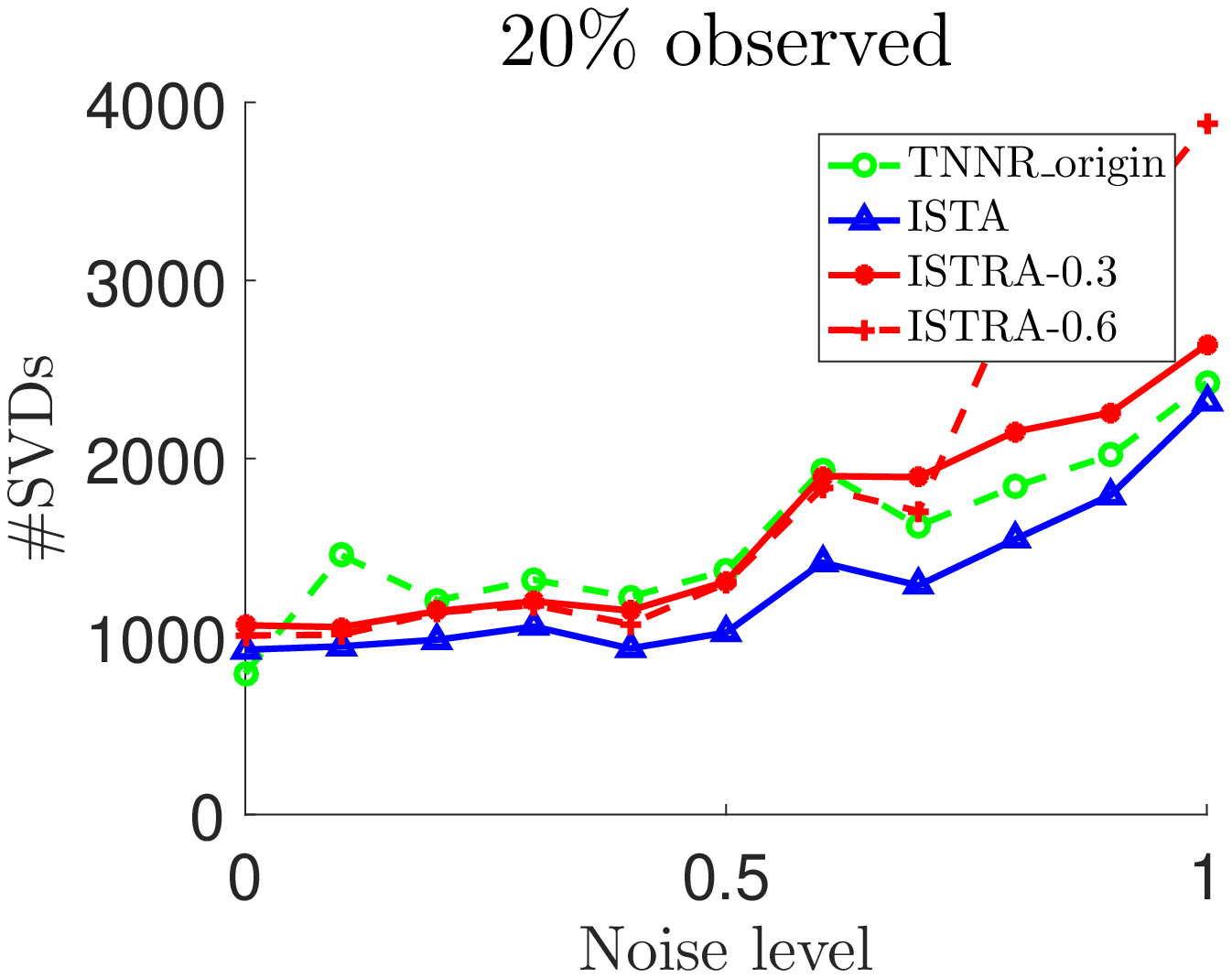}\hspace*{\fill}
\includegraphics[scale=0.28]{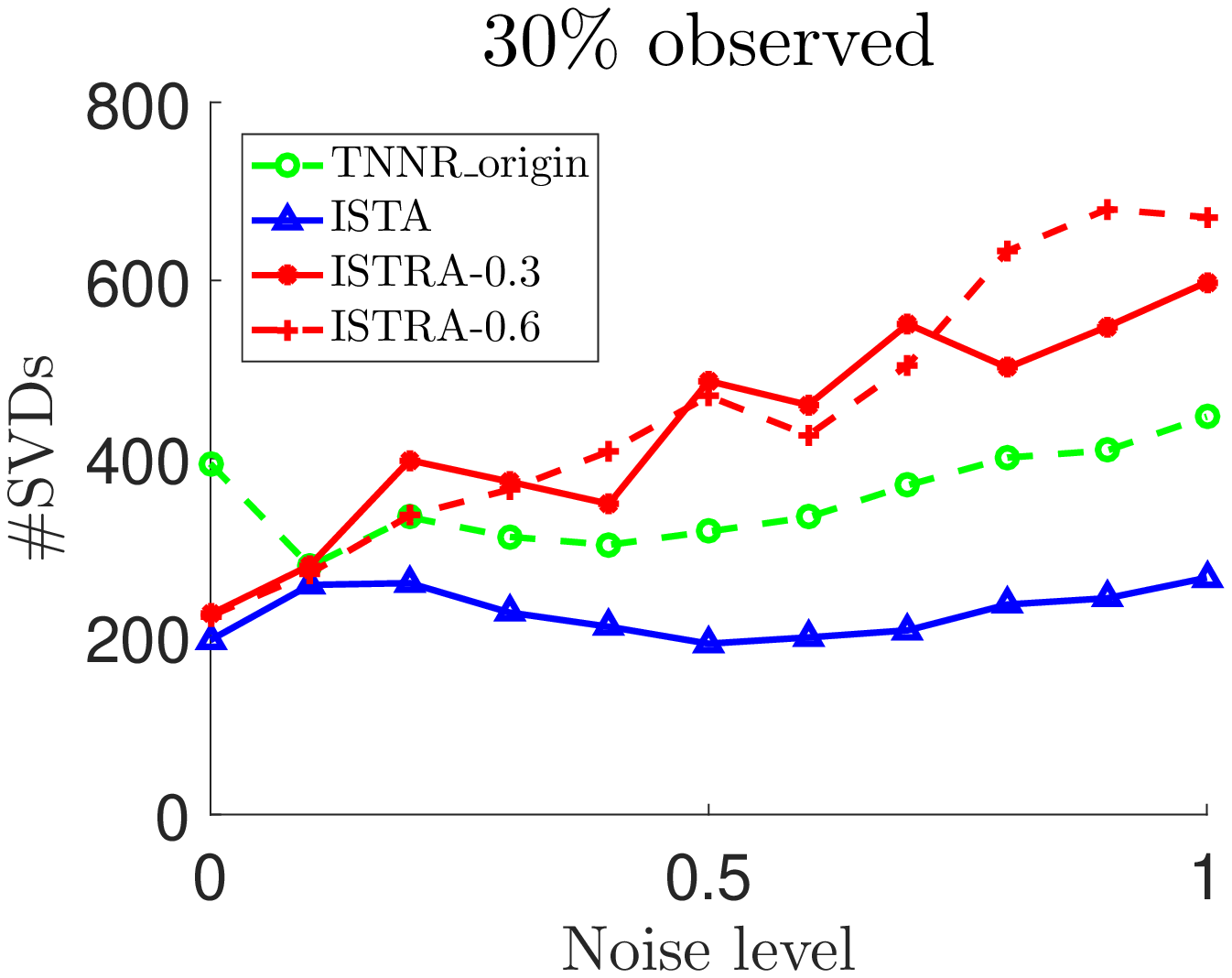}\hspace*{\fill}
\includegraphics[scale=0.28]{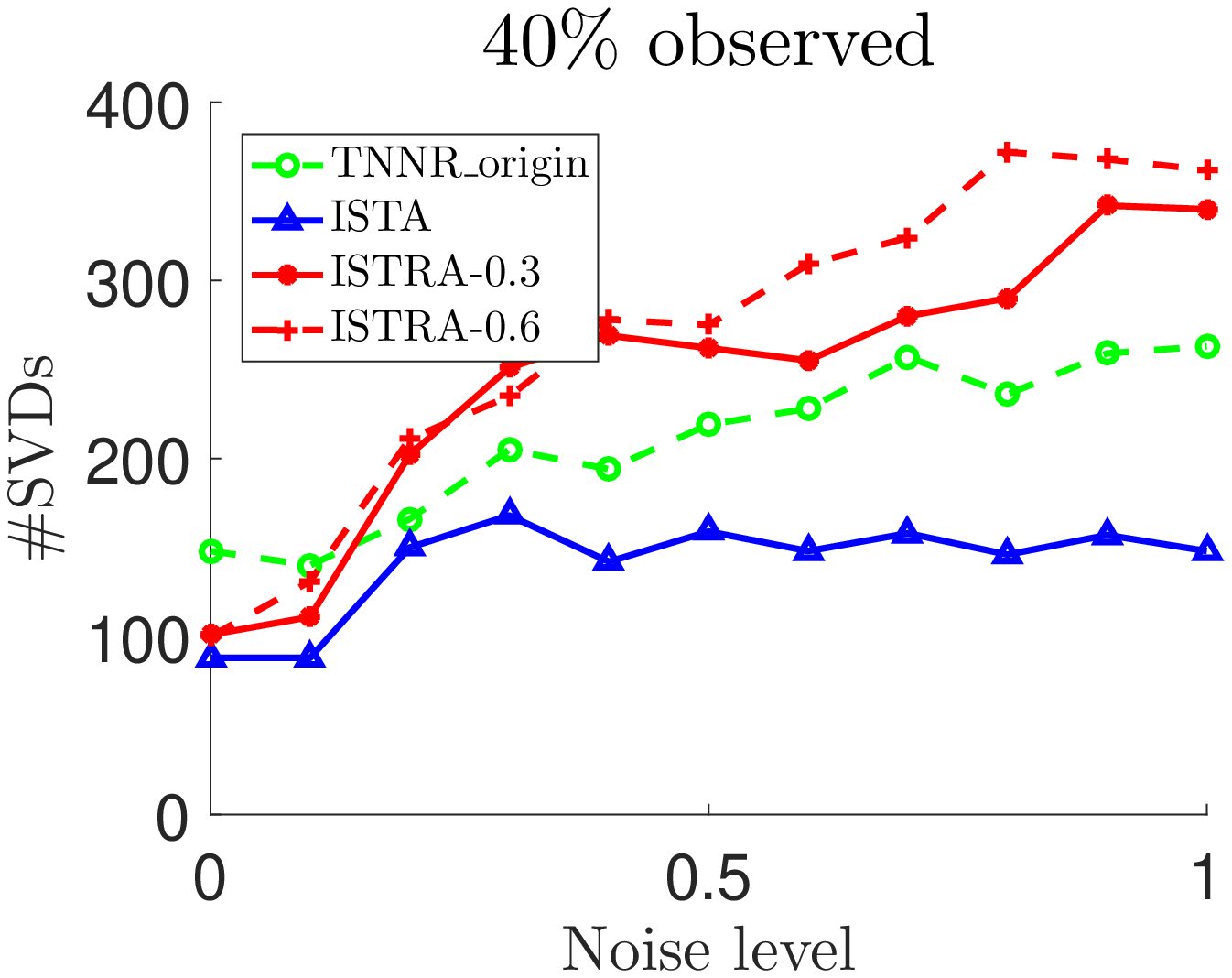}\hspace*{\fill}
\includegraphics[scale=0.28]{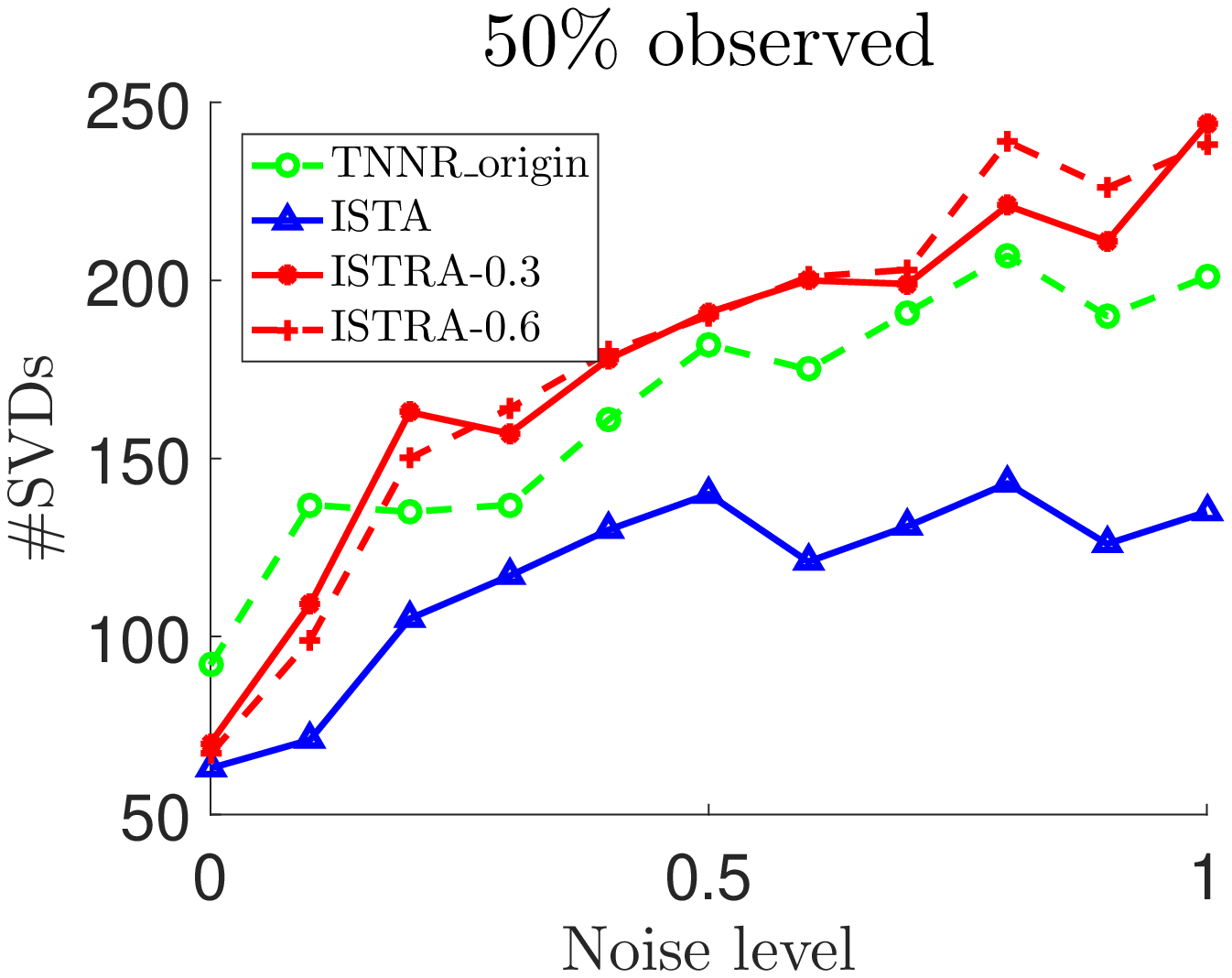}
\caption{{Number of SVD computations versus noise
with different observations}}
\label{Fig.all_noise_svd}
\end{figure*}

As can be observed from Figure \ref{Fig.all_rank}-\ref{Fig.all_noise_svd}, the proposed ISTA and ISTRA are more robust to noise and more reliable as the underlying rank and noise increases.
Particularly, our algorithms have notable advantages when problem becomes harder (less entries or entries with larger noise are observed), and therefore is able to survive more corrupted data,
which will significantly enhance the low rank recovery in real applications.
Compared with TNNR\_origin, we can see that ISTA needs fewer SVD computations to converge and gives comparable or more accurate solutions in the most settings, which shows the correctness of our theory and make the proposed algorithms more appealing in the real-world applications.
By comparing ISTA and ISTRA, we can find that ISTRA can continue making progress when ISTA has stopped.
Although small $p$ produces better approximation of rank function, it makes algorithm more likely be stuck in poor solutions, which also result in fewer SVDs to make progress as we can see from Figure~\ref{Fig.all_rank_svd}, \ref{Fig.all_noise_svd}.
In Figure~\ref{Fig.all_rank_svd}, the peaks arise for all algorithms when underlying truths first become not achievable, in which case ISTA and ISTRA are still attempting to complete the matrix with more iterations and gives better solutions.
\begin{figure}[htbp]
 \centering
\includegraphics[scale=0.28]{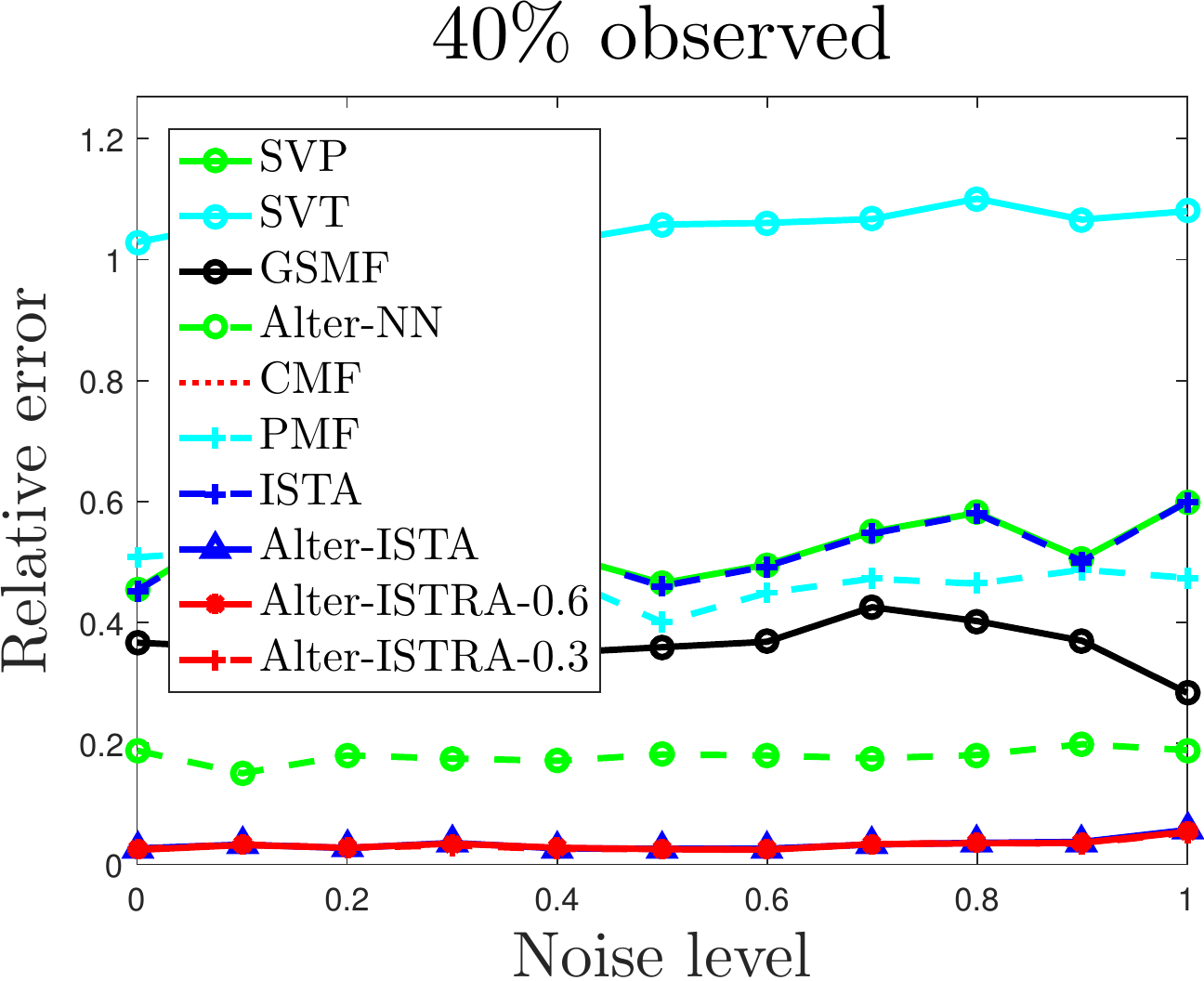}\hspace*{\fill}
\includegraphics[scale=0.28]{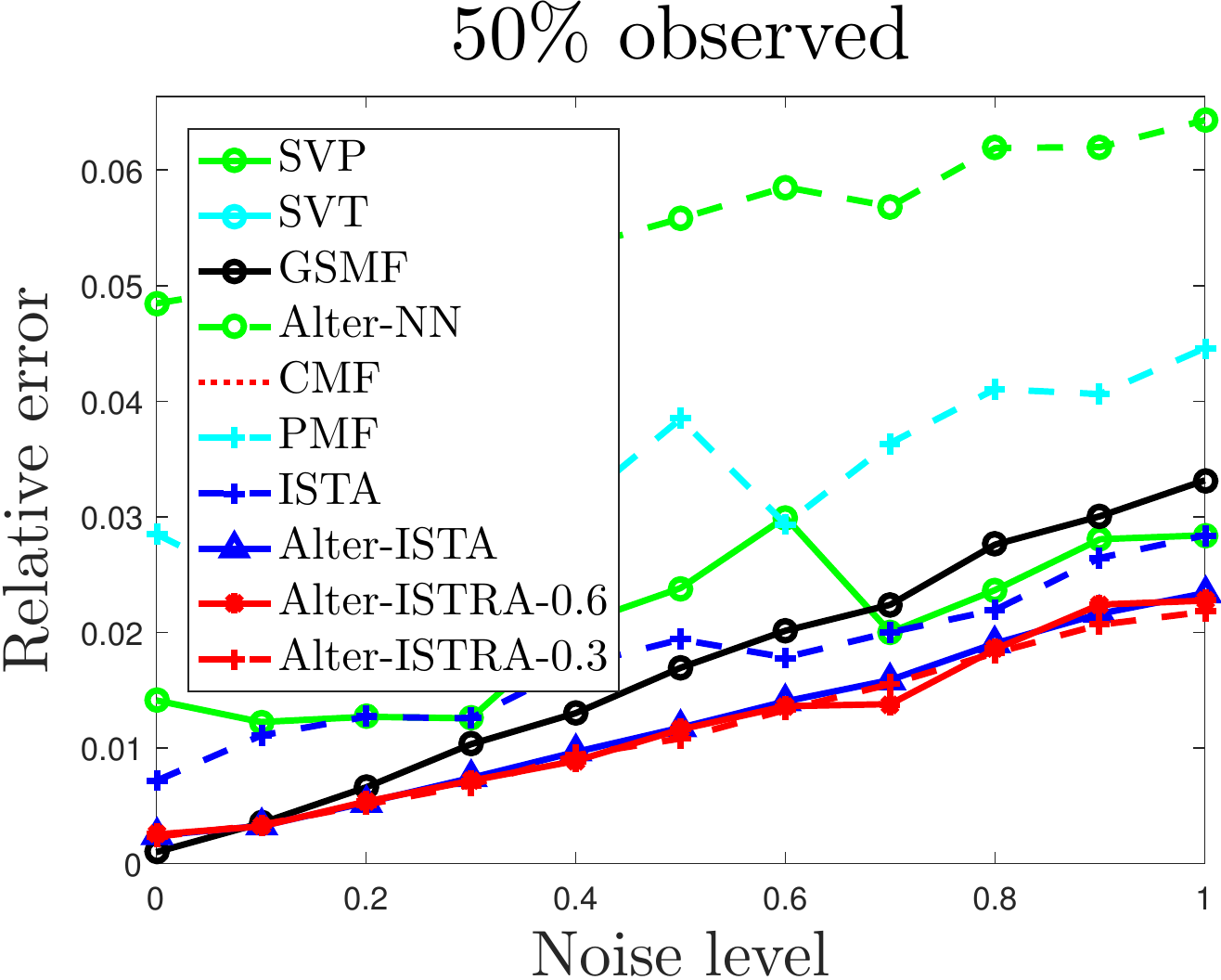}\hspace*{\fill}
\includegraphics[scale=0.28]{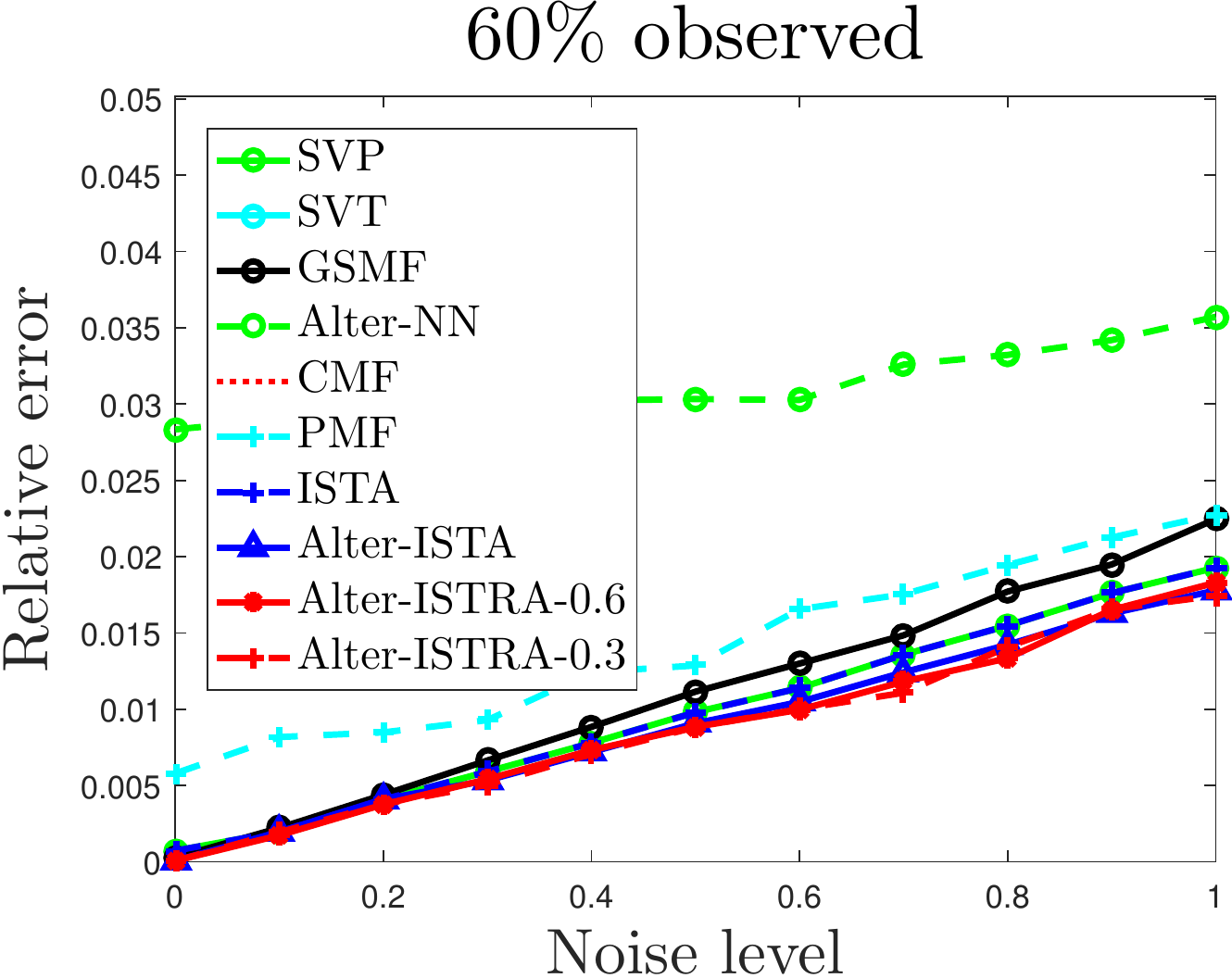}\hspace*{\fill}
\includegraphics[scale=0.28]{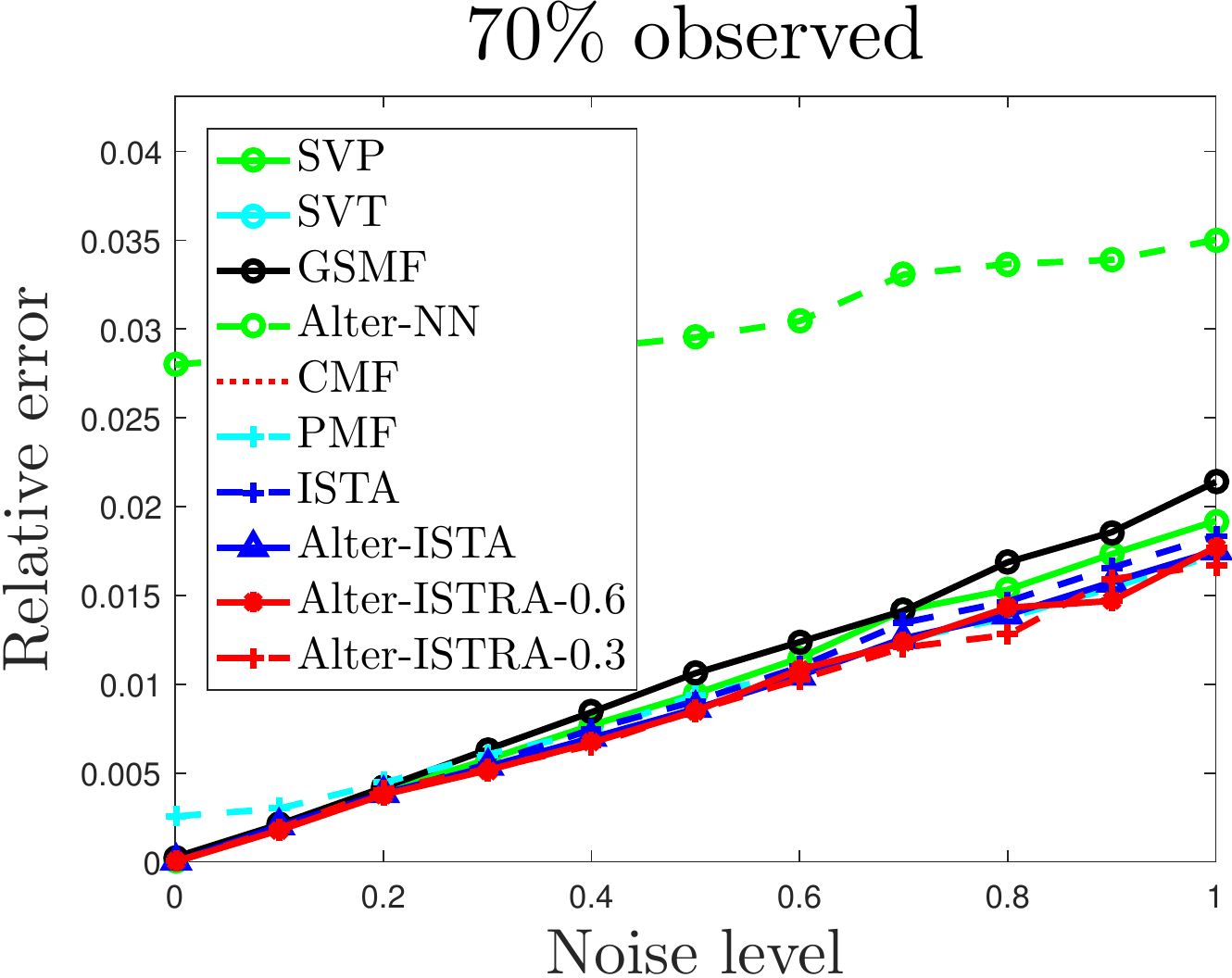}
\caption{{Relative error versus noise with different
observation ratios}}
\label{Fig.alter}
\end{figure}



To show the effectiveness of Algorithm~\ref{alg_pg1} and Algorithm~\ref{alg_spg1}, and the correctness of the assumption that shared and domain specific components could help complete the noised matrices, more experiments are conducted on multiple synthetic matrices.
The recommendation based baselines are included for completeness, among which PMF \citep{mnih2008probabilistic} is a classical single-viewed collaborative filtering method, CMF \citep{singh2008relational} and GSMF \citep{yuan2014recommendation} also consider exploiting cross-domain information.
Since ISTA performs better than TNNR\_origin and comparable with ISTRA as shown before, only ISTA is included in the following comparison.
To investigate the behavior of the proposed non-convex penalties, we also evaluate the performance of problem (\ref{eqn:problem3}) with the standard nuclear norm penalties, which is denoted by Alter-NN.

Due to the fact that the problems are more difficult and time-consuming compared to single matrix scenario, the settings are changed in the following experiments. The synthetic data is constructed on two domains for experimental
investigation. We randomly generate two $100\times100$ matrices with
shared and distinct components as follows:
\begin{equation}\label{eq_generatesyn}
Z^d=M^d+D^d,\;\; Y_{\Omega}^d = Z_{\Omega}^d+\varepsilon, \;d=1,2.
\end{equation}
Here $\{Z^d\}$ are the ground truth for all the domains, and
$\{Y_{\Omega}^d\}$ are the noisy observed matrices. The shared
components are generated by $M^d=AB^d$ where $A$ is shared across
all the domains, $A\in\mathbb{R}^{100\times10}$ and $B^d\in
\mathbb{R}^{10\times100}$ consist of i.i.d. Gaussian entries with
variance 25. The distinct parts are generated by
$D^d=P^dQ^d$ 
where $P^d\in\mathbb{R}^{100\times10}$ and
$Q^d\in\mathbb{R}^{10\times100}$ also consist of i.i.d. Gaussian
entries but with variance 100. The observation indexes
$\{\Omega_d\}$ are sampled uniformly at random. The variance of the
shared components is set 
smaller than that of the distinct components to simulate real
situations. For all methods, parameters are tuned as mentioned before.
Average results of 10 rounds are shown in Figure~\ref{Fig.alter}.

We can
first observe that CMF and SVT fail to recover the matrices in all
settings. The performance of CMF is likely due to the fact that the
distinct components are more significant than the shared part,
contradicting with the assumption of CMF; while the number of
observed entries does not satisfy the recovery condition of SVT,
which explains its degeneration of performance.
We can also see that Alter-NN cannot achieve very low RE level even in the settings with high observed ratio or low noise level, but it is quit stable compared with other baselines, which is similar to the performance of APGL in Figure.~\ref{Fig.all_rank}-\ref{Fig.all_noise}.
Meanwhile, the improvement of Alter-ISTA
over Alter-NN justifies the advantage of no-convex regularization over the standard nuclear norm.
Just like the single matrix scenario, Alter-ISTRA performs better than Alter-ISTA, but requires much more running time.
All the other algorithms perform reasonably when the observation ratio is above
60\%.
When the ratio decreases to 50\%, the $\mathrm{RE}$ values of all the
baselines grow faster with increasing noise than Alter-ISTA and Alter-ISTRA. When
the observed ratio drops to 40\%, all the comparing methods fail to
recover the matrices correctly even if the observations are
noiseless; whereas Alter-ISTA and Alter-ISTRA are capable of exploit the correlations
among multiple domains to significantly alleviate the low-rank problem, which justify the necessity of assumption.

\subsection{Real Image Data}

\begin{figure*}[htbp]
\centering
\includegraphics[scale=0.3]{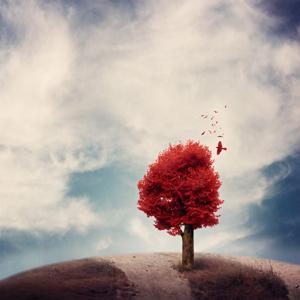}\hspace{15pt}
\includegraphics[scale=0.3]{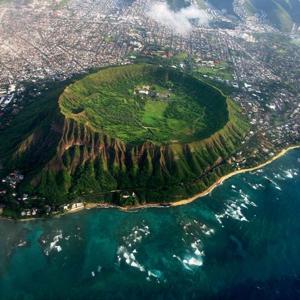}\hspace{15pt}
\includegraphics[scale=0.3]{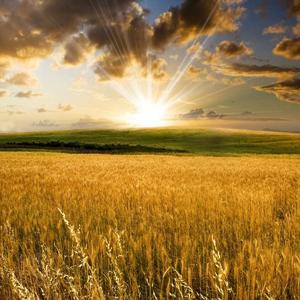}\hspace{15pt}
\includegraphics[scale=0.3]{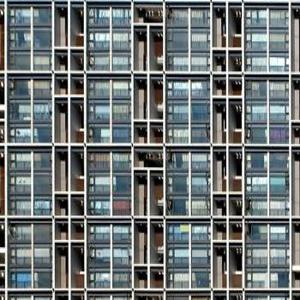}\\

\vspace{8pt}
\hspace{0pt}
\includegraphics[scale=0.3]{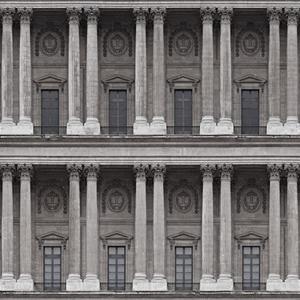}\hspace{15pt}
\includegraphics[scale=0.3]{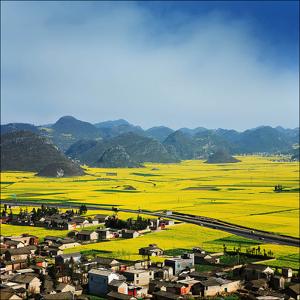}\hspace{15pt}
\includegraphics[scale=0.6]{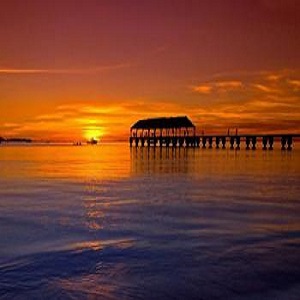}\hspace{15pt}
\includegraphics[scale=0.3]{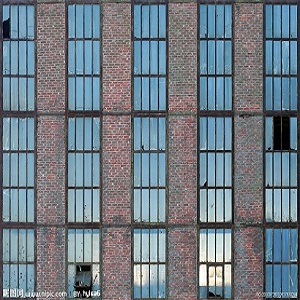}
\caption{{Images used in experiments (number 1-8)}}
\label{test image}
\end{figure*}

\begin{figure*}[htbp]
\centering
\includegraphics[scale=0.28]{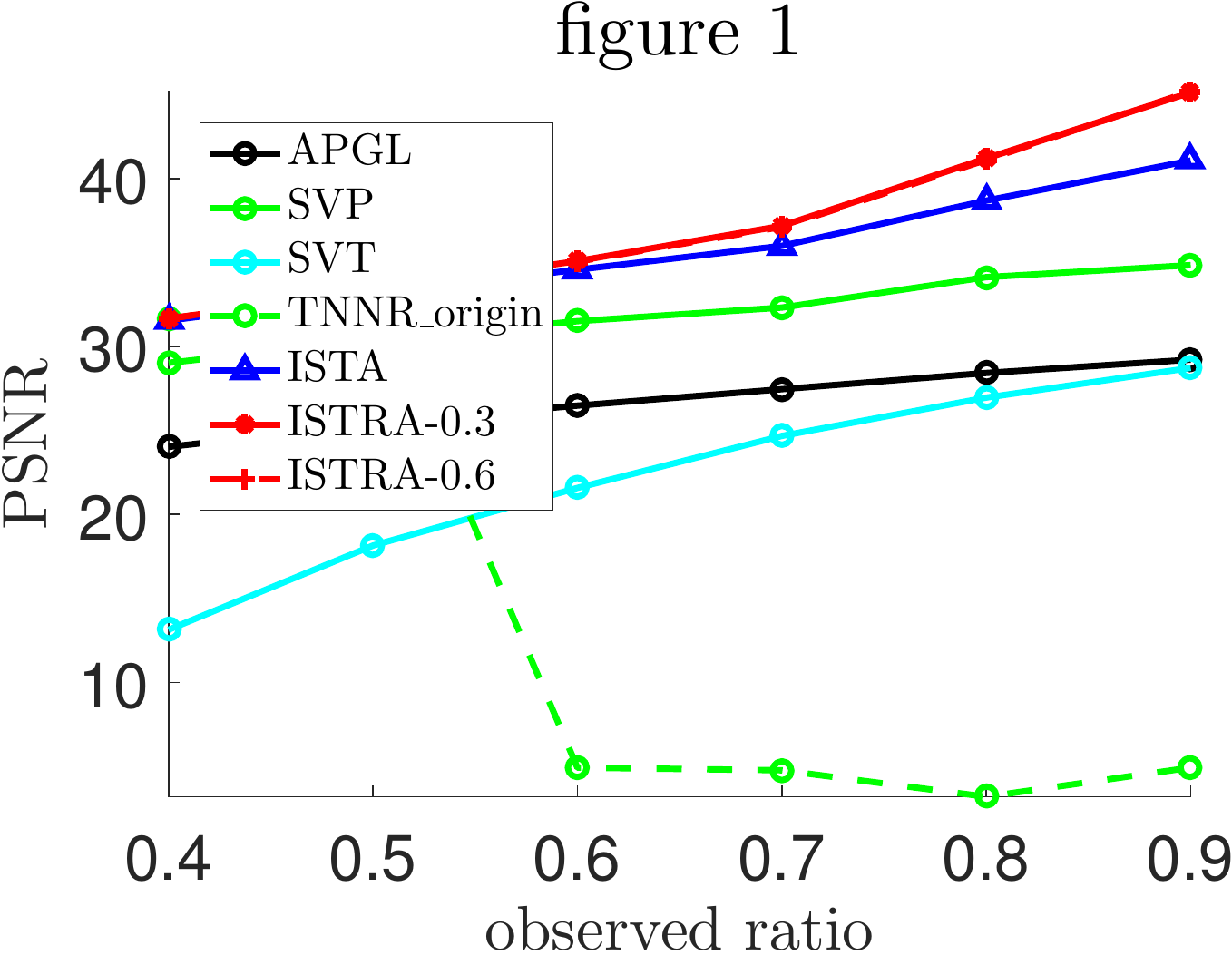}\hspace*{\fill}
\includegraphics[scale=0.28]{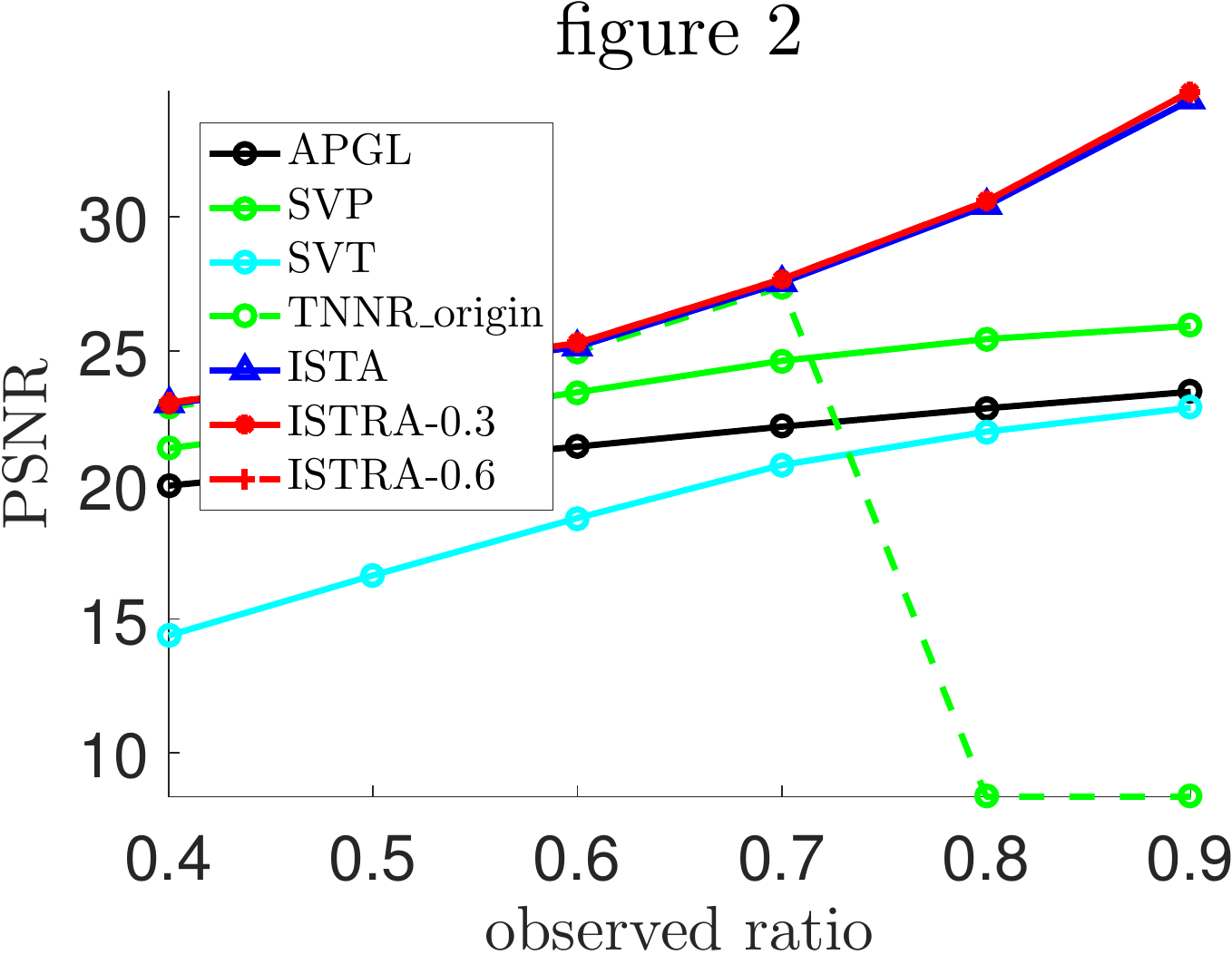}\hspace*{\fill}
\includegraphics[scale=0.28]{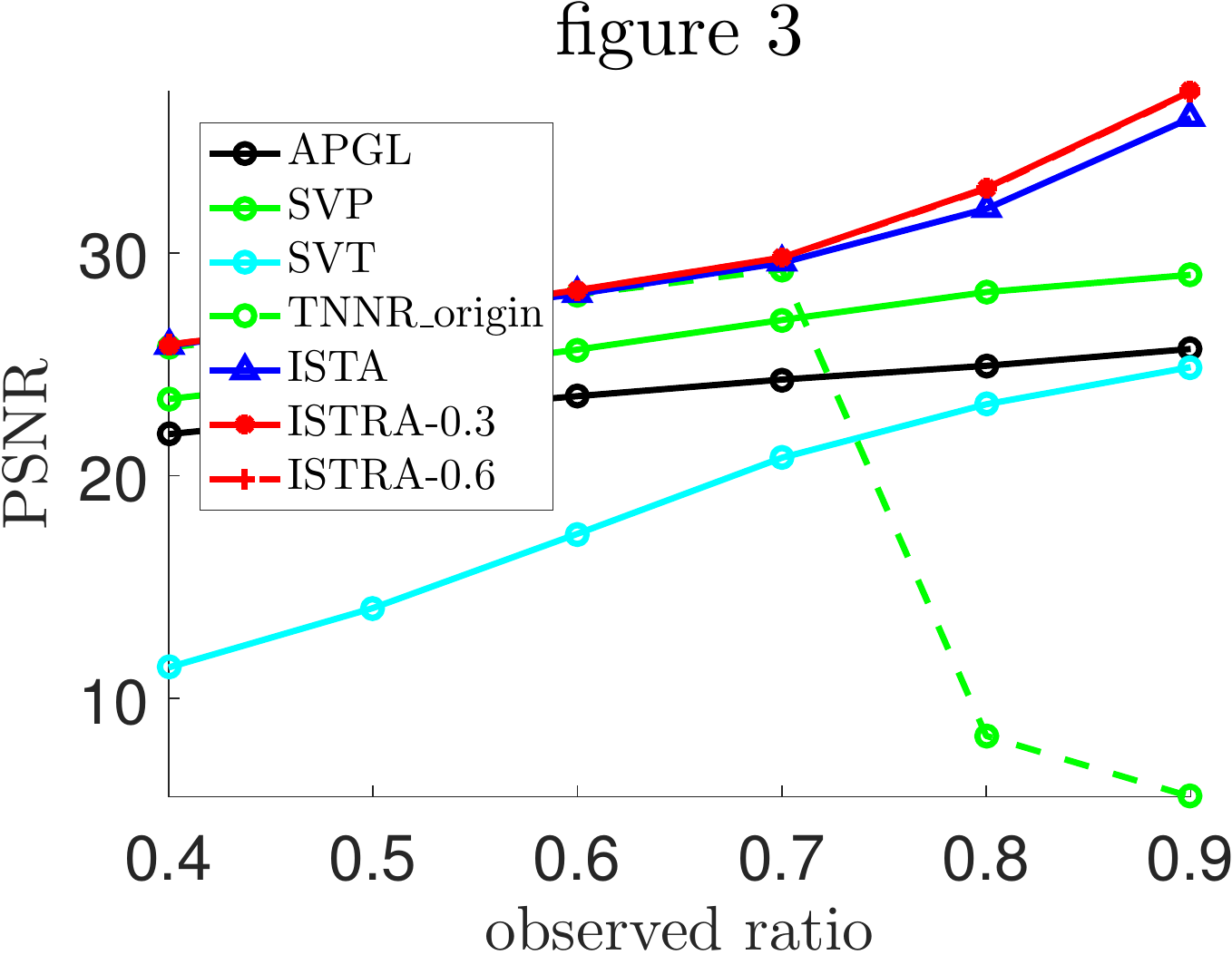}\hspace*{\fill}
\includegraphics[scale=0.28]{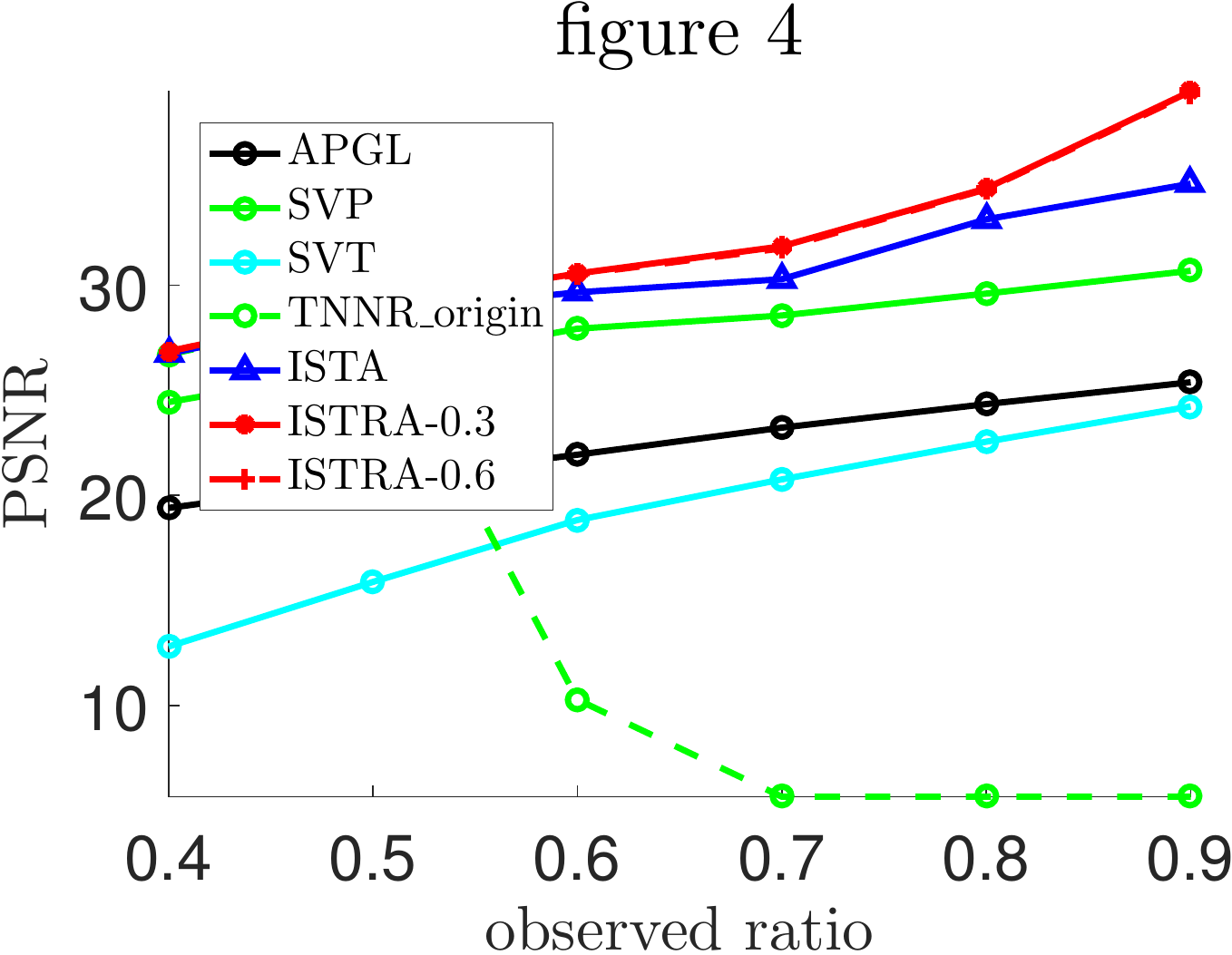}

\medskip
\includegraphics[scale=0.28]{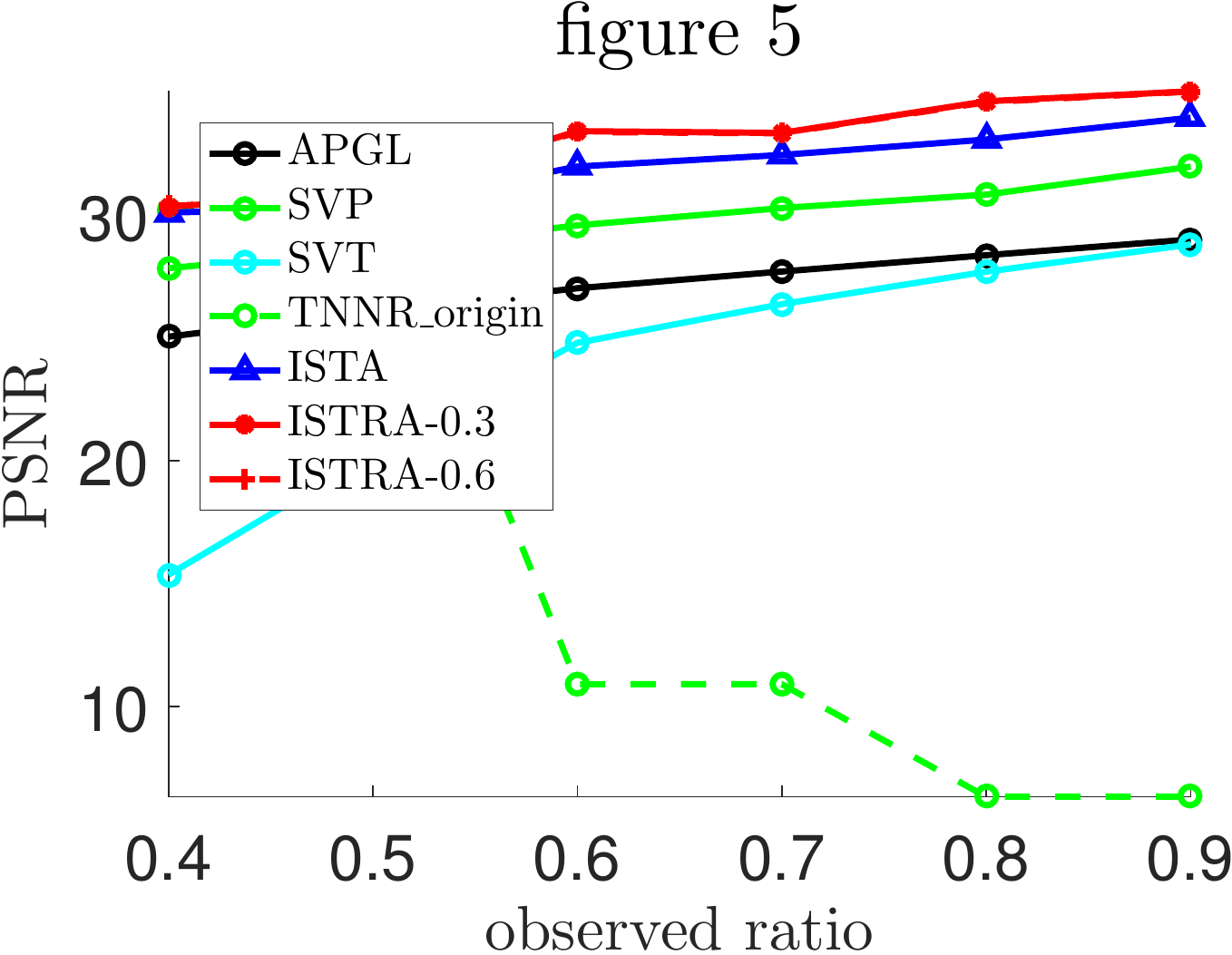}\hspace*{\fill}
\includegraphics[scale=0.28]{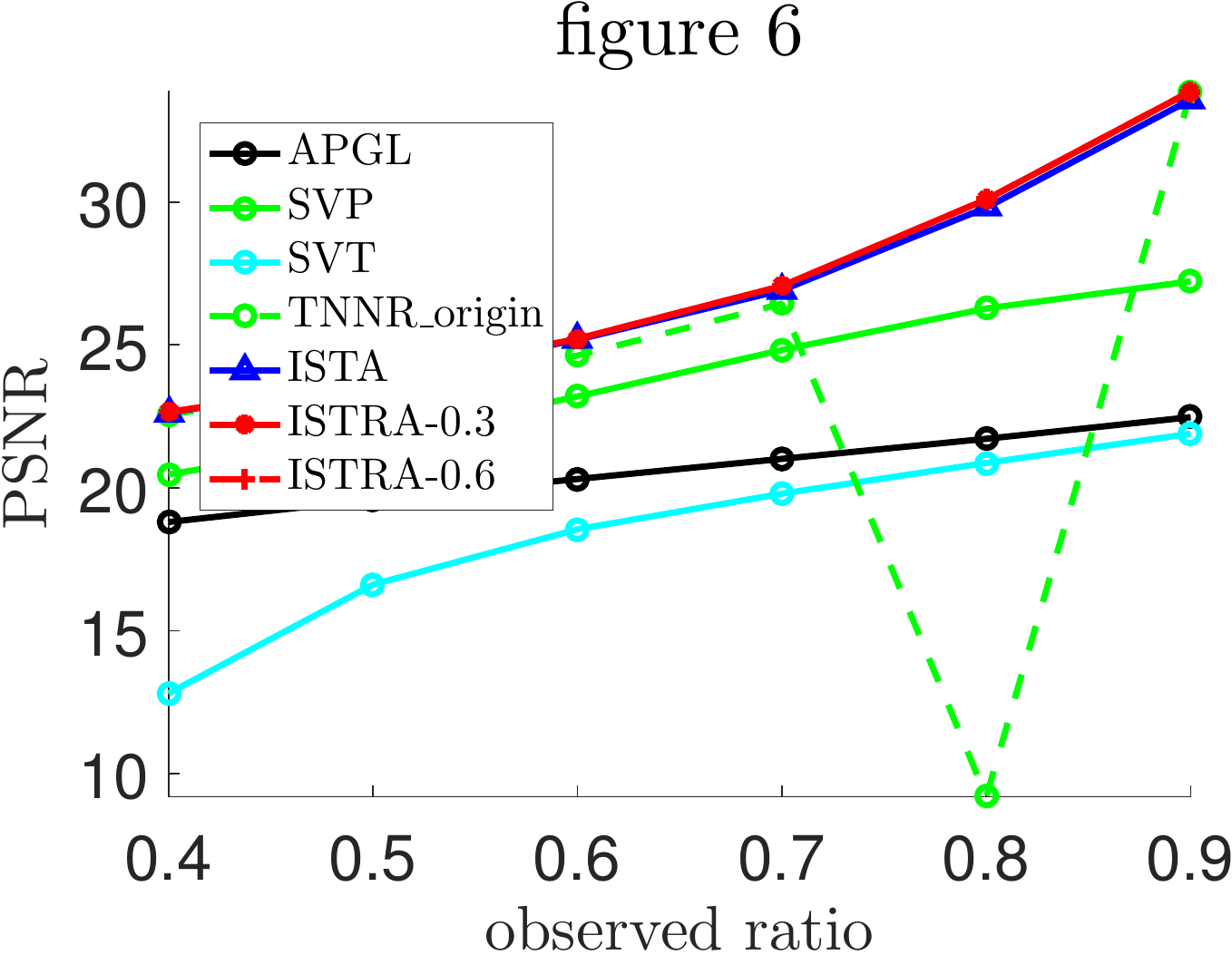}\hspace*{\fill}
\includegraphics[scale=0.28]{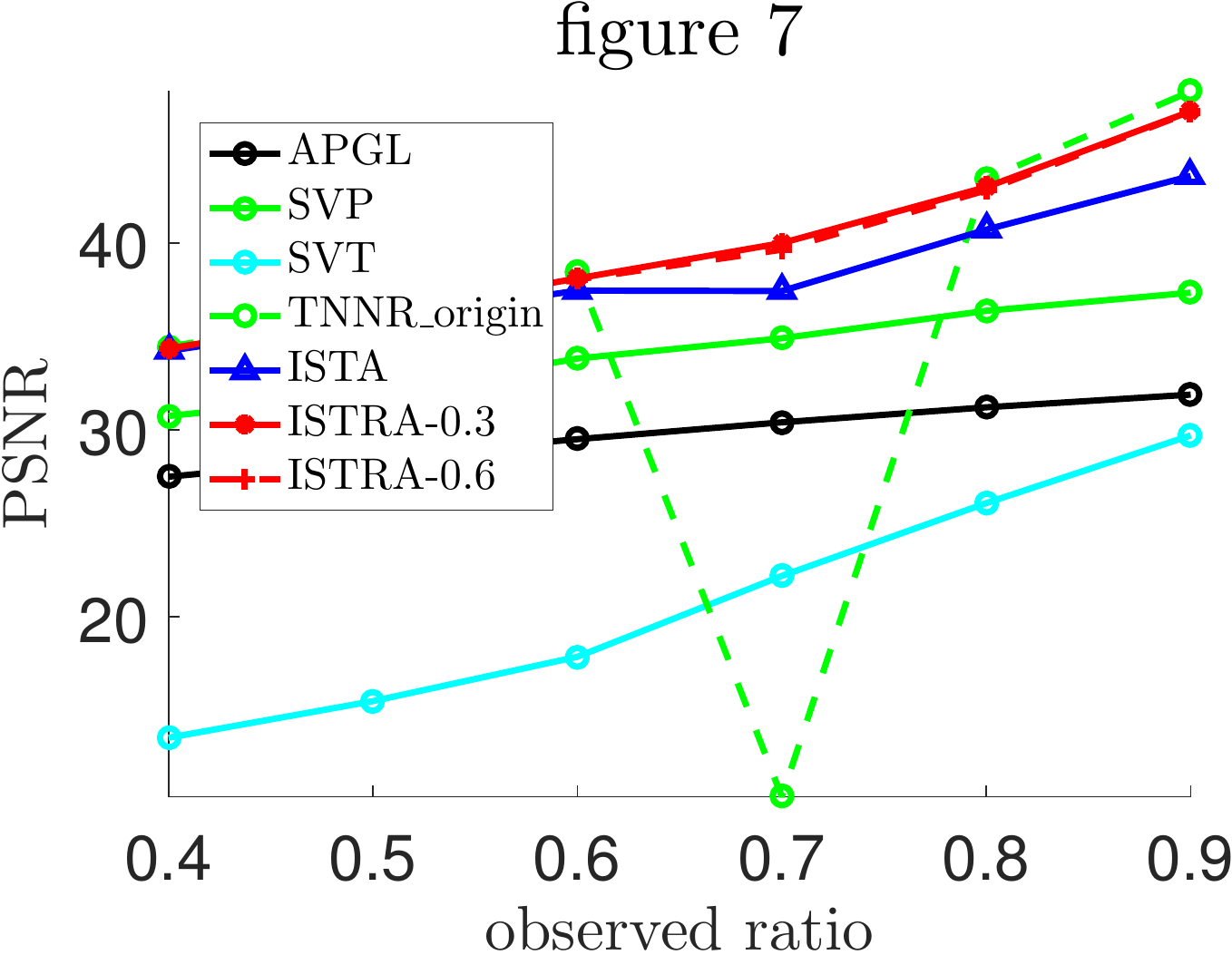}\hspace*{\fill}
\includegraphics[scale=0.28]{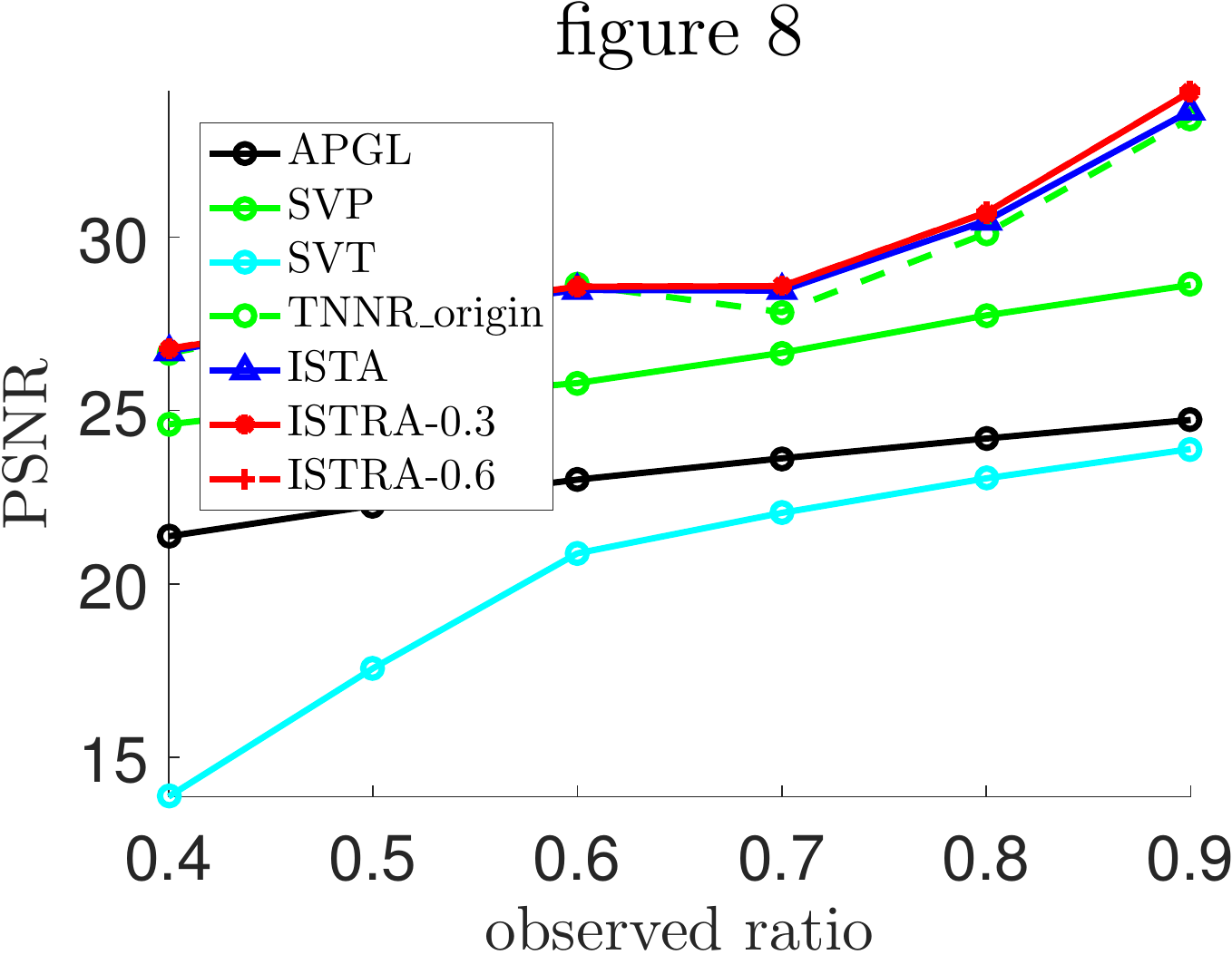}
\caption{{PSNR values of recovered images with different observed ratios for random mask (top row: images 1-4, bottom row: images 5-8)}}
\label{Fig.rand}
\end{figure*}

\begin{figure*}[htbp]
 \centering

\subfigure[{Random mask}]{
\includegraphics[scale=0.33]{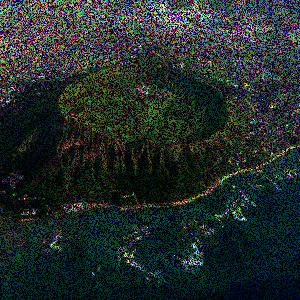}}
\subfigure[{SVP 22.34}]{
\includegraphics[scale=0.33]{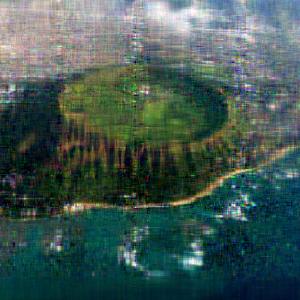}}
\subfigure[{SVT 18.15}]{
\includegraphics[scale=0.33]{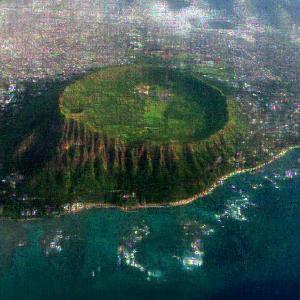}}
\subfigure[{APGL 20.71}]{
\includegraphics[scale=0.33]{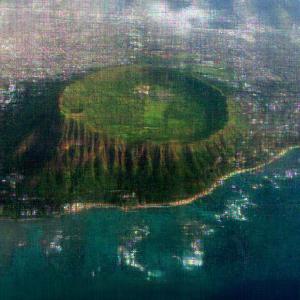}}\\

\medskip
\hspace{3pt}\subfigure[{TNNR\_origin 24.02}]{
\includegraphics[scale=0.33]{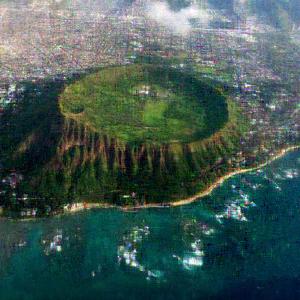}}
\subfigure[{ISTA 24.14}]{
\includegraphics[scale=0.33]{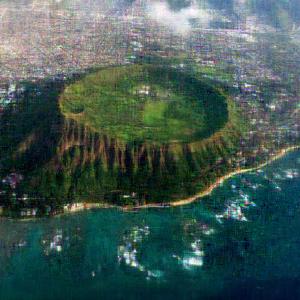}}
\subfigure[{ISTRA-0.3 24.24}]{
\includegraphics[scale=0.33]{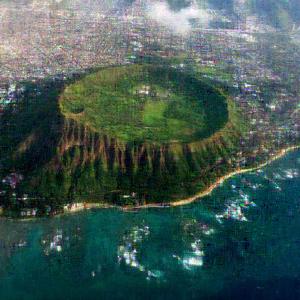}}
\subfigure[{ISTRA-0.6 24.23}]{
\includegraphics[scale=0.33]{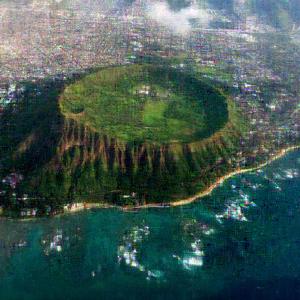}}
\caption{\footnotesize{Recovered images and PSNR values by different methods (50\% pixels are randomly masked)}}
\label{show image_random}
\end{figure*}

Here we consider the task of image inpainting which can also be treated as a matrix completion problem.
Regarding a noisy image as three separate incomplete matrices (3 channels),
we aim to recover missing pixels by exploiting the low rank structure.
The quality of recovered image is evaluated by the well known PSNR (Peak Signal-to-Noise Ratio) measure, which is defined as $10\log_{10}\big(\frac{255^2}{MSE}\big)$ and $MSE$ is mean squared error.
Higher PSNR values indicate better performance.
All parameters are tuned as in previous section.

We test all methods using $8$ images ($300\times 300$ pixels) in Figure \ref{test image}.
We solve the matrix completion
tasks with random
mask and minor noise, where the missing pixels are randomly sampled and noise is i.i.d. standard Gaussian, which is relatively small compared with signal.
The results are shown in Figure \ref{Fig.rand} and \ref{show image_random}.
We can see that
the ISTRA achieves higher or comparable PSNR values with ISTA, but TNNR\_origin fails to recover some images
when more entries observed.
We find that the rank of recovered matrices of ISTA and ISTRA increase consistently while observed ratio grows, and in which case the smallest singular values are often closed to zero.
This could happen when more details are observed, because the true images are not low-rank usually but they can be approximated by low-rank representations appropriately.
As a two-loop algorithm, TNNR\_origin cannot update the singular direction for small but non-zero singular values frequently, which leads the algorithm to a suboptimal solution and explains the obtained results.

\begin{table*}[t]
\caption{{PSNR values of recovered images with text mask and iteration numbers of SVD computation}}
\label{table:imresults}
\centering
\scalebox{0.71}{
\begin{tabular}{|c|c|c|c|c|c|c|c|c|c|}
  \hline
  Image & SVP & SVT  & APGL & TNNR\_origin & ISTA & ISTRA-0.3 & ISTRA-0.6 & \#SVD (TNNR\_origin) & \#SVD (ISTA) \\
  \hline
  1     & 26.76 & 25.97 & 28.93 & 27.05 & \textbf{32.00} & 31.23 & 31.39  & 1272 & 153      \\
  2     & 22.20 & 26.10 & 25.43 & \textbf{24.99} & 23.58 & 24.52 & 24.51  & 4796 & 147      \\
  3     & 23.22 & 26.18 & 26.95 & 27.07 & 29.42 & 30.26 & \textbf{30.33}  & 1829 & 368      \\
  4     & 24.62 & 25.80 & 25.32 & \textbf{28.69} & 26.75 & 27.48 & 27.61  & 5305 & 262       \\
  5     & 27.65 & 30.58 & 30.17 & 30.06 & \textbf{32.99} & 32.34 & 32.43  & 4274 & 112 \\
  6     & 22.62 & 21.53 & 22.29 & 23.09 & 22.52 & 23.27 & \textbf{23.39}  & 4618 & 255 \\
  7     & 29.27 & 30.97 & 31.77 & 31.09 & \textbf{35.23} & 33.33 & 33.57  & 3964 &  68 \\
  8     & 24.27 & 26.68 & 26.32 & \textbf{31.23} & 26.28 & 27.12 & 27.14  & 1328 & 358 \\
  \hline
\end{tabular}}
\end{table*}

\begin{figure*}[htbp]
 \centering
\subfigure[{Random mask}]{
\includegraphics[scale=0.33]{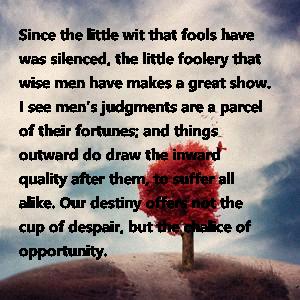}}
\subfigure[{SVP 26.76}]{
\includegraphics[scale=0.33]{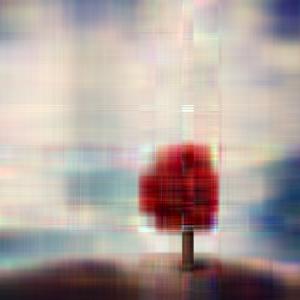}}
\subfigure[{SVT 25.97}]{
\includegraphics[scale=0.33]{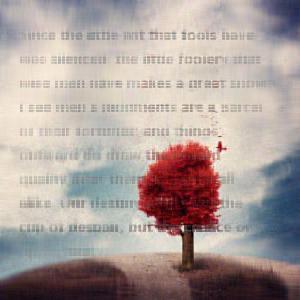}}
\subfigure[{APGL 28.93}]{
\includegraphics[scale=0.33]{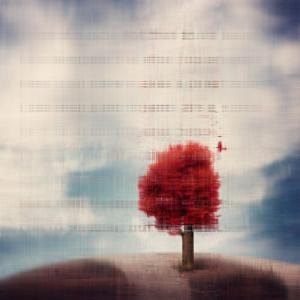}}\\

\subfigure[{TNNR\_origin 27.06}]{
\includegraphics[scale=0.33]{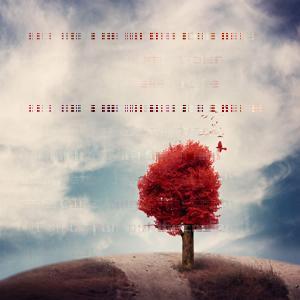}}
\subfigure[{ISTA 32.0}]{
\includegraphics[scale=0.33]{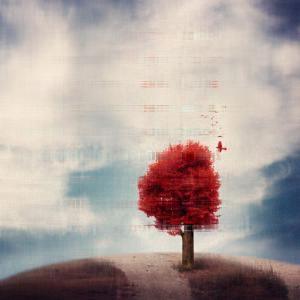}}
\subfigure[{ISTRA-0.3 31.22}]{
\includegraphics[scale=0.33]{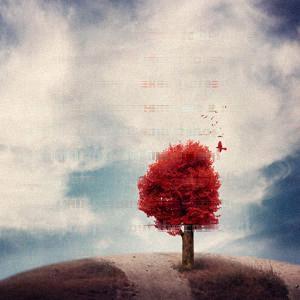}}
\subfigure[{ISTRA-0.6 31.39}]{
\includegraphics[scale=0.33]{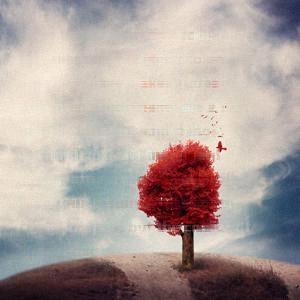}}
\caption{{Recovered images and PSNR values by different methods (text mask)}}
\label{show_image_text}
\end{figure*}
Next, we conduct experiments on text mask, which is harder since missing pixels are no longer random chosen. The parameters of APGL and SVT are tuned between $[1:100:10^4]$ to achieve better performances.
The results are shown in Table~\ref{table:imresults} and Figure~\ref{show_image_text}. We can see that truncated nuclear norm is more robust compared with reweighted nuclear norm. Compared with TNNR\_origin, ISTA requires about 10 times less iterations, which shows the effectiveness of proposed algorithms.

\subsection{Multi-Domain Recommendation}
\begin{table*}[h]
\caption{Statistics of the multi-domain recommendation
data}\label{table1}
\begin{center}
\begin{tabular}{|c|c|c|}
\hline
Domains& Book & Movie\\
\hline\hline \#Users & 13090 & 13090 \\ \hline \#Items & 17590 &
17922 \\ \hline
Sparsity & 99.66\% & 98.68\% \\
\hline
\end{tabular}
\end{center}
\end{table*}

\begin{table*}[t]
\caption{Comparison of performance with different training ratios.
Results are presented in the form of
$\mathrm{RMSE}_\textrm{test}$\scriptsize{($\mathrm{RMSE}_\textrm{train}$)}.}\label{table_results}
\begin{center}
\scalebox{0.76}{
\begin{tabular}{|c|c||c|c|c|c|c|c|}
\hline
Domains& Training & SVP & ISTA & PMF & CMF & GSMF & Alter-ISTA \\
\hline\hline \multirow{3}{*}{Book} & 80\% &
0.9606\scriptsize{(0.4898)} & 0.8801\scriptsize{ (0.6144)} &
0.7809\scriptsize{ (0.5235)} & 0.8172\scriptsize{ (0.6362)} &
0.7813\scriptsize{ (0.5684)} & \textbf{ 0.7389\scriptsize{
(0.4008)}}\\ \cline{2-8} & 60\% & 1.0147\scriptsize{(0.4658)} &
0.9066\scriptsize{ (0.5663)} & 0.7967\scriptsize{ (0.5353)} &
0.8517\scriptsize{ (0.6523)} & 0.7962\scriptsize{(0.6078)} &
\textbf{ 0.7479\scriptsize{ (0.4550)}}\\ \cline{2-8}
& 40\% & 1.1571\scriptsize{(0.4175)} & 1.0239\scriptsize{ (0.5563)} & 0.8397\scriptsize{ (0.5083)} & 0.9345\scriptsize{ (0.6227)} & 0.8030\scriptsize{ (0.5643)}& \textbf{0.7558\scriptsize{ (0.4911)}}\\
\hline\hline \multirow{3}{*}{Movie} & 80\% &
0.7661\scriptsize{(0.6011)} & 0.7336\scriptsize{ (0.6524)} &
0.7342\scriptsize{ (0.6014)} & 0.7325\scriptsize{ (0.6228)} &
0.7315\scriptsize{ (0.6177)} &\textbf{0.7130\scriptsize{
(0.6367)}}\\ \cline{2-8} & 60\% & 0.7870\scriptsize{(0.5905)} &
0.7429\scriptsize{ (0.6391)} & 0.7432\scriptsize{ (0.5952)} &
0.7423\scriptsize{ (0.6142)} & 0.7401\scriptsize{(0.5978)} &
\textbf{0.7209\scriptsize{ (0.6643)}}\\ \cline{2-8}
& 40\% & 0.8387\scriptsize{(0.5616)} & 0.7752\scriptsize{ (0.6259)} & 0.7678\scriptsize{ (0.5764)} & 0.7829\scriptsize{ (0.5784)} & 0.7870\scriptsize{ (0.4892)} & \textbf{0.7342\scriptsize{ (0.6885)}} \\
\hline
\end{tabular}}
\end{center}
\end{table*}

To measure the performance of Alter-ISTA and Alter-ISTRA in the practical task of
multi-domain recommendation, we use the data from a public website
Douban\footnote{\url{http://www.douban.com}}, where users can rate
movies, books and music, etc. We take two domains of ratings,
\emph{books} and \emph{movies} in our experiment. We remove users
and items with less than 10 ratings to provide enough ratings for
split into training and test sets for evaluation. A dataset is then
obtained containing 13090 users with 17590 ratings on books and
17922 ratings on movies. All ratings take values from 1 to 5. The
details of the dataset are listed in Table~\ref{table1}.

To evaluate the quality of recommendation, we use Root Mean Square
Error, $\mathrm{RMSE}(X)=\sqrt{||X_\Omega-Y_\Omega||^2/N}$, to
measure the discrepancy of predictions and the ground truth. We
compare to both matrix completion algorithms and recommendation
methods here as well. The penalty parameters are tuned between $[5:5:300]$
and the truncated rank for $X_0$ and $X_{1:2}$ are 20 and 30 respectively.
We conduct the experiments with different training ratios (80\%, 60\%
and 40\%) for a comprehensive comparison. The training sets are
sampled uniformly at random and the procedure is repeated 10 times.
The results are summarized in Table \ref{table_results}, where test
$\mathrm{RMSE}$ values are shown with training $\mathrm{RMSE}$
values inside the brackets. Bold values indicate the best
performance on the test data that is statistically significant with
95\% confidence. The results of SVT, TNNR-Origin and Alter-NN
are not reported here because first three algorithms have to compute full SVD in the first dozens of iterations which are too
expensive and not applicable to large scale problems.
Alter-ISTRA is not reported as well since it also requires full SVD and does not have significant advantages compared to Alter-ISTA based on previous experiments. 

From Table \ref{table_results}, we can observe that all the
recommendation methods achieve comparable performance in the
\emph{movie} domain, which contains relatively sufficient training
data. Meanwhile in the \emph{book} domain, CMF does not perform very
well as the training set is extremely sparse and the connection
between domains is weaker than it assumes. The performance of GSMF,
which allows different factors for different domains, is comparable
to
PMF, and better than the other baselines. 
ISTA performs comparably with the recommendation methods in the
\emph{movie} domain, while in the \emph{book} domain the
performances of the matrix completion approaches degenerate
significantly. This is probably because SVP and ISTA are more
sensitive to noise when sparsity is high.
 The last column records
the results of our proposed method of Alter-ISTA which demonstrates
significant superiority over the comparing algorithms. This
justifies that Alter-ISTA can effectively exploit the consistency
while modeling independency across multiple domains with the
benefits of improving the quality of recommendation.

\section{Conclusion}
In this paper, we propose the ISTA and ISTRA algorithm to solve rank minimization problems with different penalties.
We prove that the proposed algorithms can converge to a high-quality critical point globally with sublinear convergence rate $O(1/T)$, which is a much stronger result compared with existing work.
Empirical results on synthetic data and real-world applications further verify the accuracy and efficiency of our methods.
In experiments, we also observed that the iteration complexities of proposed algorithms on non-convex objectives were in the same order as proximal method on convex objective (nuclear norm), which indicated that the ISTA and ISTRA might achieve faster convergence rate in certain scenarios.
We hope to investigate the requirement of faster convergence rate in the near future.
\bibliography{bib}

\appendix

\section{Proof of Theorem~\ref{thm:main1}}\label{sec:app1}
\begin{lemma}
\label{lem_subgradient_bound}
\emph{(Upper bound for the subgradient)}
Suppose all assumptions are hold. For each iteration $t>0$, define
\begin{equation}\label{eqn:subgradient1}
G_{t+1} = \mu^{-1}(X_{t}-X_{t+1})+\nabla f(X_{t+1})-\nabla f(X_t)
\end{equation}
Then we have $G_{t+1}\in\partial F(X_{t+1})$ for all $t>0$, and
\begin{equation}
\|G_{t+1}\|_F\leq \rho_2\|X_t-X_{t+1}\|_F
\end{equation}
where $\rho_2 = \mu^{-1}+L$.
\end{lemma}
\begin{proof}
According to the optimal condition of convex optimization problem (\ref{eq_theo1}), we know that there exist a vector $\v_{t+1}\in\R^n$, $v_{t+1,i}\in \partial_{\sigma_i(X_{t+1})} |\sigma_i(X_{t+1})|$ for all $i=1,\ldots, n$, such that
\begin{align*}
\sigma(X_{t+1})-\sigma(X_t-\mu\nabla f(X_t)) + \mu\diag(\w)\v_{t+1} = 0
\end{align*}
On the other hand, by Theorem~7.1 in \citep{lewis2005nonsmooth}, for all $(U_{t+1}, V_{t+1}) \in O(U_{t+1}, V_{t+1})$
, we have that $U_{t+1}\diag(\v_t)V_{t+1}^\top \in \partial \sigma(X_{t+1})$. Therefore, based on chain rule we know that
\begin{align*} U_{t+1}(\diag(\w)\diag(\v_{t+1}))V_{t+1}^\top \in \partial g(X_{t+1})
\end{align*}
Based on the update scheme (\ref{eqn:proximal1}) in Corollary~\ref{corollary_proximal_map}, we know that $X_{t+1}$ and $X_t-\mu\nabla f(X_t)$ have same left and right singular vectors, $U_{t+1}$ and $V_{t+1}$ respectively. Hence we have that
\begin{align*}
\mu^{-1}(X_t-X_{t+1})-\nabla f(X_t) \in \partial g(X_{t+1})
\end{align*}
Then we have
\begin{align*}
G_{t+1} = \nabla f(X_{t+1})+\mu^{-1}(X_t-X_{t+1}) - \nabla f(X_t) \in \partial F(X_{t+1})
\end{align*}
Following this very reason and the $L$-smoothness of $f$, we have that
\begin{align*}
\|G_{t+1}\|_F &\leq \|\nabla f(X_{t+1}) - \nabla f(X_t)\|_F+ \|\mu^{-1}(X_t - X_{t+1})\|_F\\
&\leq (L+\mu^{-1})\|X_t-X_{t+1}\|_F
\end{align*}
which complete the proof.
\end{proof}
for the sake of simplicity, we define
\begin{align*}
\rho_1 = min\{\mu^{-1}-L\}
\end{align*}

\begin{lemma}\label{lemma_convergence_properties}
(\emph{Convergence properties}) Suppose all assumptions of $F$ are hold. We have following properties.
\begin{itemize}
\item[(i)]{ The sequence $\{F(X_t)\}_{t\in\mathbb{N}}$ is non-increasing
    and
    \begin{equation}\label{eqn:lem_converge_1}
    \frac{\rho}{2}\|X_{t+1}-X_{t}\|_F^2\leq F(X_{t})-F(X_{t+1}),\;\forall t\geq0.
    \end{equation}
    }
\item[(ii)]{ We have
    \begin{equation}\label{eqn:lem_converge_2}
    \sum_{t=1}^\infty\|X_{t+1}-X_{t}\|_F^2<\infty,
    \end{equation}
    then $\lim_{t\rightarrow\infty}{\|X_{t+1}-X_{t}\|_F=0}$.
    }
\end{itemize}
\end{lemma}
\begin{proof}
Since $X_{t+1}$ is in the optimal set of problem
(\ref{eqn:proximal_map}), in the $(t+1)$-th iteration we have
\begin{align*}
\begin{split}
&\langle\nabla f(X_{t}),X_{t+1}-X_{t}\rangle+g(X_{t+1}) +\frac{1}{2\mu}\|X_{t+1}-X_{t}\|_F^2\leq
g(X_{t})
\end{split}
\end{align*}
Following the smoothness of $f$, we have
\begin{align*}
f(X_{t+1})&\leq f(X_{t})+\frac{L}{2}\|X_{t+1}-X_{t}\|_F^2\nonumber +\langle\nabla f_{t+1}(X_{t}),
X_{t+1}-X_{t}\rangle
\end{align*}
Combining above inequalities we get

\begin{equation}\label{eq:theo2_3}
\begin{split}
&F(X_{t+1}) \leq f(X_{t+1})+g(X_{t+1})\leq f(X_{t})+g(X_{t}) -\frac{\rho_1 }{2}\|X_{t+1}-X_{t}\|_F^2 \\
&\leq F(X_t) - \frac{\rho_1 }{2}\|X_{t+1}-X_{t}\|_F^2
\end{split}
\end{equation}
where the first inequality follows the concavity of $g$. Since we choose the step size smaller than the reciprocal of the
largest Lipschitz constant $L$ as shown in Algorithm \ref{alg_pg0}, following (\ref{eq:theo2_3}) we have that the sequence
$\{F(X_{t})\}_{t\in\mathbb{N}}$ is non-increasing, which is followed by (i). Meanwhile, since
$F$ is bounded from blow,
it will converge to some real number $\overline{\phi}$. By summing up (\ref{eqn:lem_converge_1}) from $t=0$ to $N-1$ and taking the
limit $N\rightarrow\infty$, we can prove (ii).
\end{proof}

%

Equipped with above lemmas, we can get some useful properties of
the limit points. The set of all limit points is
denoted by
\begin{align*}
\mathrm{limit}(X_0) = \{&\hat{X}\in\mathbb{R}^{n\times
    m}: \exists
    \mbox{ an}\; \mbox{increasing\;sequence\;of\;integers } \{t_l\}_{l\in\mathbb{N}}, \\
    &X^{t_l}\rightarrow\hat{X}\; as \;t_l\rightarrow\infty\}.
\end{align*}
Let $\{X_t\}_{t\in \mathbb{N}}$ be the sequence
generated by Algorithm~\ref{alg_pg0} from $X_0$. Following the assumptions of $f$, we know that $\{X_t\}_{t\in \mathbb{N}}$ is a bounded sequence (Remark~5 in \citep{attouch2010proximal}). 
Then we have following propositions of limit points.
\begin{proposition}
\label{prop:limit_point}
\emph{(Properties of limit$(X_0)$, Proposition~2 in \citep{attouch2009convergence})} Let
$\{X_t\}_{t\in\mathbb{N}}$ be the sequence generated by Algorithm~\ref{alg_pg0}
with start point $X_0$. The following assertions hold.
\begin{itemize}
\item[(i)]{$\varnothing\neq\emph{limit}(X_0)\subset \mathrm{crit}\;(F)$,
    where $\mathrm{crit}\;(F)$ is the set of critical points of $F$.}
\item[(ii)]{We have
    \begin{equation}\label{eqn:limit_point1}
    \lim_{t\rightarrow\infty}{\mathrm{dist}(X_t,\mathrm{limit}(Z_0))}=0.
    \end{equation}}
\item[(iii)]{$\emph{limit}(X_0)$ is a non-empty, compact and
connected set.}
\item[(iv)]{The objective $F$ is finite and constant on $\emph{limit}(X_0)$.}
\end{itemize}
\end{proposition}

\begin{proof}[proof of Theorem~\ref{thm:main1}]
Define $X_*$ be a limit point of $\{X_t\}_{t\in\N}$, which means that there exists a subsequence $\{X_{t^l}\}_{l\in\N}$ converging to $X_* $ as $l\rightarrow \infty$. Following (\ref{eq:theo2_3}) and lower semi-continuous of $f$ and $g$, we have
\begin{align*}
\liminf_{l\rightarrow \infty} f(X_{t^l})\geq f(X_*)\quad
\end{align*}
Since $X_{t+1}$ is in the optimal set of problem
(\ref{eqn:proximal_map}), in the $(t+1)$-th iteration we have
\begin{align*}
& \langle X_{t+1}-X_t,\nabla f(X_t) \rangle + \frac{1}{2\mu}\|X_{t+1}-X_t\|_F^2 + g(X_{t+1})\\
\leq &\langle X_{*}-X_t,\nabla f(X_t) \rangle + \frac{1}{2\mu}\|X_{*}-X_t\|_F^2 + g(X_{*})
\end{align*}
Since the distance between successive iterations tends to 0 (by Lemma~\ref{lemma_convergence_properties}.(ii)), choosing $t = t_l-1$ we have $X_{t_l-1}$ tends to $X_*$ as $l\rightarrow \infty$. Besides, since $\{X_t\}_{t\in\N}$ is a bounded sequence and $\nabla f(X)$ is continuous by assumptions,   we have
\begin{align*}
&\limsup_{l\rightarrow \infty} g(X_{t_l}) \leq \limsup_{l\rightarrow \infty}  g(X_{t_l}) \\ & \leq \limsup_{l\rightarrow \infty}  \big\{ \langle X_*-X_{t_l-1}, \nabla f(X_{t_l-1}) \rangle + \frac{1}{2\mu}\|X_*-X_{t_l-1}\|_F^2\big\}+g(X_*)\\
& = g(X_*)
\end{align*}
Thus we have $g(X_{t_l})\rightarrow g(X_*)$ as $l\rightarrow \infty$. As a result, we can obtain that
\begin{equation}\label{eqn:limit_point2}
\lim_{l \rightarrow \infty} F(X_{t_l}) = F(X_*)
\end{equation}

As a consequence, if there exists an integer $\bar{t}$ such that $F(X_{\bar{t}})=F(X_*)$, then following (\ref{eqn:lem_converge_1}) we know that $X_{\bar{t}+1} = X_*$ and by induction $\{X_t\}_{t>\bar{t}}$ is stationary at $X_*$ and all results are hold. If this is not true, use Lemma~\ref{lemma_convergence_properties}.(ii) again we have that $F(X_*) < F(X_t)$ for all $t\in\N$.

To prove the convergence of sequence $\{X_t\}_{t\in \N}$, we need following lemma to show the KL property in the neighbourhood of critical point.

\begin{lemma}\label{lem:uniformized_KL}
\emph{(Uniformized KL rpoperty~\cite{bolte2014proximal})} Let $\Omega'$ be a compact set and $F$ be a proper and lower semi-continuous function. Assume that $F$ is constant on $\Omega$ and satisfies the KL property at each point of $\Omega$. Then, there exist $\epsilon>0$, $\delta >0 $ and $\varphi\in\Phi$ such that for all $\bar{X} \in \Omega$ and all $X$ in the intersection
\begin{equation} \label{eqn:lem_uniformizedKL_1}
\big\{X\in \R^{m\times n}:\dist(X,\Omega') \leq \epsilon\big\}\cap [F (\bar{X}) < F(X) < F(\bar{X})+\delta ]
\end{equation}
then we have
\begin{align*}
\varphi'(F(X)-F(\bar{X})) \dist(0, \partial F(X)) \geq 1
\end{align*}
\end{lemma}

For all $\delta > 0$, there exist a non-negative integer $t_0$ such that $F(X_*)< F(X_{t})+\delta$ for all $t > t_1$. Following eq.~(\ref{eqn:limit_point1}) we know that for all$ \epsilon>0$, there exist a non-negative integer $t_2$ such that $\dist(X_{t},X_*) \leq \epsilon$. Following these facts, we know that $X_t$ belongs to the intersection defined in (\ref{eqn:lem_uniformizedKL_1}) if $t> t_0 = \max\{t_1,t_2\}$. Thus following Lemma~\ref{lem:uniformized_KL} and Proposition~\ref{prop:limit_point}, we have
\begin{align*}
\varphi'(F(X_t) - F(X_*)) \dist(0, \partial F(X_t))\geq 1
\end{align*}
whenever $t>t_0$. Following Lemma~\ref{lem_subgradient_bound}, we also have that
\begin{align*}
\varphi'(F(X_t) - F(X_*)) \geq \frac{1}{\rho_2 \|X_{t-1}-X_{t}\|_F}
\end{align*}
By the concavity of $\varphi$, we have that
\begin{equation}\label{eqn:thm1_dis_lbound}
\varphi(F(X_t)- F(X_*)) - \varphi(F(X_{t+1})- F(X_*)) \geq \varphi'(F(X_t)-F(X_*))(F(X_t) - F(X_{t+1}))
\end{equation}
Besides, following Lemma~\ref{lemma_convergence_properties}, we have
\begin{align*}
\frac{\rho_1}{2}\|X_{t+1}-X_{t}\|_F^2\leq F(X_{t})-F(X_{t+1}),\;\forall t\geq0
\end{align*}
Define
\begin{align*}
\delta_{p,q} = \varphi(F(X_p)- F(X_*)) - \varphi(F(X_{q})- F(X_*))
\end{align*}
for $p,q \in N$. Thus (\ref{eqn:thm1_dis_lbound}) turns to
\begin{align*}
\delta_{t,t+1} \geq \frac{\rho_1\|X_{t+1}- X_t\|_F^2}{ 2\rho_2\|X_{t}-X_{t-1}\|_F}
\end{align*}
Due to the fact that $2\sqrt{ab}\leq a+b$ for $a,b>0$, we have
\begin{equation}\label{eqn:thm1_telescoping_1}
2\|X_{t+1}-X_t\|_F \leq \|X_t-X_{t-1}\|_F + \frac{2\rho_1\delta_{t,t+1}}{\rho_2}
\end{equation}
By summing up (\ref{eqn:thm1_telescoping_1}) from $t=t_0 + 1,\ldots T$, we have
\begin{align*}
&2\sum_{t=t_0+1}^T\|X_{t+1}-X_t\|_F \\
&\leq \sum_{t=t_0+1}^T \|X_t-X_{t-1}\|_F +
\frac{\rho_1}{\rho_2}\sum_{t=t_0 +1 }^T\delta_{t,t+1}\\
&= \sum_{t=t_0+1}^T\|X_{t+1}-X_t\|_F + \|X_{t_0+1} -X_{t_0}\|_F-\|X_{T+1}-X_T\|_F +
\frac{2\rho_1}{\rho_2}\sum_{t=t_0 + 1 }^T\delta_{t,t+1}\\
& \leq  \sum_{t=t_0+1}^T\|X_{t+1}-X_t\|_F + \|X_{t_0+1} -X_{t_0}\|_F +
\frac{2\rho_1}{\rho_2}\sum_{t=t_0 + 1}^T\delta_{t,t+1}\\
& \leq  \sum_{t=t_0+1}^T\|X_{t+1}-X_t\|_F + \|X_{t_0+1} -X_{t_0}\|_F + \frac{2\rho_1}{\rho_2}\delta_{t_0+1,T+1}\\
& \leq  \sum_{t=t_0+1}^T\|X_{t+1}-X_t\|_F + \|X_{t_0+1} -X_{t_0}\|_F + \frac{2\rho_1}{\rho_2} (\varphi(F(X_{t_0})-F(X_*)))
\end{align*}
where the third inequality follows from the definition of $\delta_{t,t+1}$ and the last inequality follows from the non-negativeness of $\varphi$. Thus, for any $T>t_0$ we have
\begin{align*}
\sum_{t=t_0+1}^T\|X_{t+1}-X_t\|_F \leq \|X_{t_0+1} -X_{t_0}\|_F + \frac{2\rho_1}{\rho_2} (\varphi(F(X_{t_0})-F(X_*)))
\end{align*}
which implies (\ref{eqn:thm1_sum_bound}) as $T\rightarrow\infty$ and r.h.s is bounded. Then for any $q>p>t_0$,
\begin{align*}
\|X_p-X_q\|_F = \|\sum_{t=p}^{q-1}X_t-X_{t+1}\|_F \leq \sum_{t=p}^{q-1} \|X_t-X_{t+1}\|_F
\end{align*}
where the inequality follows from the triangle inequality. Then following (\ref{eqn:thm1_sum_bound}), we have that $\sum_{t=t_0+1}^\infty \|X_{p}-X_{p+1}\|$ converges to zeros as $t_0\rightarrow \infty$, which means $\{X_t\}_{t\in \N}$ is a convergent sequence. Then following the Proposition~\ref{prop:limit_point}.(i), we can conclude the result.
\end{proof}

\section{Proof of Lemma~\ref{lem:limit_point2}}\label{sec:app2}
In the beginning we will revise the proves of lemmas in Appendix.~\ref{sec:app1} to make sure that they are still satisfied for Algorithm~\ref{alg_spg}.
\begin{lemma}\label{lemma_convergence_properties2}
(\emph{Convergence properties}) Suppose all assumptions of $F$ are hold. We have following properties.
\begin{itemize}
\item[(i)]{ The sequence $\{F(X_t)\}_{t\in\mathbb{N}}$ is non-increasing
    and
    \begin{equation}\label{eqn:lem_converge2_1}
    \frac{\rho}{2}\|X_{t+1}-X_{t}\|_F^2\leq F(X_{t})-F(X_{t+1}),\;\forall t\geq0.
    \end{equation}
    }
\item[(ii)]{ We have
    \begin{equation}\label{eqn:lem_converge2_2}
    \sum_{t=1}^\infty\|X_{t+1}-X_{t}\|_F^2<\infty,
    \end{equation}
    then $\lim_{t\rightarrow\infty}{\|X_{t+1}-X_{t}\|_F=0}$.
    }
\end{itemize}
\end{lemma}

\begin{proof}
Since $X_{t+1}$ is in the optimal set of problem
(\ref{eqn:proximal_problem2}) and $u_t(X_t) = g(X_t)$, in the $(t+1)$-th iteration we have
\begin{align*}
&\langle\nabla f(X_{t}),X_{t+1}-X_{t}\rangle+u_t(X_{t+1}) +\frac{1}{2\mu}\|X_{t+1}-X_{t}\|_F^2\leq
g(X_{t})
\end{align*}
Following the smoothness of $f$, we have
\begin{align*}
f(X_{t+1})&\leq f(X_{t})+\frac{L}{2}\|X_{t+1}-X_{t}\|_F^2\nonumber +\langle\nabla f_{t+1}(X_{t}),
X_{t+1}-X_{t}\rangle
\end{align*}
Combining above inequalities and the fact that $u_t(X_{t+1})\geq g(X_{t+1})$, we get

\begin{equation}\label{eqn:lem_convergence2_3}
\begin{split}
&F(X_{t+1}) \leq f(X_{t+1})+g(X_{t+1})\leq f(X_{t})+g(X_{t}) -\frac{\rho_1 }{2}\|X_{t+1}-X_{t}\|_F^2 \\
&\leq F(X_t) - \frac{\rho_1 }{2}\|X_{t+1}-X_{t}\|_F^2
\end{split}
\end{equation}
where the first inequality follows the concavity of $g$. Since we choose the step size smaller than the reciprocal of the
largest Lipschitz constant $L$ as shown in Algorithm \ref{alg_spg}, following (\ref{eqn:lem_convergence2_3}) we have that the sequence
$\{F(X_{t})\}_{t\in\mathbb{N}}$ is non-increasing, which is followed by (i). Meanwhile, since
$F$ is bounded from blow,
it will converge to some real number $\overline{\phi}$. By summing up (\ref{eqn:lem_converge2_1}) from $t=0$ to $N-1$ and taking the
limit $N\rightarrow\infty$, we can prove (ii).
\end{proof}

\begin{lemma}
\label{lem_subgradient2_bound}
\emph{(Upper bound for the subgradient)}
Suppose all assumptions are hold. For each iteration $t>0$, define
\begin{equation}\label{eqn:subgradient21}
G_{t+1} = \mu^{-1}(X_{t}-X_{t+1})+\nabla f(X_{t+1})-\nabla f(X_t) + D_t
\end{equation}
where $D_t = U_{t+1}\diag(\w_{t+1} - \w_t)V_{t+1}^\top$, $(U_{t+1},V_{t+1} ) \in O(U_{t+1},V_{t+1})$ such that $ X_{t+1}= U_{t+1} \diag(\sigma(X_{t+1})) V_{t+1}^\top$. Then we have $G_{t+1}\in\partial F(X_{t+1})$ for all $t>0$, and
\begin{equation}
\|G_{t+1}\|_F\leq \rho_3\|X_t-X_{t+1}\|_F
\end{equation}
where $\rho_3 = L+\mu^{-1} + \frac{(1-p)pn}{ \varepsilon^{2-p}} $.
\end{lemma}
\begin{proof}
According to the optimal condition of convex optimization problem (\ref{eq_theo1}), we know that there exist a vector $\v_{t+1}\in\R^n$, $v_{t+1,i}\in \partial_{\sigma_i(X_{t+1})} |\sigma_i(X_{t+1})|$ for all $i=1,\ldots, n$, such that
\begin{align*}
\sigma(X_{t+1})-\sigma(X_t-\mu\nabla f(X_t)) + \mu\diag(\w_t)\v_{t+1} = 0
\end{align*}
Following the same analysis as in Lemma~\ref{lem_subgradient_bound} we know that
\begin{align*} U_{t+1}(\diag(\w_{t+1})\diag(\v_{t+1}))V_{t+1}^\top \in \partial g(X_{t+1})
\end{align*}
Based on the update scheme (\ref{eqn:proximal1}) for penalty $u_t$, we have that
\begin{align*}
\mu^{-1}(X_t-X_{t+1})-\nabla f(X_t)+ D_{t+1} \in \partial g(X_{t+1})
\end{align*}
Then we have
\begin{align*}
G_{t+1} = \nabla f(X_{t+1})+\mu^{-1}(X_t-X_{t+1}) - \nabla f(X_t) + D_{t+1} \in \partial F(X_{t+1})
\end{align*}
Next we will bound the extra term $D_{t+1}$ as follows
\begin{align*}
&\|D_{t+1}\|_F = \|\w_{t+1}-\w_t\|_2\leq \sum_{i=1}^n |w_{t+1,i}-w_{t,i}| \\
&= p\sum_{i=1}^n \bigg|\frac{( \sigma_i(X_t) + \varepsilon )^{1-p} - ( \sigma_i(X_{t+1}) + \varepsilon )^{1-p}}{[(\sigma_i(X_t) + \varepsilon )( \sigma_i(X_{t+1}) + \varepsilon )]^{1-p}} \bigg|\\
&\leq \frac{p}{\varepsilon^{2(1-p)}}\sum_{i = 1}^n \bigg|\frac{1-p}{\min\{(\sigma_i(X_t) + \varepsilon)^p, (\sigma_i(X_{t+1}) + \varepsilon)^p\}}(\sigma_i(X_t)-\sigma_i(X_{t+1}) \bigg|\\
&\leq \frac{(1-p)pn}{ \varepsilon^{2-p}} \|X_{t+1}-X_t\|_2 \leq \frac{(1-p)pn}{ \varepsilon^{2-p}} \|X_{t+1}-X_t\|_F
\end{align*}
where the first equality follows from the unitarily invariant property of Frobenius norm; the first inequality follows from triangle inequality; the second inequality follows from the concavity of function $(x+\varepsilon)^{1-p}$, $x\geq 0$ and its lower bound, which is $\varepsilon^{1-p}$; and the last inequality follows from the upper bound of spectral norm. Then we have that
\begin{align*}
\|G_{t+1}\|_F &\leq \|\nabla f(X_{t+1}) - \nabla f(X_t)\|_F+ \|\mu^{-1}(X_t - X_{t+1})\|_F +\|D_{t+1}\|_F\\
&\leq \bigg(L+\mu^{-1} + \frac{(1-p)pn}{ \varepsilon^{2-p}} \bigg)\|X_t-X_{t+1}\|_F
\end{align*}
which complete the proof.
\end{proof}
\begin{proof}[proof of Lemma~\ref{lem:limit_point2}]
Equipped with modified Lemma~\ref{lemma_convergence_properties2} and \ref{lem_subgradient2_bound}, we can get our results following the same reason as for Proposition~\ref{prop:limit_point}.
\end{proof}

\section{Proof of Theorem~\ref{thm:main3}} \label{sec:app3}
\begin{proof}
For simplicity, we use the following abbreviations in the $(t+1)$-th
iteration:
\begin{equation}\label{eq_abbreviation}
\begin{split}
&f_{t+1}(X^d_{t})=
f(X^0_{t+1},\ldots,X^{d-1}_{t+1},X^d_{t},\ldots,X^D_{t}),\\
&f_{t+1}(X^d_{t+1})=
f(X^0_{t+1},\ldots,X^d_{t+1},X^{d+1}_{t},\ldots,X^D_{t}).
\end{split}
\end{equation}
We also define
\begin{equation}\label{eq_theo2_def1}
\rho= \min\{\mu^{-1}-L_1,\ldots,\mu^{-1}-L_D\},
\end{equation}
the sequence generated by Algorithm 2 as
\begin{align*}\label{eq_theo2_def2_1}
&Z_{t}=(X_{t}^0,\ldots,X_{t}^D),\; \forall t\geq0,
\end{align*}
and
\begin{align*}
\sum_{d=0}^D\|X^d_{t-1}-X^d_t\|_F^2 = \|Z_{t-1}-Z_t\|_F^2.
\end{align*}
Then following (\ref{eq_abbreviation}), we get
\begin{align*}
F_{t}(Z_t) = f_{t}(Z_{t})+\sum_{d=0}^D{g^d(X^d_{t})}.
\end{align*}

To prove the global convergence, we start with
extending the proof of convergence
properties from single-variate case
to multivariate case.
\begin{lemma}\label{lemma_convergence_properties3}
\emph{(Convergence properties)} Suppose that
Assumption~\ref{assump:2}.(\ref{assump:b2}) and (\ref{assump:b4}) are hold. The following assertions hold.
\item[(i)]{ The sequence $\{F(Z_t)\}_{t\in\mathbb{N}}$ is non-increasing
    and
    \begin{equation}\label{eq_lemma3_1}
    \frac{\rho}{2}\|Z_{t+1}-Z_{t}\|_F^2\leq F(Z_{t})-F(Z_{t+1}),\;\forall t\geq0.
    \end{equation}
    }
\item[(ii)]{ We have
    \begin{equation}\label{eq_lemma3_2}
    \sum_{t=1}^\infty{\sum_{d=0}^D\|X^d_{t+1}-X^d_{t}\|_F^2}
    =\sum_{t=1}^\infty\|Z_{t+1}-Z_{t}\|_F^2<\infty,
    \end{equation}
    then $\lim_{t\rightarrow\infty}{\|Z_{t+1}-Z_{t}\|_F=0}$.
    }
\end{lemma}
\begin{proof}
Since $X^d_{t+1},d=0,\ldots,D,$ is the optimal solution of problem
(\ref{eqn:proximal_map}), in the $(t+1)$-th iteration we have
\begin{align*}
&\langle\nabla_{X^d_{t}}f_{t+1}(X^d_{t}),X^d_{t+1}-X^d_{t}\rangle+g^d(X^d_{t+1}) +\frac{1}{2\mu}\|X^d_{t+1}-X^d_{t}\|_F^2\leq
g^d(X^d_{t})
\end{align*}
Following Assumption~\ref{assump:2}.(\ref{assump:b1}), we have
\begin{align*}
f_{t+1}(X_{t+1}^d)\leq f_{t+1}(X_{t}^d)+\frac{L_d}{2}\|X^d_{t+1} - X^d_{t}\|_F^2 +\langle\nabla_{X^d_{t}}f_{t+1}(X^d_{t}),
X^d_{t+1}-X^d_{t}\rangle
\end{align*}
Combining above two inequalities, we get

\begin{align*}
f_{t+1}(X^d_{t+1})&+g^d(X^d_{t+1})\leq f_{t+1}(X^d_{t})+g^d(X^d_{t}) -\frac{\mu^{-1}-L_d}{2}\|X^d_{t+1}-X^d_{t}\|_F^2
\end{align*}
Adding up the above inequalities regarding $d=0,...,D$, for all
$t\geq0$ we have
\begin{equation}\label{eq_theo2_4}
\begin{split}
F(Z_{t})-F(Z_{t+1})=&\sum_{d=0}^D[f_{t}(X^d_{t})+g^d(X^d_{t}) -f_{t+1}(X^d_{t+1})-g^d(X^d_{t+1})]\\
\geq&\sum_{d=0}^D\frac{\mu^{-1}-L_{d}}{2}\|X_{t+1}-X_{t}\|_F^2.
\end{split}
\end{equation}
Following (\ref{eq_theo2_4}), we have that the sequence
$\{F(Z_{t})\}_{t\in\mathbb{N}}$ is non-increasing, and since
$F$ is bounded from blow according to Assumption~\ref{assump:2}.(\ref{assump:b1}),
it will converge to some real number $\overline{\phi}$. Meanwhile,
Since we choose the step size smaller than the reciprocal of the
largest Lipschitz constant $L_{\max}$ as shown in Algorithm 2, from
(\ref{eq_theo2_def1}) it follows that
\begin{equation}\label{eq_theo2_5}
\begin{split}
\sum_{d=0}^D\frac{\mu^{-1}-L_{d}}{2}\|X_{t+1}-X_{t}\|_F^2
\geq&\frac{\rho}{2}\|Z_{t+1}-Z_{t}\|_F^2.
\end{split}
\end{equation}
Combining (\ref{eq_theo2_4}) and (\ref{eq_theo2_5}), (i) is proved.

By summing up (\ref{eq_lemma3_1}) from $t=0$ to $N-1$ and taking the
limit $N\rightarrow\infty$, we can prove (ii).
\end{proof}

\begin{lemma}\label{lem_subgradient_bound2}
\emph{(The lower bound of the iterate gap based on subgradient)}
Suppose that Assumption~\ref{assump:2}.(\ref{assump:b1}), (\ref{assump:b2}) and (\ref{assump:b3}) are hold. Let
$\{Z^k\}_{k\in\mathbb{N}}$ be the sequence generated by Algorithm 2
which is assumed to be bounded. For each iteration $t>0$ and
$d=0,...,D$, define
\begin{equation}\label{eq_lemma2_2}
\begin{split}
G^d_{t}=&\mu^{-1}(X^d_{(t-1)}-X^d_{t})
+\nabla_{X^d}f_{t}(Z_{t}) -\nabla_{X^d}f_{t}(X^d_{(t-1)}),\;d=0,\ldots,D.
\end{split}
\end{equation}
We have $(G^0_{t},\ldots,G^D_{t})\in\partial F(Z_{t})$, and
\begin{equation}\label{eq_lemma2_3}
\begin{split}
\|(G^0_{t},\ldots,G^D_{t})\|_F\leq &((D-1)M+(1+D)\mu^{-1})\|Z_t-Z_{t-1}\|_F,\; \forall t>0.
\end{split}
\end{equation}
\end{lemma}
\begin{proof}
Following the proof of Lemma~\ref{lem_subgradient_bound} and optimal condition of problem (\ref{eqn:proximal_map}), we have that
\begin{align*}
G_{t+1}^d = \nabla f(X_{t+1}^d) + \mu^{-1} (X_{t}^d - X_{t+1}^d ) - \nabla f(X_t^d)\in \partial F(X_{t+1}),\;\mathrm{for\; all\;} d=0,\ldots,D
\end{align*}
as defined in (\ref{eq_lemma2_2}).

Based on Assumption~\ref{assump:2}.(\ref{assump:b1}) and (\ref{assump:b2}) and assuming that the sequence
$\{Z_t\}_{t\in \mathbb{N}}$ is bounded, for $d=0,\ldots,D-1$ we have
\begin{equation}\label{eq_lemma3_7}
\begin{split}
\|G^d_t\|_F\leq&\mu^{-1}\|X^d_{t-1}-X^d_t\|_F+ \|\nabla_{X^d}f_{t}(Z_{t})
-\nabla_{X^d}f_t(Z_{t-1})\|_F\\
\leq&\mu^{-1}\|X^d_{t-1}-X^d_t\|_F+M\|Z_t-Z_{t-1}\|_F\\
\leq&(M+\mu^{-1})\|X^d_{t-1}-X^d_t\|_F+M\sum_{d'\neq d}\|X^{d'}_{t-1}-X^{d'}_t\|_F\\
\leq&(M+\mu^{-1})\|Z_{t-1}-Z_t\|_F,
\end{split}
\end{equation}
where we use the fact that $\nabla f$ is $M$-Lipschitz
continuous on bounded subsets.
For $d=D$, following the Lipschitz continuous gradient
property of $X^D$ and the fact that $\mu^{-1}\geq L_{D}$, we have
\begin{equation}\label{eq_lemma3_8}
\begin{split}
\|G^D_t\|_F\leq&\mu^{-1}\|X^D_{t-1}-X^D_t\|_F+
\|\nabla_{X^D}f(X^D_{t-1})-\nabla_{X^D}f(X^D_{t})\|_F\\
\leq&\mu^{-1}\|X^D_{t-1}-X^D_t\|_F+\mu^{-1}\|X^D_{t-1}-X^D_t\|_F\\
\leq&2\mu^{-1}\|X^D_{t-1}-X^D_t\|_F.
\end{split}
\end{equation}
When $t>0$, we can conclude
\begin{equation}\label{eq_lemma3_9}
\begin{split}
&\|(G^1_t,\ldots,G^D_t)\|_F\leq\sum_{d=0}^D\|G^d_t\|_F\\
\leq&((D-1)M+(D+1)\mu^{-1})\|Z_t-Z_{t-1}\|_F.
\end{split}
\end{equation}

\end{proof}
By modifying the two lemmas above, we can conclude the properties of
the limit point set. Let $\{Z_t\}_{t\in \mathbb{N}}$ be the sequence
generated by Algorithm~\ref{alg_pg1} from $Z_0$. The set of all limit points is
denoted by
\begin{equation}\label{eq_theo2_6_lim_point}
\begin{split}
\mathrm{limit}(Z_0) = \{&\hat{Z}\in\mathbb{R}^{n_0\times
    m_0}\times\ldots\times\mathbb{R}^{n_D\times m_D}: \exists
    \mbox{ an}\; \mbox{increasing\;sequence\;of\;integers } \{t_l\}_{l\in\mathbb{N}},\\
    &Z^{t_l}\rightarrow\hat{Z}\;as\;t_l\rightarrow\infty\}.
\end{split}
\end{equation}
Then the following lemma is hold following the very same proof of Lemma~\ref{prop:limit_point}.
\begin{lemma}
\emph{(Properties of limit$(Z_0)$)} Suppose
that Assumption~\ref{assump:2} is hold. Let
$\{Z_t\}_{t\in\mathbb{N}}$ be the sequence generated by Algorithm 2
with start point $Z_0$. The following assertions hold.
\item[(i)]{$\varnothing\neq\emph{limit}(Z_0)\subset \mathrm{crit}\;(F)$,
    where $\mathrm{crit}\;(F)$ is the set of critical points of $F$.}
\item[(ii)]{We have
    \begin{equation}\label{eq_lemma4_1}
    \lim_{t\rightarrow\infty}{\mathrm{dist}(Z_t,\mathrm{limit}(Z_0))}=0.
    \end{equation}}
\item[(iii)]{$\emph{limit}(Z_0)$ is a non-empty, compact and
connected set.}
\item[(iv)]{The objective $F$ is finite and constant on $\emph{limit}(Z_0)$.}
\end{lemma}

To this end, all lemmas used in the proof of Theorem~\ref{thm:main1} are verified, and Theorem~\ref{thm:main3} is proved following the same reason of Theorem~\ref{thm:main1}.
\end{proof}
\end{document}